\documentclass{article}

\PassOptionsToPackage{numbers}{natbib}
\usepackage[final]{neurips_2021}

\usepackage[utf8]{inputenc} 
\usepackage[T1]{fontenc}    

\usepackage{url}            
\usepackage{booktabs}       
\usepackage{amsfonts}       
\usepackage{nicefrac}       
\usepackage{microtype}      
\usepackage{xcolor}         
\usepackage{microtype}
\usepackage{graphicx}
\usepackage{mathtools}
\usepackage{comment}
\usepackage{tikz}
\usepackage{tikz-3dplot}
\usepackage{pgfplots}
\usetikzlibrary{positioning, calc}
\pgfplotsset{compat=1.13}
\usepgfplotslibrary{colormaps,fillbetween} 
\usepackage{tkz-euclide} 
\usetikzlibrary{calc,angles,quotes,3d,intersections, positioning,intersections,shapes}
\usepackage{siunitx}
\usepackage{amsmath,amsthm,amsfonts,amssymb}
\usetikzlibrary{automata}
\usepackage{algorithm}
\usepackage{algpseudocode}
\usepackage{array}
\newcolumntype{H}{>{\setbox0=\hbox\bgroup}c<{\egroup}@{}}
\usepackage{makecell}
\usepackage{subcaption}
\usepackage{multirow}
\usepackage{wrapfig}
\usepackage{optidef}

\usepackage[colorlinks]{hyperref}

\usepackage{enumitem}

\newtheorem{theorem}{Theorem}
\newtheorem{lemma}{Lemma}
\newtheorem{assumption}{Assumption}[section]

\newcommand{\B}{\bfseries}
\newcommand{\vecprod}[1]{{\left\langle#1\right\rangle}}

\newcommand{\bfx}{\mathbf{x}}
\newcommand{\bfu}{\mathbf{u}}
\newcommand{\bfv}{\mathbf{v}}
\newcommand{\bfk}{\mathbf{k}}
\newcommand{\bfg}{\mathbf{g}}
\DeclareMathOperator*{\argmax}{arg\,max}
\DeclareMathOperator*{\argmin}{arg\,min}

\tikzset{
	reuse path/.code={\pgfsyssoftpath@setcurrentpath{#1}}
}
\tikzset{even odd clip/.code={\pgfseteorule},
	protect/.code={
		\clip[overlay,even odd clip,reuse path=#1]
		(-6383.99999pt,-6383.99999pt) rectangle (6383.99999pt,6383.99999pt);
}}
\tikzset{
	dot/.style={circle, fill, minimum size=#1, inner sep=0pt, outer sep=0pt},
	dot/.default = 4.5pt,
	hemispherebehind/.style={ball color=gray!20!white, fill=none, opacity=0.3},
	hemispherefront/.style={ball color=gray!65!white, fill=none, opacity=0.4},
	ellipsoidfront/.style={ball color=gray!50!white, fill=none, opacity=0.5},
	circlearc/.style={thick,color=gray!90},
	circlearchidden/.style={thick,dashed,color=gray!90},
	equator/.style = {thick, black},
	diameter/.style = {thick, black},
	axis/.style={thick, -stealth,black!60, every node/.style={text=black, at={([turn]1mm,0mm)}},
	},
}

\newcommand\zjbj[4][7pt]{%
	\draw let \p1=(#2),\p2=(#4),\p0=(#3) in
	(#3)++({atan2(\y1-\y0,\x1-\x0)}:#1)
	--++({atan2(\y2-\y0,\x2-\x0)}:#1)
	--++({atan2(\y1-\y0,\x1-\x0)}:-#1);
}
\newcommand\zjbjcolor[4][7pt]{%
	\draw[color=red] let \p1=(#2),\p2=(#4),\p0=(#3) in
	(#3)++({atan2(\y1-\y0,\x1-\x0)}:#1)
	--++({atan2(\y2-\y0,\x2-\x0)}:#1)
	--++({atan2(\y1-\y0,\x1-\x0)}:-#1);
}
\newcommand{\nn}{\num[group-separator={,},group-minimum-digits=3]}

\title{Finding Optimal Tangent Points for \\ Reducing Distortions of Hard-label Attacks}

\author{
	Chen Ma $^1$ \quad Xiangyu Guo $^2$  \quad Li Chen $^{1,}$\thanks{Corresponding author.} \quad Jun-Hai Yong $^{1}$ \quad Yisen Wang $^{3,4}$ \\
	$^1$ School of Software, BNRist, Tsinghua University, Beijing, China \\
	$^2$ Department of Computer Science and Engineering, University at Buffalo, Buffalo NY, USA \\
	$^3$ Key Lab. of Machine Perception, School of Artificial Intelligence, Peking University, Beijing, China \\
	$^4$ Institute for Artificial Intelligence, Peking University, Beijing, China \\
	\texttt{machenstar@163.com, xiangyug@buffalo.edu} \\ \texttt{\{chenlee,yongjh\}@tsinghua.edu.cn, yisen.wang@pku.edu.cn}
}

\begin{document}
	
	\maketitle
	
	\begin{abstract}
		One major problem in black-box adversarial attacks is the high query complexity in the hard-label attack setting, where only the top-1 predicted label is available. In this paper, we propose a novel geometric-based approach called Tangent Attack (TA), which identifies an optimal tangent point of a virtual hemisphere located on the decision boundary to reduce the distortion of the attack. 
		Assuming the decision boundary is locally flat, we theoretically prove that the minimum $\ell_2$ distortion can be obtained by reaching the decision boundary along the tangent line passing through such tangent point in each iteration. To improve the robustness of our method, we further propose a generalized method which replaces the hemisphere with a semi-ellipsoid to adapt to curved decision boundaries. Our approach is free of pre-training. Extensive experiments conducted on the ImageNet and CIFAR-10 datasets demonstrate that our approach can consume only a small number of queries to achieve the low-magnitude distortion. The implementation source code is released online at \url{https://github.com/machanic/TangentAttack}.
	\end{abstract}
	
	\section{Introduction}
	
	Adversarial attacks cause deep neural networks (DNNs) to make incorrect predictions by slightly perturbing benign images during the test time. They can be divided into two main categories on the basis of the amount of information exposed by the target model, namely white-box and black-box attacks. Many white-box attacks \cite{Carlini2017TowardsET,madry2018towards,moosavi2016deepfool} have been proposed, and they can compute the gradients w.r.t. the target model's input images to generate adversarial examples with the first-order optimization techniques. In contrast, black-box attacks are more practical because they craft adversarial examples without requiring the target model's gradients. 
	
	Over the past years, the community has made considerable efforts in developing black-box attacks, and the proposed methods can be divided into transfer- and query-based attacks. 
	Transfer-based attacks \cite{liu2017delving,wang2020unified,wu2020skip} generate adversarial examples by using a white-box attack method against a surrogate model to fool the target model. Although there is no need to query the target model in these attacks, the attack success rate can not be guaranteed, especially in the case of targeted attacks. To achieve satisfactory attack success rates, the query-based attacks use elaborate queries to obtain the feedback of the target model for crafting adversarial examples. In the score-based setting, the query-based attacks \cite{bai2020improving,cheng2019improving,ilyas2018prior,ma2021simulator} estimate approximate gradients by querying the predicted scores of the target model at multiple points. However, in most real-world scenarios, the score-based setting is not applicable because the public service returns only the top-1 predicted label (\textit{i.e.,} the hard label) rather than the predicted score. In this case, since the feedback information is limited and the objective function is discontinuous, the attack requires solving a high-dimensional combinatorial optimization problem, which is often challenging. 
	
	\begin{figure}[t]
		\centering
		\def\r{1.0}
		\begin{tikzpicture}[scale=0.9]
		\newcommand\zjbjgreen[4][7pt]{%
			\draw[color=black!50!green] let \p1=(#2),\p2=(#4),\p0=(#3) in
			(#3)++({atan2(\y1-\y0,\x1-\x0)}:#1)
			--++({atan2(\y2-\y0,\x2-\x0)}:#1)
			--++({atan2(\y1-\y0,\x1-\x0)}:-#1);
		}

		\node at (1.5,2.6) {$H$};

		\coordinate (x) at  (1.5,2);
		\fill (x) circle[radius=1.5pt] node[anchor=east] {\small{original image} $\mathbf{x}$};
		\coordinate (O) at  (4,3);
		\draw[draw=black,densely dotted,thick] ($(O)+(\r,0)$) arc (0:180:\r); 
		
		\fill (O) circle[radius=1.5pt] node[anchor=north] {$\mathbf{x}_{t-1}$};
		\coordinate (G) at ($(O)+(0,\r)$);
		\draw[thick] (7,3) -- (O) .. controls (1.5,3) .. (1.8,3.2) parabola bend (G)  (7,4.2);
		\draw[fill=brown!50,nearly transparent]  (7,3) -- (O) .. controls (1.5,3) .. (1.8,3.2) parabola bend (G)  (7,4.2);
		
		\fill[fill=red] (G) circle[radius=1.5pt] node[anchor=south,color=red] {HSJA point $\mathbf{g}$};
		\draw[-stealth,red,densely dashed,thick] (O) -- (G);
		
		\coordinate (H) at (3.1339745962155616,3.5);
		
		\coordinate (interset_HSJA) at (2.7500, 3.0000);
		\fill[fill=red] (interset_HSJA) circle[radius=1.5pt] node[anchor=north west,color=red] {$\mathbf{x}_\mathbf{g}$};
		\draw[red,densely dashed,thick] (x) -- (interset_HSJA);
		\draw[-stealth,red,thick] (G) -- (interset_HSJA);

		\coordinate (tangent) at (3.3103, 3.7241);
		\fill[fill=black!50!green] (tangent) circle[radius=1.5pt] node
		[anchor=south east,color=black!50!green] {tangent point $\mathbf{k}$};
		\coordinate (interset_tangent) at (2.5500, 3.0000);	
		\fill[fill=black!50!green] (interset_tangent) circle[radius=1.5pt] node[text=black!50!green] at (2.5,3.3) {$\mathbf{x}_{t}$};
		\draw[-stealth,black!50!green,thick] (tangent) -- (interset_tangent);
		\draw[-stealth,black!50!green,densely dashed,thick] (O) -- (tangent);
		\draw[black!50!green,densely dashed,thick] (x) -- (interset_tangent);
		\zjbjgreen[4pt]{O}{tangent}{x};

		\node at (5.8,3.5) {\small adversarial region};
		\node at (5.8,2.3) {\small non-adversarial region};
		\node at (5.8,4.6) {\small non-adversarial region};
		
		\end{tikzpicture}
		\caption{Simplified two-dimensional illustration of our motivation. $H$ is the decision boundary, and $\mathbf{x}_{t-1}$ is the current adversarial example mapped onto the decision boundary at the $(t-1)$-th iteration. HSJA updates $\mathbf{x}_{t-1}$ along the gradient direction to reach $\mathbf{g}$ and then maps it to $H$ at $\mathbf{x}_{\mathbf{g}}$ along the line through $\bfx$ and $\bfg$. However, the optimal update should be the tangent point $\mathbf{k}$ because it can be mapped onto $H$ at $\mathbf{x}_t$ that has the shortest distance to the original image $\mathbf{x}$.}
		\label{fig:fig1}
	\end{figure}
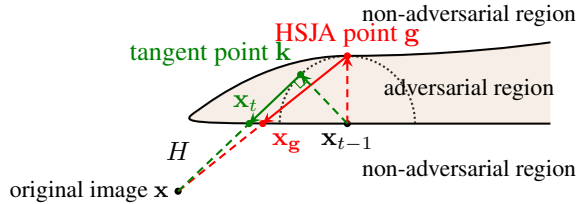
	To reformulate the attack as a real-valued continuous optimization problem, OPT \cite{cheng2018queryefficient}, Sign-OPT~\cite{cheng2019sign}, and RayS \cite{chen2020rays} focus on minimizing an objective function $g(\theta)$, which is defined as the distance from the original image to the nearest adversarial example along the direction $\theta$. However, when attacking complex models, it may be difficult to find a suitable direction $\theta$ along which adversarial examples exist.
	
	Boundary Attack (BA) \cite{brendel2018decisionbased}, HopSkipJumpAttack (HSJA) \cite{chen2019hopskipjumpattack}, QEBA \cite{li2020qeba}, qFool \cite{liu2019qFool}, and Policy Driven Attack (PDA) \cite{yan2021policydriven} eliminate the search of the direction $\theta$. 
	Instead, they start from a large adversarial perturbation and then reduce its distortion while staying in the adversarial region. Because the output labels of the target model flip only near the decision boundary, these attacks restrict their explorations to the regions near the decision boundary. However, they do not thoroughly investigate the geometric properties of the decision boundary to accelerate the attack. 
	For example, PDA uses a reinforcement learning framework to train a policy network to predict search directions, which are not geometrically optimal.
	In addition, the prediction accuracy of the policy network decreases significantly in the later stages of the iterations, resulting in worse performance of these iterations.
	HSJA and qFool simply use the gradient $\mathbf{u}$ estimated at the decision boundary as the direction of each update, while ignoring a geometrically critical issue, \textit{i.e.,} $\mathbf{u}$ is not the optimal direction to be followed (Fig. \ref{fig:fig1}). We could explore better search directions at each attack iteration.

	To find the optimal search direction for minimizing the distortions of attack, we propose a new geometric-based approach whose motivation is illustrated in Fig. \ref{fig:fig1}. We construct a virtual semicircle $B$ centered at $\mathbf{x}_{t-1}$ to indicate all possible locations that $\mathbf{x}_{t-1}$ can reach along different directions, and the radius of $B$ limits the range of updates for successful attacks. It is easy to observe that moving along the tangent line can reach the nearest location of the decision boundary to the original image $\mathbf{x}$, thereby producing the adversarial example with the minimum distortion. In real attack scenarios, the image data reside in a high-dimensional space, and the semicircle becomes a hemisphere. In this case, the benefit of using tangent points still exists, and we provide the detailed description and the formal proof in Section \ref{sec:method} and appendix.
	
	To summarize, the main contributions of this study are as follows.
	\begin{enumerate}[itemindent=0pt,leftmargin=1.15em]
		\item We cast the problem of minimizing the distortion in hard-label attacks into a geometric problem. We discover that the minimum distortion can be obtained by searching the optimal tangent point of a virtual hemisphere around the adversarial example at each iteration.
		\item We propose a novel geometric-based method to obtain a closed-form solution of the optimal tangent point. We provide an intuitive explanation of our approach, as well as a formal proof of its correctness. 
		\item To improve robustness, we further propose a generalized method that replaces the hemisphere with a semi-ellipsoid to adapt to the target models with curved decision boundaries. 
		\item Extensive experiments conducted on the CIFAR-10 \cite{krizhevsky2009learning} and ImageNet \cite{ImageNet} datasets demonstrate the effectiveness of our approach. 
	\end{enumerate}

	\section{Related Work}
	\label{sec:related_work}
	Query-based black-box attacks can be divided into score- and decision-based attack (a.k.a. the hard-label attack). Score-based attacks \cite{ilyas2018blackbox,bai2020improving,bhagoji2018practical,chen2017zoo,cheng2019improving,ilyas2018prior,ma2021simulator,moonICML19} use the predicted probability score to craft adversarial examples, which is not always available in most real-world systems. Hard-label attacks are more useful, but obviously more challenging, because only the top-1 predicted label can be obtained. The hard-label attacks usually fall into three categories. 
	
	The first category starts from the original image $\mathbf{x}_0$ and attempts to find a optimal direction $\theta$ to reach the adversarial example. OPT \cite{cheng2018queryefficient} searches an optimal $\theta$ to minimize the distance from $\mathbf{x}_0$ to the nearest adversarial example. Sign-OPT \cite{cheng2019sign} improves the query efficiency of OPT by using a single query to estimate the sign of the directional derivative. RayS \cite{chen2020rays} eliminates the gradient estimation and proposes a fast check step to efficiently find the direction $\theta$. However, RayS is only applicable to the untargeted attack under the $\ell_\infty$ norm because it is difficult to find a suitable direction to reach the region of the target class in a targeted attack, especially in the case of a large number of classes.

	The second category starts from a large perturbation or an image of the target class, and then reduces its distortion while
	staying in adversarial region, thereby gradually making it closer to the original image. BA \cite{brendel2018decisionbased} and NES \cite{ilyas2018blackbox} are two representative methods, but they have high query complexity. Biased BA \cite{Brunner2019BiasedBoundaryAttack} reinterprets BA as a biased sampling framework and incorporates different biases to improve the query efficiency. HSJA \cite{chen2019hopskipjumpattack} utilizes the gradient estimation and the binary search to outperform BA. HSJA can be used as the baseline of hard-label attacks. QEBA \cite{li2020qeba} improves HSJA by using dimension reduction techniques. SurFree \cite{maho2021surfree} and GeoDA~\cite{rahmati2020geoda} improve the performance of HSJA by exploiting geometric properties of DNNs. However, the geometric features of DNNs have not been fully explored, and they do not support the targeted attack. PDA \cite{yan2021policydriven} uses a reinforcement learning framework to train a policy network to predict promising directions. However, PDA requires a high-cost pre-training, which is not always available in all tasks.
	
	The third category uses a random sampling technique to improve query efficiency. Customized Adversarial Boundary (CAB) attack \cite{shi2020polishing} uses the current noise to select the sensitive area of images and customize sampling distribution. Evolutionary \cite{dong2019efficient} improves the query efficiency by reducing the dimension of the search space with the stochastic coordinate selection.

	The issue of RayS and GeoDA is that they only support untargeted attacks. Because moving the original image far enough in any direction can always make it escape the non-adversarial region in untargeted attacks. However, in targeted attacks, it is difficult to find a suitable direction along which the target class's region exists.
	In contrast, our approach supports all types of attacks, and we exploit the geometric characteristics of the decision boundaries of DNNs to boost the attack.

	
	\section{The Proposed Approach}
	\label{sec:method}
	\subsection{The Goal of Hard-label Attacks}
	
	Given a target model $f: \mathbb{R}^d \rightarrow \mathbb{R}^{k}$ and a benign image $\mathbf{x}\in[0,1]^d$ that is correctly classified using $f$, the goal of the adversary is to slightly perturb $\mathbf{x}$ into $\mathbf{x}_{\text{adv}}$, such that $f(\mathbf{x}_{\text{adv}})$ outputs incorrect prediction. In the hard-label attack, the adversary can only observe the top-1 predicted label of $f$, denoted as $\hat{y} = \arg \max_i f(\mathbf{x}_{\text{adv}})_i$. We define an indicator function $\phi(\cdot)$ of a successful attack:
	\begin{equation}
	\phi(\mathbf{x}_{\text{adv}}) \vcentcolon= \begin{cases}
	1 & \text{if } \hat{y} = y_\text{adv}\text{ in the targeted attack},\\
	& \quad\text{or } \hat{y} \neq y \text{ in the untargeted attack},\\
	0 & \text{otherwise},
	\end{cases}
	\label{eqn:phi}
	\end{equation}
	where $y\in\mathbb{R}$ is the true label of $\mathbf{x}$ and $y_\text{adv}\in\mathbb{R}$ is a predefined target class label. In this study, we focus on generating an adversarial example $\mathbf{x}_\text{adv}$ that satisfies $\phi(\mathbf{x}_\text{adv}) = 1$, such that the distortion $d(\mathbf{x}_\text{adv}, \mathbf{x}) := \|\mathbf{x}_{\text{adv}} - \mathbf{x}\|_p$ is minimized. This goal can be formulated as the optimization problem:
	\begin{equation}
	\min_{\mathbf{x}_\text{adv}} \,d(\mathbf{x}_\text{adv}, \mathbf{x})\quad \rm{s.t. }\quad \phi(\mathbf{x}_\text{adv}) = 1.
	\end{equation}
	\subsection{Motivation}
	\label{sec:motivation}
	Most recent attacks belong to the second category (Section \ref{sec:related_work}) follow a common procedure: it starts from an adversarial image yet not close enough to the benign one, then it iteratively searches for a closer adversarial image. 
	Let us take a typical attack HSJA \cite{chen2019hopskipjumpattack} for example (Fig. \ref{fig:fig1}). First, the attack initializes the adversarial example by using an image of the target class (in a targeted attack) or a noisy version of the original image (in an untargeted attack). Next, it performs the binary search to map the initial sample onto the decision boundary $H$, denoted as $\mathbf{x}_{t-1}$. Then, the algorithm iteratively performs three steps to update $\mathbf{x}_{t-1}$: estimating gradient $\mathbf{u}$ at $H$ by sampling many probes around $\mathbf{x}_{t-1}$; jumping to the point $\mathbf{g}$ along the direction $\mathbf{u}$ with the step size determined by the geometric progression; and mapping $\mathbf{g}$ onto $H$ at $\mathbf{x}_{\mathbf{g}}$ by performing the binary search along the line passing through $\mathbf{g}$ and the original image $\mathbf{x}$. However, geometrically speaking, the estimated gradient $\mathbf{u}$ is not the optimal search direction, and we can find a better boundary point $\mathbf{x}_t$ by connecting $\mathbf{x}$ to a certain point on the semicircle $B$ and then taking the intersection point on $H$. Fig.~\ref{fig:fig1} shows that the tangent point $\mathbf{k}$ is the optimum update because it leads to the minimum distortion of the attack.
	
	\subsection{Definition of Optimal Tangent Points}
	\label{sec:define_tangent}
	
	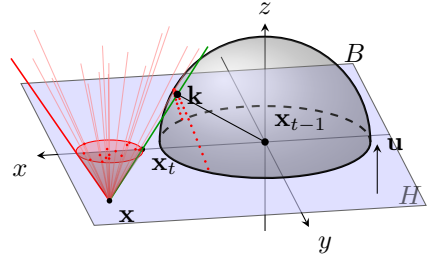
\begin{wrapfigure}{r}{0.455\textwidth}
		\vspace{-5mm}
		\centering
		\tikzset{yxplane/.style={canvas is yx plane at z=#1}}
		\def\r{2.0}
		\tdplotsetmaincoords{70}{170}
		\begin{tikzpicture}[tdplot_main_coords,scale=0.7]
		
		\begin{scope}[thin]
		\coordinate (O) at (0,0,0);
		\coordinate (x) at  (1.5*\r,0,-0.5*\r);
		\coordinate (x_k) at  (2.3670, 0.0000, 0.0000); 
		\coordinate (x_k_other) at (1.5*\r+ 0.633, 0, 0);  
		\coordinate  (k) at (1.6899, 0.0000, 1.0697); 
		
		\draw[-stealth] (-1.2*\r,0,0) -- (2.2*\r,0,0) node[anchor=north east,thin] {$x$};
		\draw[-stealth] (0,-2.4*\r,0) -- (0,2.4*\r,0) node[anchor= north west,thin] {$y$};
		\draw[-stealth] (0,0,-0.9*\r) -- (0,0,1.2*\r) node[anchor=south,thin] {$z$};
		\filldraw[fill=blue!20,opacity=0.6,draw,thin] (-1.2*\r,-2*\r,0) -- (-1.2*\r,2*\r,0) -- (2*\r,2*\r,0) -- (2*\r,-2*\r,0) -- cycle node[black] at (-1.1*\r,1.7*\r,0) {$H$};
		\draw[tdplot_screen_coords,thick] (-\r,0,0) arc (180:0:\r) node at (0,0,-0.85*\r) {$B$}; 
		\draw[thick,dashed,canvas is xy plane at z = 0] (\r,0) arc (0:-180:\r);
		\draw[thick,canvas is xy plane at z = 0] (\r,0) arc (0:180:\r);
		
		\shade[
		tdplot_screen_coords,hemispherefront,
		delta angle=-180,
		x radius=\r,
		] (\r,0)
		arc [y radius={\r*sin(20)},start angle=0,delta angle=-180]
		arc [y radius=\r,start angle=-180];
		\fill (O) circle[radius=2pt] node[anchor=south west] {$\mathbf{x}_{t-1}$};
		\end{scope}
		\fill (x_k) circle[radius=1.5pt] node[anchor=north west] {$\mathbf{x}_{t}$};

		\foreach \t in {(1.3379490713381077, -1.4865061510701991, 0.01384721401432265),
			(1.3823733887814025, -1.4378454265476082, 0.14712016634420821),
			(1.425422571177439, -1.3754441624571767, 0.2762677135323172),
			(1.4667081559981636, -1.2994874355839658, 0.40012446799449136),
			(1.505818556123482, -1.2102686095679278, 0.5174556683704469),
			(1.5423212310934784, -1.1082245925021033, 0.6269636932804356),
			(1.575768539737997, -0.9939718534403945, 0.7273056192139911),
			(1.6057075933663065, -0.8683394997633099, 0.8171227800989199),
			(1.6316940520574834, -0.7323950123955696, 0.8950821561724508),
			(1.6533092910592475, -0.5874579698953998, 0.9599278731777428),
			(1.6701797673435592, -0.43509753367212467, 1.0105393020306779),
			(1.6819968472992386, -0.2771107936114069, 1.045990541897716),
			(1.6853671218155768, 0.21020839803545205, 1.0561013654467304),
			(1.67572361671103, 0.3699600046407294, 1.0271708501330905),
			(1.6609218758969908, 0.5248906964273493, 0.9827656276909733),
			(1.6412380156352349, 0.6731197037546932, 0.9237140469057037),
			(1.6170209449616324, 0.8129792830282161, 0.851062834884897),
			(1.588672961753796, 0.9430446903595405, 0.7660188852613885),
			(1.5566299548234632, 1.0621444126041182, 0.6698898644703895),
			(1.5213431455573834, 1.169352738925627, 0.5640294366721494),
			(1.4832637869691456, 1.2639684608086488, 0.44979136090743704),
			(1.4428316367098928, 1.3454842853515403, 0.3284949101296792),
			(1.4004674611385515, 1.4135515449556957, 0.20140238341565375),
			(1.3565693895012525, 1.4679442301095946, 0.06970816850375744)}
		{\fill[fill=red] \t circle[radius=0.7pt];
		}
		
		\draw[black] (O) -- (k);
		\draw[black,-stealth] (-0.8*\r,1.6*\r,0) -- (-0.8*\r,1.6*\r,1.0) node[anchor=west] {$\mathbf{u}$}; 
		
		\coordinate (x_prime) at  (1.5*\r,0,0);
		\begin{scope}[canvas is xy plane at z=0]
		\filldraw[draw=red, fill=red, fill opacity=0.2]
		(x_prime) circle[radius=0.633];  
		\end{scope}

		\draw [black!40!green,semithick] (x) -- (1.1010291346571968,0,\r);  
		\draw[red,semithick] (x) -- (4.898970865342803,0,\r);  
		\foreach \e in {(3.228429081427254, 1.1784493699404006, 2.0), (1.8276020805402082, -0.8802508077051043, 2.0), (3.642845902104619, 0.508735970536522, 2.0), (2.4485633078938944, 0.8472584012817591, 2.0), (3.7029649867711343, 1.6311566470662888, 2.0), (1.627910001178166, -1.232094642115344, 2.0), (4.49318592703429, 0.07509813845707723, 2.0), (4.321828308358434, 1.2071699077636497, 2.0), (2.6496386478170297, 0.034653322598611924, 2.0), (3.1141296539216206, 0.5799566538060195, 2.0), (3.0650001183299174, -0.397049191846528, 2.0), (1.4798145664457494, -0.08616694607726645, 2.0), (1.2831530506610842, 0.07562033776479252, 2.0), (3.383234499990499, 1.032625502378612, 2.0), (4.1928752618149705, -1.4768567767051943, 2.0)
		}
		{
			\draw[red!50,opacity=0.6,draw,thin] (x) -- \e;
		}
		
		\coordinate (cone_edge) at  ({1.5*\r-0.633 * cos(30)},  {0.633 * sin(30)}, 0); 
		
		\fill[
		left color=red!50!white,
		right color=red!50!white,
		middle color=red!60,
		top color=red!20,
		bottom color=red,
		shading=axis,
		opacity=0.25
		]
		(x)--(cone_edge)--(x_k)--(x_k_other) -- cycle;
		
		
		\foreach \p in {(3.076143027142418, 0.3928164566468002, 0.0), (2.6092006935134027, -0.2934169359017014, 0.0), (3.214281967368206, 0.16957865684550733, 0.0), (2.8161877692979647, 0.2824194670939197, 0.0), (3.2343216622570448, 0.5437188823554295, 0.0), (2.5426366670593885, -0.410698214038448, 0.0), (3.497728642344763, 0.02503271281902574, 0.0), (3.440609436119478, 0.4023899692545499, 0.0), (2.8832128826056764, 0.011551107532870641, 0.0), (3.0380432179738737, 0.19331888460200647, 0.0), (3.0216667061099725, -0.1323497306155093, 0.0), (2.493271522148583, -0.028722315359088816, 0.0), (2.427717683553695, 0.02520677925493084, 0.0), (3.1277448333301665, 0.34420850079287063, 0.0), (3.3976250872716567, -0.49228559223506474, 0.0)}
		{
			\fill[fill=red] \p circle[radius=0.7pt] ;
		}
		\fill (x) circle[radius=1.5pt] node[anchor=north west] {$\mathbf{x}$};
		\fill (k) circle[radius=2pt] node[anchor=west] {$\mathbf{k}$};
		\end{tikzpicture}
		\caption{Geometrical explanation of Theorem \ref{theorm:1}. All points on $H$ that are closer to $\bfx$ than $\bfx_t$ are within a red disk, which is the intersection of $H$ and the red cone whose vertex is $\bfx$. Clearly $\bfk$ is the only intersection point of the cone and the hemisphere $B$. Thus, of all the lines intersecting $B$, only the intersection of the tangent line and $H$ is closest to $\bfx$.
		}
		\label{fig:geo_explain}
		\vspace{-0.5cm}
	\end{wrapfigure}
	There is some experimental evidence that the decision boundaries of DNNs are smooth surfaces with the low curvature \cite{fawzi2016robustness,rahmati2020geoda}. Based on this observation, we approximate the local decision boundary with a hyperplane $H$.
	Fig. \ref{fig:3D_tangent_attack} illustrates our problem in three-dimensional space. We will derive our algorithm from $\mathbb{R}^3$ and show that it can be directly extended to higher-dimensional spaces. Lastly, we address the case of curved decision boundaries. 
	
	As described in Section \ref{sec:motivation}, $\mathbf{x}_{t-1}$ denotes the \textit{boundary sample} that already lies on the decision boundary. We create a virtual hemisphere $B$ centered at $\mathbf{x}_{t-1}$ with a radius $R$, which is an estimation of the safe region where we can search for adversarial examples. In the targeted attack, $B$ represents the current estimation of the target class region around $\mathbf{x}_{t-1}$. Ideally, $B$ is small enough to be fully contained in the target class's region, \textit{i.e.,} $\phi(\mathbf{x}^\prime) = 1$ for $\forall \mathbf{x}^\prime\in B$. Then, we need to find a point on $B$ which would produce the minimum distortion when mapped to $H$. In two-dimensional case, this point is the tangent point. However, in $n$-dimensional space where $n\ge 3$, there are infinitely many tangent lines of $B$ passing through $\mathbf{x}$ which create infinitely many tangent points on $B$, shown as the red points in Figs. \ref{fig:geo_explain} and \ref{fig:3D_tangent_attack}. Still, we will show that exactly one tangent point can lead to the minimum distortion when mapping it onto $H$ along the tangent line.
	
	Formally, let $\bfk$ be any point on the surface of $B$, $\mathbf{u}$ be the approximate gradient of $H$ estimated at $\mathbf{x}_{t-1}$, and $\bfx_t$ be the intersection of $H$ and the line passing through $\bfx$ and $\bfk$, we have the following theorem.
	
	\begin{theorem}
		Let $H$, $\bfu$, $\mathbf{k}$, $\mathbf{x}$, and $\mathbf{x}_{t-1}$ be defined above, then the distance $\| \mathbf{x} - \mathbf{x}_t\|$ is minimized if $\mathbf{k}$ is the optimal solution of the following constrained optimization problem:\label{theorm:1}
		\begin{align}
		\arg \max_{\mathbf{k}} \quad & \vecprod{\mathbf{k} - \mathbf{x}_{t-1},\mathbf{u}} \label{eqn:objective}\\
		\rm{s.t.} \quad & \vecprod{\mathbf{k} - \mathbf{x}_{t-1}, \mathbf{x}-\mathbf{k}} = 0, \label{eqn:tangent}\\
		& \| \mathbf{k} - \mathbf{x}_{t-1} \| =  R,  \label{eqn:sphere}\\
		& \vecprod{\mathbf{k} - \mathbf{x}_{t-1}, \mathbf{u}} \geq  0. \label{eqn:valid}
		\end{align}
		
		In particular, the optimal $\bfk$ is in the plane spanned by $\bfu$ and $\bfx-\bfx_{t-1}$.
	\end{theorem}
	The objective function of Eq. \eqref{eqn:objective} is to maximize the projection of the vector $\mathbf{k} - \mathbf{x}_{t-1}$ onto $\mathbf{u}$, which is equivalent to finding $\mathbf{k}$ that is farthest away from $H$. The first constraint ensures that $\mathbf{k}$ is a tangent point. The second constraint indicates $\mathbf{k}$ is on the surface of $B$. The last constraint states that $\mathbf{k}$ cannot appear on the same side of $H$ as $\mathbf{x}$, which is always satisfied if $\| \Pi_H (\mathbf{x} - \mathbf{x}_{t-1}) \| \ge R$, where the notation $\Pi_H:\mathbb{R}^n\mapsto H$ denotes the orthogonal projection from $\mathbb{R}^n$ onto the hyperplane $H$. If there is no feasible solution, then our algorithm (Algorithm \ref{alg:tangent_attack}) reduces the radius $R$ to guarantee that the last constraint is always satisfied. The formal proof of Theorem \ref{theorm:1} is presented in the appendix. Here, we refer the readers to Fig. \ref{fig:geo_explain} for an intuitive geometrical explanation. Eq.~\eqref{eqn:objective} is computationally expensive to solve in the high-dimensional space. Fortunately, we show there actually exists a closed-form solution.
	
	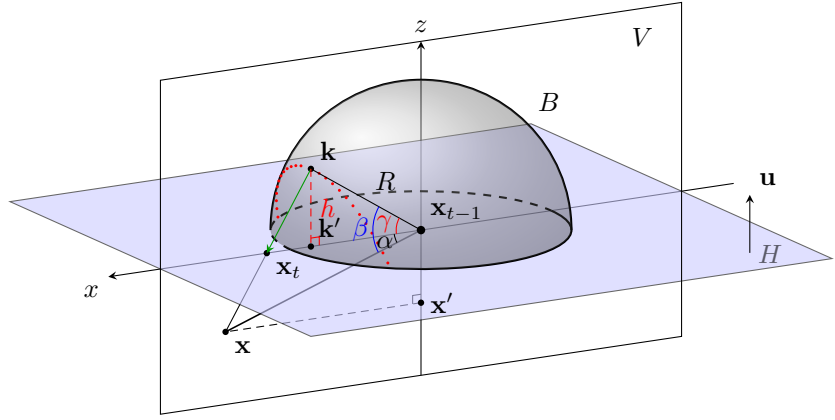
\begin{figure}[t]
		\centering
		\def\r{2.5}
		\tdplotsetmaincoords{75}{150}
		\tikzset{zxplane/.style={canvas is zx plane at y=#1}}
		\tikzset{yxplane/.style={canvas is yx plane at z=#1}}
		\tikzset{yzplane/.style={canvas is yz plane at x=#1}}
		\begin{tikzpicture}[tdplot_main_coords,scale=0.8]
		\begin{scope}[thin]
		\draw[-stealth] (-2.4*\r,0,0) -- (2.4*\r,0,0) node[anchor=north east,thin] {$x$};
		\draw[-stealth] (0,0,-\r) -- (0,0,1.3*\r) node[anchor=south,thin] {$z$};
		\coordinate (O) at (0,0,0);
		\coordinate (x) at  (1.5*\r,0,-0.5*\r);
		\coordinate (x_prime) at (0,0,-0.5*\r);
		\fill (x) circle[radius=1.5pt] node[anchor=north west] {$\mathbf{x}$};
		\coordinate (x_k) at  (2.9588, 0.0000, 0.0000);
		\draw[semithick] (O) -- (x);
		\draw[densely dashed](x_prime) -- (x); 
		\zjbj[4pt]{O}{x_prime}{x};  
		\draw (x) -- (x_k);
		\filldraw[fill=blue!20,opacity=0.6,draw,thin] (-2*\r,-2*\r,0) -- (-2*\r,2*\r,0) -- (2*\r,2*\r,0) -- (2*\r,-2*\r,0) -- cycle node[black] at (-1.7*\r,1.7*\r,0) {$H$};
		\draw[thin] (-2*\r,0,-\r) -- (-2*\r,0,1.3*\r) -- (2*\r,0,1.3*\r) -- (2*\r,0,-\r) -- cycle node[black] at (-1.7*\r,0,1.1*\r) {$V$};
		\draw[tdplot_screen_coords,thick] (-\r,0,0) arc (180:0:\r) node at (0,0,-0.85*\r) {$B$}; 
		
		\draw[thick,dashed,canvas is xy plane at z = 0] (\r,0) arc (0:-180:\r);
		\draw[thick,canvas is xy plane at z = 0] (\r,0) arc (0:180:\r);
		
		\shade[
		tdplot_screen_coords,hemispherefront,
		delta angle=-180,
		x radius=\r,
		] (\r,0)
		arc [y radius={\r*sin(15)},start angle=0,delta angle=-180]
		arc [y radius=\r,start angle=-180];
		
		\fill (O) circle[radius=2pt] node[anchor=south west] {$\mathbf{x}_{t-1}$};
		
		\end{scope}

		\foreach \t in {
			(1.672436339172634,-1.8581326888377487,0.01730901751790337),
			(1.727966735976753,-1.7973067831845102,0.1839002079302603),
			(1.7817782139717986,-1.7193052030714706,0.34533464191539653),
			(1.833385194997704,-1.6243592944799572,0.5001555849931143),
			(1.8822731951543525,-1.5128357619599095,0.6468195854630585),
			(1.9279015388668483,-1.3852807406276293,0.7837046166005444),
			(1.969710674672496,-1.2424648168004933,0.9091320240174889),
			(2.007134491707883,-1.0854243747041374,1.02140347512365),
			(2.0396175650718544,-0.915493765494462,1.1188526952155635),
			(2.0666366138240595,-0.7343224623692498,1.1999098414721785),
			(2.0877247091794486,-0.5438719170901557,1.263174127538347),
			(2.1024960591240482,-0.34638849201425864,1.307488177372145),
			(2.1067089022694705,0.262760497544315,1.320126706808413),
			(2.0946545208887875,0.46245000580091183,1.2839635626663635),
			(2.0761523448712387,0.6561133705341866,1.2284570346137167),
			(2.051547519544043,0.8413996296933663,1.1546425586321296),
			(2.02127618120204,1.01622410378527,1.0638285436061214),
			(1.9858412021922451,1.1788058629494256,0.9575236065767356),
			(1.9457874435293288,1.3276805157551475,0.8373623305879869),
			(1.9016789319467287,1.4616909236570337,0.7050367958401867),
			(1.8540797337114319,1.5799605760108109,0.5622392011342965),
			(1.8035395458873662,1.6818553566894257,0.4106186376620991),
			(1.7505843264231888,1.7669394311946198,0.2517529792695673),
			(1.6957117368765653,1.8349302876369933,0.08713521062969687)}
		{\fill[fill=red] \t circle[radius=0.7pt];
		}
		\fill (x_k) circle[radius=1.5pt] node[anchor=north west] {$\mathbf{x}_{t}$};
		\coordinate  (k) at (2.1124, 0.0000, 1.3371);
		\fill (k) circle[radius=1.5pt] node[anchor=south west] {$\mathbf{k}$};
		\draw[black!40!green,-stealth] (k) -- (x_k);
		\draw[black] (O) -- node [anchor=south west]{$R$} (k);
		\coordinate(k_prime) at (2.1124, 0.0000, 0);
		\fill (k_prime) circle[radius=1.5pt] node at (1.75, 0.0000, 0.28){$\mathbf{k}^\prime$};
		\fill (x_prime) circle[radius=1.5pt] node [anchor=west]{$\mathbf{x}^\prime$};  
		\draw[red,densely dashed] (k) -- node[right,text=red] {$h$} (k_prime);
		\zjbjcolor[4pt]{O}{k_prime}{k};  
		\draw[black,-stealth] (-1.6*\r,1.6*\r,0) -- (-1.6*\r,1.6*\r,1.0) node[anchor=south west] {$\mathbf{u}$}; 
		\pic [draw=red,text=red, "$\gamma$", angle eccentricity=1.5,angle radius=0.33cm] {angle = k--O--k_prime};
		\pic [draw=blue,text=blue, "$\beta$", angle eccentricity=1.25,angle radius=0.64cm] {angle = k--O--x};
		\pic [draw=black,text=black, "$\alpha$", angle eccentricity=1.4,angle radius=0.36cm] {angle = k_prime--O--x};
		\end{tikzpicture}
		\caption{The illustration of the optimal tangent point $\mathbf{k}$ for the flat decision boundary. The tangent point $\bfk$ is on the surface of a virtual hemisphere $B$ with a radius $R$ centered at the adversarial example $\mathbf{x}_{t-1}$. $\mathbf{x}$ is the original image, $\mathbf{x}_t$ is the intersection point of the tangent line and the decision hyperplane $H$, $\mathbf{k}^\prime$ and $\mathbf{x}^\prime$ are the orthogonal projections of $\mathbf{k}$ and $\mathbf{x}$ onto $H$ and $z$ axis, respectively.}
		\label{fig:3D_tangent_attack}
	\end{figure}
	
	\subsection{Closed-Form Solution of the Optimal Tangent Point}
	
	The main intuition of the derivation is illustrated in Fig. \ref{fig:3D_tangent_attack}, which shows an example in $\mathbb{R}^3$. For ease of presentation, we move $\bfx_{t-1}$ to $\mathbf{0}$. The known variables are $\mathbf{x}$, $\mathbf{x}_{t-1}$, the unit normal vector $\mathbf{u}$ of hyperplane $H$, and the radius $R$. We need to solve for the unknown $\mathbf{k}$. Let $V = \text{span}(\{\mathbf{x},\mathbf{u}\})$ be the plane spanned by $\mathbf{x}$ and $\mathbf{u}$. According to Theorem~\ref{theorm:1}, we know $\mathbf{k} \in V$. Let us denote the angle between $\mathbf{x}$ and $H$ as $\alpha$, the angle between $\mathbf{x}$ and $\mathbf{k}$ as $\beta$, and the angle between $\mathbf{k}$ and $H$ as $\gamma$. Then, $\beta = \alpha + \gamma$ because all three points $\bfx$, $\bfk$, and $\bfx_{t-1}$ are on the same plane $V$. Because the angle between $\mathbf{x}$ and $\mathbf{u}$ is $\pi\mathbin{/} 2 + \alpha$, we have
	\begin{equation}
	\vecprod{\mathbf{x}, \mathbf{u}} = \| \mathbf{x} \| \cdot \| \mathbf{u} \| \cdot \cos (\frac{\pi}{2} + \alpha) = \| \mathbf{x} \| \cdot \| \mathbf{u} \| \cdot \left(-\sin \alpha\right).
	\label{eqn:x_u_inner_prod}
	\end{equation}
	Then we have
	$\sin \alpha = -\frac{\vecprod{\mathbf{x}, \mathbf{u}}}{\| \mathbf{x} \| \cdot \| \mathbf{u} \|}$
	and
	$\cos \alpha = \sqrt{1- \sin^2 \alpha} = \frac{\sqrt{\|\mathbf{x} \|^2 \cdot \|\mathbf{u} \|^2 - \vecprod{\mathbf{x}, \mathbf{u}}^2}}{\| \mathbf{x} \| \cdot \| \mathbf{u} \|}$. By the constraint~\eqref{eqn:tangent}, $\mathbf{x}-\mathbf{k}$ is orthogonal to $\mathbf{k}$. Thus, we have $\cos \beta = R \mathbin{/} \| \mathbf{x} \|$ and $\sin \beta = \sqrt{1-\cos^2 \beta} = \sqrt{\| \mathbf{x} \|^2 - R^2}\mathbin{/} \| \mathbf{x} \|$.
	Then $\sin \gamma$ and $\cos \gamma$ can be derived as functions of $\alpha$ and $\beta$ from basic facts of trigonometric functions:  
	\begin{equation}
	\begin{aligned}
	\sin \gamma &= \sin (\beta- \alpha) = \sin\beta \cos\alpha - \cos\beta \sin\alpha, \\
	\cos \gamma &= \cos (\beta- \alpha) = \cos\beta \cos\alpha + \sin\beta \sin\alpha.
	\end{aligned}
	\label{eqn:sin_cos_gamma}
	\end{equation}
	Now, let $\mathbf{k}^\prime \in H$ be the orthogonal projection of $\mathbf{k}$ onto the plane $H$. The distance between $\bfk$ and $\bfk'$ is denoted as $h$ (Fig. \ref{fig:3D_tangent_attack}). Then $h = R \cdot \sin\gamma = R \cdot (\sin\beta \cos\alpha - \cos\beta \sin\alpha)$.
	
	
	
	To derive $\mathbf{k}^\prime$, let us denote $\mathbf{x}^\prime$ as the orthogonal projection of $\mathbf{x}$ onto $z$ axis. So we have $\mathbf{x}^\prime = \vecprod{\mathbf{x}, -\mathbf{u}} \cdot (-\mathbf{u}) \mathbin{/} \|\mathbf{u} \|^2 =\vecprod{\mathbf{x},\mathbf{u}} \cdot \mathbf{u} \mathbin{/} \|\mathbf{u} \|^2$. Then, because $\mathbf{k}^\prime$ and $\mathbf{x} - \mathbf{x}^\prime$ are on the same direction, we have
	\begin{equation}
	\frac{\mathbf{k}^\prime}{\| \mathbf{k}^\prime\|} = \frac{\mathbf{x} - \mathbf{x}^\prime}{\| \mathbf{x} - \mathbf{x}^\prime \|}.
	\label{eqn:normalized_k_prime}
	\end{equation}
	Now, $\| \mathbf{k}^\prime\| = R \cdot \cos \gamma$ and $\mathbf{x}^\prime = \vecprod{\mathbf{x},\mathbf{u}} \cdot \mathbf{u} \mathbin{/} \|\mathbf{u} \|^2 $ are plugged into Eq. \eqref{eqn:normalized_k_prime}, and $\mathbf{k}^\prime$ is obtained as
	\begin{equation}
	\mathbf{k}^\prime = \frac{\mathbf{x} - \mathbf{x}^\prime}{\| \mathbf{x} - \mathbf{x}^\prime \|} \cdot \| \mathbf{k}^\prime\| = \frac{\mathbf{x} - \vecprod{\mathbf{x},\mathbf{u}} \cdot \mathbf{u} \mathbin{/} \|\mathbf{u} \|^2 }{\| \mathbf{x} - \vecprod{\mathbf{x},\mathbf{u}} \cdot \mathbf{u} \mathbin{/} \|\mathbf{u} \|^2 \|} \cdot R \cdot \cos \gamma.
	\label{eqn:k_prime}
	\end{equation}
	Therefore, $\mathbf{k}$ can be derived as
	\begin{equation}
	\mathbf{k} = \mathbf{k}^\prime + h \cdot \mathbf{u} = \frac{\mathbf{x} - \vecprod{\mathbf{x},\mathbf{u}} \cdot \mathbf{u} \mathbin{/} \|\mathbf{u} \|^2 }{\| \mathbf{x} - \vecprod{\mathbf{x},\mathbf{u}} \cdot \mathbf{u} \mathbin{/} \|\mathbf{u} \|^2 \|} \cdot R \cdot \cos \gamma + h \cdot \mathbf{u}.
	\label{eqn:k}
	\end{equation}
	Finally, because $\mathbf{x}_{t-1}$ has been moved to the origin, we need to move $\mathbf{k}$ back by adding $\mathbf{x}_{t-1}$. 
	
	We remark that although the above derivation is illustrated in $\mathbb{R}^3$, it can be directly applied to higher dimensions. The reason is Theorem~\ref{theorm:1}, which essentially reduces any dimension space to $\mathbb{R}^2$: to find the optimal $\bfk$, we only need to focus on the plane $V$ spanned by $\bfu$ and $\bfx-\bfx_{t-1}$.
	
	\subsection{Generalized Tangent Attack}
	\begin{figure}[t]
		\centering
		\def\r{2.0}
		\tikzset{
			partial ellipse/.style args={#1:#2:#3}{
				insert path={+ (#1:#3) arc (#1:#2:#3)}
			}
		}
		\begin{minipage}[t]{.49\textwidth}
			\centering
			
			\begin{tikzpicture}
			\def\lr{9.0 }
			\def\sr{2.0 }
			\begin{axis}[
			hide axis,
			domain = -4:4,
			zmax   = 12,
			colormap/cool,
			ymin=-5,ymax=5,xmax=5,xmin=-5,zmin=-3,view={180}{25},
			]
			\draw[thick,canvas is xz plane at y = 0] (0,0) [partial ellipse=0:180:\sr and \lr];
			\shade[
			canvas is xz plane at y = 0,ellipsoidfront,
			delta angle=-180,
			x radius=\sr,
			] (\sr,0)
			arc [y radius={\sr * sin(45)},start angle=0,delta angle=-180]	 
			(0,0) [partial ellipse=0:180:\sr and \lr];
			\draw[thick,dashed,canvas is xy plane at z = 0] (\sr,0) arc (0:-180:\sr);
			\draw[thick,canvas is xy plane at z = 0] (\sr,0) arc (0:180:\sr);
			\end{axis}
			\begin{axis}[
			hide axis,
			domain = -4:4,
			zmax   = 12,
			colormap/cool,
			ymin=-5,ymax=5,xmax=5,xmin=-5,zmin=-3,view={150}{25},
			]
			\def\lr{9.0 }
			\def\sr{3.0 }
			\coordinate (x) at  (3.2,0,-2.5);
			\coordinate (O) at (0,0,0);
			\coordinate (k_prime) at (1.4725863759467117, 0.0, 0); 
			\coordinate (k) at (1.4725863759467117, 0.0, 6.428295171061992);
			\draw[-stealth] (-4,0,0) -- (4,0,0) node[anchor=north east,thin] {$x$};
			\draw[-stealth] (0,0,-4) -- (0,0,12) node[anchor=south,thin] {$z$};
			
			\draw[red,thick] (O) -- (k_prime);
			\fill (O) circle[radius=2pt] node[anchor=north west] {$\mathbf{x}_{t-1}$};
			\draw[red,densely dashed] (k) -- node[right,text=red] {$h$} (k_prime);
			\zjbjcolor[4pt]{O}{k_prime}{k};  
			
			\draw[blue] (O)-- (-1.3,2.7,0);
			\node[text=blue] (S) at (-1.3,1.1,0) {$S$};
			\node[text=blue] (L) at (-0.5,0,4.0) {$L$};
			\draw[blue] (O) -- (0,0,\lr);
			\draw (k) -- (x);
			\draw[black!60!green] (x) -- (O);
			\addplot3 [samples=30,surf,fill=blue!20,opacity=0.3] {(x^2+y^2)/3.5};
			\draw[thin] (-3.5,0,-3) -- (-3.5,0,11) -- (3.5,0,11) -- (3.5,0,-3) -- cycle node[black] at (-3,0,10) {$V$};

			\fill (x) circle[radius=1.5pt] node[anchor=north west] {$\mathbf{x}$};
			
			\fill (k)circle[radius=1.5pt] node[anchor=south west] {$\mathbf{k}$};
			
			\fill (k_prime)circle[radius=1.5pt];
			
			\coordinate (w) at (0.7,0,-0.6);
			\coordinate (z) at (0,0,-1);
			\pic [draw=black!60!green,text=black!60!green, "$\theta$", angle eccentricity=1.8,angle radius=0.2cm] {angle = x--O--z};
			\node (H) at (-4.2,0,3.5) {$H$};
			\end{axis}
			\end{tikzpicture}
			\subcaption{3D view of Generalized Tangent Attack}
			\label{fig:3D_generalized_tangent_attack}
		\end{minipage}
		\begin{minipage}[t]{.49\textwidth}
			\centering
			\begin{tikzpicture}
			\newcommand{\boundellipse}[3]
			{(#1) ellipse [x radius=#2,y radius=#3]
			}
			\newcommand\zjbjblue[4][7pt]{%
				\draw[color=blue] let \p1=(#2),\p2=(#4),\p0=(#3) in
				(#3)++({atan2(\y1-\y0,\x1-\x0)}:#1)
				--++({atan2(\y2-\y0,\x2-\x0)}:#1)
				--++({atan2(\y1-\y0,\x1-\x0)}:-#1);
			}
			
			\draw[-stealth] (5,3) -- (2,3) node[anchor=east] {$\mathbf{v}$};
			\coordinate (x) at  (2,2);
			\coordinate (O) at  (4,3);
			\coordinate (x_projection) at (1,3);
			\draw[black!60!green](x) -- (O);
			\fill (x) circle[radius=1.5pt] node[anchor=north] {$\mathbf{x} (x_0, z_0)$};
			
			
			
			\draw[thick] (O) [partial ellipse=0:180:1 and 2];
			\draw[densely dotted,thick] (O) [partial ellipse=0:-180:1 and 2];
			

			\coordinate (H) at (3.1339745962155616,3.5);
			\draw[-stealth] (4,2) -- (4,5.5) node[anchor=south] {$\mathbf{u}$};
			\draw[decorate,decoration={brace,raise=0.8pt,amplitude=3pt,mirror},draw=blue] (O) -- (4,5);
			\draw[decorate,decoration={brace,raise=0.3pt,amplitude=3pt,mirror},draw=blue] (O) -- (3,3);
			\coordinate (tangent) at (3.31732051, 4.46143589);
			\coordinate (another_tangent) at (3.74150302, 1.06797587);
			\fill (tangent) circle[radius=1.5pt] node[anchor=south east] {$\mathbf{k} (x_k, z_k)$};
			\fill (another_tangent) circle[radius=1.5pt] node[anchor=north east] {$\mathbf{k}^\prime (x_{k^\prime}, z_{k^\prime})$};
			\draw(x) -- (tangent);
			
			\node[blue] at (4.3,4) {$L$};
			\node[blue] at (3.5,3.3) {$S$};
			\coordinate (B) at (4,2);
			\pic [draw=black!60!green,text=black!60!green, "$\theta$", angle eccentricity=1.5,angle radius=0.3cm] {angle = x--O--B};
			\fill (O) circle[radius=1.5pt];
			\end{tikzpicture}
			\subcaption{2D plane $V$ spanned by $\bfx$ and $\bfu$.}
			\label{fig:2D_plane}
		\end{minipage}
		\caption{Illustration of the derivation of Generalized Tangent Attack, which replaces the hemisphere $B$ with a semi-ellipsoid to increase the height of $\mathbf{k}$ to adapt to curved decision boundaries.}
		\label{fig:ellipsoid}
	\end{figure}
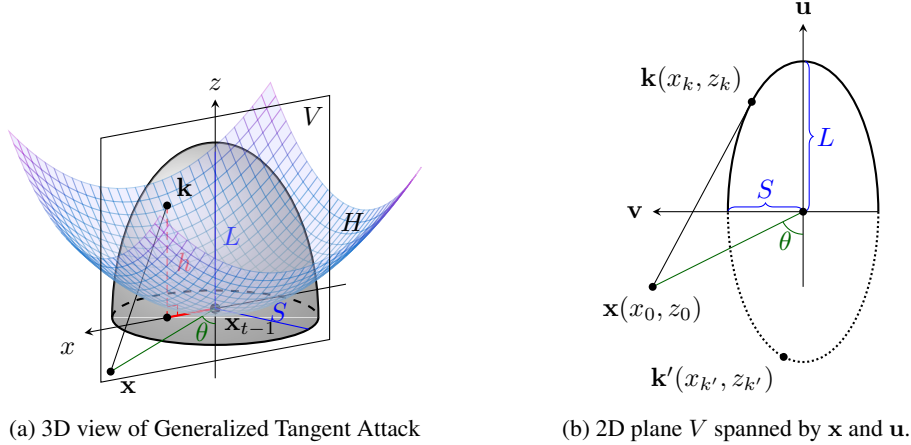
	When the local decision boundary is not flat enough, the boundary point obtained via the tangent line may not be the optimal, as shown in Fig.~\ref{fig:3D_generalized_tangent_attack}. A simple solution is to simply continue halving the radius $R$ of the hemisphere: as long as $R$ becomes small enough, the local flatness can always be obtained. However, too small a radius will reduce the convergence rate of our algorithm, because the distance between $\mathbf{x}_t$ and $\mathbf{x}_{t-1}$ is proportional to $R$. 
	Therefore, when the classification decision boundary is a curved surface, the attack algorithm should change the shape of hemisphere rather than simply reducing $R$. Based on this idea, we propose the Generalized Tangent Attack (G-TA).
	
	First, although the shape of the decision boundary can be very complex in a high-dimensional space, the important thing for our algorithm is only the situation in the two-dimensional plane spanned by $\mathbf{x}$ and $\mathbf{u}$.
	In particular, if the decision boundary is ``downward'' curved (as opposed to the example in Fig.~\ref{fig:3D_generalized_tangent_attack}), then searching along the tangent line is still a better approach than HSJA's solution. Thus, the only 
	``bad case'' we have to deal with is when the decision boundary is ``upward'' curved and has a large curvature, as shown in Fig.~\ref{fig:3D_generalized_tangent_attack}. 
	
	According to Theorem \ref{theorm:1}, we only need to focus on the plane $V$ spanned by $\mathbf{x}$ and $\mathbf{u}$, as shown in Fig.~\ref{fig:2D_plane}.
	Now, $\mathbf{u}$ and $\mathbf{v}:= (\mathbf{x} - \vecprod{\mathbf{x},\mathbf{u}} \cdot \mathbf{u} \mathbin{/} \|\mathbf{u} \|^2 )\mathbin{/} \|\mathbf{x} - \vecprod{\mathbf{x},\mathbf{u}} \cdot \mathbf{u} \mathbin{/} \|\mathbf{u} \|^2 \|$ form an orthogonal basis of the plane $V$, then $\mathbf{x}$ can be identified with coordinates $(x_0,z_0)$, \textit{i.e.,} $\mathbf{x}=x_0\mathbf{v}+z_0\mathbf{u}$.
	Let $\theta$ denote the angle between the vector $\mathbf{x}$ and the vector $-\mathbf{u}$, \textit{i.e.,} $\theta = \arccos \left(\frac{\vecprod{\mathbf{x},-\mathbf{u}}}{\|\mathbf{x}\| \cdot \| \mathbf{u} \|}\right)$. Then $(x_0, z_0) = (\| \bfx \| \cdot \sin \theta, -\| \bfx \| \cdot \cos \theta)$. 	
	Consider the projection of the ellipsoid on $V$ (which is an ellipse), we denote $L$ as its radius along the direction of $\bfu$, and $S$ as its radius along the direction of $\bfv$. Because the optimal tangent point $\mathbf{k}$ lies in the plane $V$, $\mathbf{k}$ can also be identified as $\bfk=x_k\bfv+z_k\bfu$, and we only need to solve for the unknown $(x_k,z_k)$.

	The ellipse is characterized by the equation $x^2 \mathbin{/} S^2 + z^2 \mathbin{/} L^2 = 1$, thus the tangent point $\mathbf{k}$ satisfies $x_k^2 \mathbin{/} S^2 + z_k^2 \mathbin{/} L^2 = 1$. Now we view $z$ as a function of $x$, and take the derivative w.r.t. $x$ at both sides of the equation to get the following formula:
	\begin{equation}
	\frac{2 x_k}{S^2} + \left.\frac{2 z_k}{L^2} \cdot \frac{d z}{d x} \right|_{x=x_k} = 0.
	\label{eqn:both_derivative}
	\end{equation}
	Thus, we have the slope of tangent line at $\mathbf{k}$ be $\left.\frac{d z}{d x} \right|_{x=x_k} = -\frac{x_k L^2}{z_k S^2}$.
	Therefore, the tangent line can be written as $z-z_k = - \frac{x_k L^2}{z_k S^2} (x-x_k)$.
	Since the tangent line passes through $\bfx$, we know $z_0-z_k = -\frac{x_k L^2}{z_k S^2} (x_0 - x_k)$.
	In summary, we can obtain the following system of equations:
	\begin{equation}
	\left\{ \begin{lgathered} 
	L^2 x_k^2 + S^2 z_k ^2  - x_k x_0 L^2 - z_0 z_k S^2 = 0, \\
	\frac{x^2_k}{S^2} + \frac{z^2_k}{L^2} = 1.
	\end{lgathered}\right.
	\label{eqn:eqn_system}
	\end{equation}
	In Eq. \eqref{eqn:eqn_system}, the known variables are $L$, $S$, $x_0$ and $z_0$, and the unknown variables that we need to solve for are $x_k$ and $z_k$. 
	In general, there are two solutions for Eq. \eqref{eqn:eqn_system}, \textit{i.e.,} $\mathbf{k}$ and $\mathbf{k}'$ depicted in Fig.~\ref{fig:2D_plane}.
	Apparently, the solution of the optimal tangent point should satisfy $z_k > 0$, so the one of two solutions in which $z_k > 0$ should be picked:
	\begin{equation}
	\left\{\begin{aligned}
	x_k &= \frac{S^2 \left(L^2 - z_0 \cdot \frac{L^2 S^2 z_0 + L^2 x_0 \sqrt{-L^2 S^2 +L^2 x_0^2 + S^2 z_0^2}}{L^2 x_0^2 + S^2 z_0^2}\right)}{L^2 \cdot x_0}, \\
	z_k &= \frac{L^2 S^2 z_0 + L^2 x_0 \sqrt{-L^2 S^2 +L^2 x_0^2 + S^2 z_0^2}}{L^2 x_0^2 + S^2 z_0^2}. 
	\end{aligned}\right.
	\label{eqn:ellipsoid_solution}
	\end{equation}
	Finally, the optimal tangent point $\mathbf{k}$ is obtained via $\mathbf{k}= \vert x_k \vert \cdot \frac{\mathbf{x} - \vecprod{\mathbf{x},\mathbf{u}} \cdot \mathbf{u} \mathbin{/} \|\mathbf{u} \|^2 }{\| \mathbf{x} - \vecprod{\mathbf{x},\mathbf{u}} \cdot \mathbf{u} \mathbin{/} \|\mathbf{u} \|^2 \|} +z_k \cdot \mathbf{u}$.
	In the implementation, the value of $L$ is determined in the same way as the radius $R$ in TA (the hemisphere version). So we fix $L=R$, and use a hyperparameter $r =L/S$ to control the value of $S$. Imagine that in the case of $\mathbb{R}^3$, the semi-ellipsoid becomes ``slender'' by setting $r>1$, thereby adapting to the decision boundary of curved surface while preserving a relatively large step size.
	
	
	\subsection{The Complete Algorithm}
	TA (the hemisphere version) and G-TA (the semi-ellipsoid version) can be combined into one algorithm process, which is shown in Algorithm \ref{alg:tangent_attack}. It first performs a binary search to map the initial sample $\tilde{\mathbf{x}}_0$ to the decision boundary.
	Note that the binary search step always maps any $\mathbf{x}_\text{adv}$ to the adversarial side of $H$ that satisfies $\phi(\mathbf{x}_\text{adv})=1$, hence the attack success rate is always 100\%.
	Then, it performs a \textit{for} loop of $T$ iterations to find the adversarial example that is close to $\mathbf{x}$. In the first iteration, we sample $B_0$ probes around the boundary sample to estimate the gradient, which is increased to $B_0\sqrt{t}$ at the $t$-th iteration. 
	This is because the error of gradient estimation in the later iterations has a greater impact on the attack performance, so using more samples can reduce the estimation error. Then, a \textit{while} loop is performed to determine a reasonable radius $R$ by repeatedly halving the radius until the tangent point $\mathbf{k}$ is in the adversarial region. Finally, Algorithm \ref{alg:tangent_attack} uses the binary search method to map $\mathbf{k}$ back to the classification decision boundary to end this iteration.
	\begin{algorithm}[!ht]
		\caption{Tangent Attack}
		\begin{algorithmic}
			\State {\bfseries{Input:}} benign image $\mathbf{x}$, attack success indicator function $\phi(\cdot)$ defined in Eq. \eqref{eqn:phi}, initial batch size $B_0$, iteration $T$, mode $m \in$ \{semi-ellipsoid, hemisphere\}, radius ratio $r$.
			\State Initialize $\tilde{\mathbf{x}}_0$ that satisfies $\phi(\tilde{\mathbf{x}}_0) = 1$;
			\State $\mathbf{x}_0 \gets \text{BinarySearch}(\tilde{\mathbf{x}}_{0},\mathbf{x},\phi)$; \Comment{boundary search}
			\State $d_0 = \| \mathbf{x}_0 - \mathbf{x} \|$;
			\For{$t$ {\bfseries{in}} $1,\dots,T$}
			\State Sample $B_t \gets B_0 \sqrt{t}$ random vectors to estimate the gradient $\mathbf{u}$; \label{line:B_t}
			\State Initialize $R \gets d_{t-1}/\sqrt{t}$; \Comment{the initial radius}
			\While{\textbf{true}} 
			\State Compute the optimal tangent point $\bfk$ based on Eq.~\eqref{eqn:k} \textbf{if} $m=\text{hemisphere}$ \textbf{else} Eq.~\eqref{eqn:ellipsoid_solution};
			\State $R \gets \frac{R}{2}$; \Comment{search the radius, and we set $L=R, S=\frac{L}{r}$ if $m=\text{semi-ellipsoid}$}
			\If{$\phi(\mathbf{k}) = 1$}
			\State \textbf{break}; 
			\EndIf
			\EndWhile
			\State $\mathbf{k} \gets \text{Clip}(\mathbf{k}, 0, 1)$;
			\State $\mathbf{x}_t \gets \text{BinarySearch}(\mathbf{k},\mathbf{x},\phi)$; \Comment{boundary search}
			\State $d_t = \| \mathbf{x}_t - \mathbf{x} \|$;
			\EndFor
		\end{algorithmic}
		\label{alg:tangent_attack}
	\end{algorithm}
	
	\section{Experiment}
	\subsection{Experimental Setting}
	\label{sec:expr_setting}
	\textbf{Datasets.} TA and G-TA are evaluated on two datasets, namely CIFAR-10 and ImageNet with the image resolutions of $32\times 32\times 3$ and $299\times 299\times 3$, respectively. We randomly select \nn{1000} correctly classified images from their validation sets for experiments. 
	
	\textbf{Method Setting.} The initial batch size $B_0$ is set to 100, which means the algorithm samples 100 probes for estimating a gradient at the first iteration. The threshold $\gamma$ that controls the termination of the binary search is set to 1.0 in the CIFAR-10 dataset and \nn{1000} in the ImageNet dataset. The radius ratio $r$ is set to 1.5 in the CIFAR-10 dataset and 1.1 in the ImageNet dataset. Besides, we also set $r$ to 1.5 when attacking defense models. In targeted attacks, the target class label is set to $y_{adv} = (y + 1) \mod C$, where $y$ is the true label, and $C$ is the number of classes.
	
	\textbf{Compared Methods.} 
	The advantage of our method is that it supports all types of attacks, including both untargeted and targeted attacks under $\ell_2$ norm and $\ell_\infty$ norm constraints.
	Therefore, for complete and fair comparisons, we select the compared methods that support both untargeted and targeted attacks with state-of-the-art performance, including Boundary Attack (BA) \cite{brendel2018decisionbased}, Sign-OPT \cite{cheng2019sign}, SVM-OPT \cite{cheng2019sign}, and HopSkipJumpAttack (HSJA) \cite{chen2019hopskipjumpattack}. HSJA is adopted as the baseline, whose hyperparameters are set to be the same with ours (\textit{e.g,} the same initial batch size $B_0$ and threshold $\gamma$).
	In addition, QEBA \cite{li2020qeba} is a HSJA-based method which has three variants: QEBA-I, QEBA-S and QEBA-F. We select QEBA-S in the additional experiment to verify whether the proposed method can improve attack performance of other HSJA-based method.
	For the consistency of experiments, we translate the implementations of Sign-OPT, SVM-OPT and HSJA from the official NumPy version into the PyTorch version by replacing each NumPy function with the corresponding PyTorch function. Thus, the two versions behave exactly the same.
	In the targeted attack, we randomly select an image from the target class of the validation set to be the initial sample of HSJA, BA, TA and G-TA. For fair comparison, we set the initial direction $\theta_0$ of Sign-OPT and SVM-OPT to the direction of a randomly selected image of the target class.
	The detailed settings are presented in the appendix.

	\textbf{Target Models.} In the CIFAR-10 dataset, we select four target models implemented using the PyTorch framework\footnote{Pre-trained weights: \url{https://github.com/machanic/SimulatorAttack}}: (1) a 272-layer PyramidNet+ShakeDrop network (PyramidNet-272) \cite{han2017deep,yamada2019shakedrop} trained using AutoAugment \cite{cubuk2019autoaugment} ; (2) a model obtained through a neural architecture search called GDAS \cite{dong2019searching}; (3) a wide residual network with 28 layers and 10 times width expansion (WRN-28) \cite{sergey2016WRN}; and (4) a wide residual network with 40 layers (WRN-40) \cite{sergey2016WRN}. In the ImageNet dataset, we select four target models from an off-the-shelf library containing pre-trained weights\footnote{Pre-trained weights: \url{https://github.com/Cadene/pretrained-models.pytorch}}: (1) Inception-v3 \cite{szegedy2016rethinking}, (2) Inception-v4~\cite{szegedy2017inceptionv4}, (3) ResNet-101 \cite{he2016deep}, and (4) SENet-154 \cite{hu2018squeeze}.
	
	\textbf{Evaluation Metric.} Following previous studies \cite{yan2021policydriven}, we report the mean $\ell_2$ distortions as $\frac{1}{\vert \mathbf{X} \vert}\sum_{\mathbf{x}\in \mathbf{X}}(\| \mathbf{x}_\text{adv}- \mathbf{x}\|)$ under different query budgets, where $\mathbf{X}$ is the test set. 
	\begin{table}[!t]
		\centering
		\caption{Mean $\ell_2$ distortions of different query budgets on the ImageNet dataset, where $r=1.1$.}
		\small
		\tabcolsep=0.1cm
		\setlength{\belowcaptionskip}{0.5pt}%
		\scalebox{0.77}{
			\begin{tabular}{cccccccc|cccccc}
				\toprule
				Target Model  & Method  & \multicolumn{6}{c|}{Targeted Attack} & \multicolumn{6}{c}{Untargeted Attack} \\
				& & @300&  @1K & @2K & @5K & @8K & @10K & @300&  @1K & @2K & @5K & @8K & @10K \\
				\midrule
				\multirow{6}{*}{Inception-v3}&  BA \cite{brendel2018decisionbased} & 111.798 & 108.044 & 106.283 & 102.715 & 86.931 & 78.326 & - & 107.558 & 102.309 & 95.776 & 78.668 & 60.296 \\
				&  Sign-OPT \cite{cheng2019sign} & 103.939 & 87.706 & 71.291 & 46.744 & 34.640 & 29.414 & 121.085 & 79.158 & 43.642 & 16.625 & 10.557 & 8.680\\
				& SVM-OPT \cite{cheng2019sign}&  \B 101.630 & 82.950 & 67.965 & 46.275 & 35.694 & 31.106 & 121.135 & 66.027 & 36.763 & 15.736 & 10.501 & 8.789\\
				& HSJA \cite{chen2019hopskipjumpattack}& 111.562 & 95.295 & 82.111 & 52.544 & 37.395 & 30.425 & 103.605 & 57.295 & 37.185 & 15.484 & 9.989 & 7.967 \\                    
				& TA &  103.781 & \B 80.327 & \B 66.708 & \B 42.121 & \B 30.846 & \B 25.566 &  \B 94.752 & 52.523 & 35.229 & 15.040 & 9.748 & 7.793\\
				& G-TA  &  103.724 & 81.089 & 67.168 & 42.434 & 31.011 & 25.587 & 94.972 & \B 52.278 & \B 34.734 & \B 14.850 & \B 9.673 & \B 7.757\\
				\Xcline{1-14}{0.1pt}
				\multirow{6}{*}{Inception-v4} & BA \cite{brendel2018decisionbased} & 110.343 & 106.616 & 104.586 & 100.321 & 84.058 & 75.507 & - & 116.075 & 111.474 & 104.451 & 86.572 & 66.283 \\
				& Sign-OPT \cite{cheng2019sign} &  101.620 & 85.731 & 69.719 & 46.416 & 34.957 & 30.004 & 132.991 & 86.431 & 48.292 & 18.678 & 11.567 & 9.262\\
				& SVM-OPT \cite{cheng2019sign}& \B 99.856 & 81.342 & 66.982 & 45.667 & 35.477 & 31.152 & 132.227 & 72.920 & 41.095 & 17.611 & 11.418 & 9.372 \\
				& HSJA \cite{chen2019hopskipjumpattack}&  109.670 & 93.916 & 80.937 & 52.358 & 37.773 & 30.958 & 110.727 & 63.731 & 42.290 & 17.936 & 11.367 & 8.911\\
				& TA & 101.666 & \B 78.683 & \B 65.304 & \B 41.629 & 30.993 & 25.958 & \B 101.207 & \B 58.616 & 40.314 & 17.639 & 11.304 & 8.907 \\
				& G-TA & 101.495 & 79.210 & 65.888 & 42.002 & \B 30.965 & \B 25.847 & 101.324 & 58.718 & \B 40.106 & \B 17.296 & \B 11.032 & \B 8.691 \\
				\Xcline{1-14}{0.1pt}
				\multirow{6}{*}{SENet-154} & BA \cite{brendel2018decisionbased} & 81.090 & 77.723 & 76.122 & 71.967 & 55.953 & 47.652 & - & 75.998 & 71.671 & 66.983 & 53.917 & 40.725 \\
				& Sign-OPT \cite{cheng2019sign} &  75.722 & 62.876 & 49.191 & 30.155 & 21.333 & 17.672  & 70.035 & 47.705 & 27.314 & 10.890 & 6.643 & 5.245\\
				& SVM-OPT \cite{cheng2019sign}& 75.680 & 58.932 & 47.073 & 30.348 & 22.553 & 19.312  & 69.854 & 40.291 & \B 23.692 & 10.494 & 6.666 & 5.409 \\
				& HSJA \cite{chen2019hopskipjumpattack}&  77.035 & 63.488 & 51.802 & 30.138 & 19.680 & 16.261 & 71.248 & 38.035 & 24.895 & 10.218 & 5.855 & 4.842\\
				& TA &  70.739 & 55.256 & \B 43.694 & \B 24.961 & \B 16.756 & \B 13.876 & \B 65.589 &  35.689 & 24.037 & 10.039 & 5.774 & 4.766\\
				& G-TA &  \B 70.591 & \B 55.224 & 44.047 & 25.041 & 16.854 & 14.047 & 65.846 &\B 35.601 &  23.730 & \B 9.902 & \B 5.720 & \B 4.738\\
				\Xcline{1-14}{0.1pt}
				\multirow{6}{*}{ResNet-101} & BA \cite{brendel2018decisionbased} & 81.565 & 77.903 & 76.366 & 72.392 & 58.746 & 51.679 & - & 64.007 & 60.389 & 56.544 & 44.175 & 31.371 \\
				& Sign-OPT \cite{cheng2019sign} & 76.732 & 63.939 & 51.231 & 32.439 & 23.160 & 19.248 & 56.244 & 38.282 & 21.985 & 10.048 & 7.050 & 6.050 \\
				& SVM-OPT \cite{cheng2019sign} &  77.031 & 61.417 & 49.842 & 32.806 & 24.553 & 20.964 & 55.894 & 32.638 & 19.409 & 9.830 & 7.185 & 6.281\\
				& HSJA \cite{chen2019hopskipjumpattack}& 76.121 & 63.091 & 52.301 & 31.018 & 20.472 & 16.911 & 56.264 & 27.443 & 17.717 & 7.649 & 4.723 & \B 4.019 \\
				& TA & \B 72.434 & \B 57.969 & \B 47.142 & \B 27.699 & \B 18.788 & \B 15.414 & 53.197 & 26.777 & 17.651 & 7.730 & 4.822 & 4.107 \\
				& G-TA &  72.459 & 58.320 & 47.297 & 27.905 & 19.045 & 15.633 & \B 53.142 & \B 26.597 & \B 17.345 & \B 7.568 & \B 4.712 & 4.021 \\
				\bottomrule
		\end{tabular}}
		\label{tab:ImageNet_normal_models_result}
	\end{table}
	\begin{table}[!ht]
		\centering
		\caption{Mean $\ell_2$ distortions with different query budgets on the CIFAR-10 dataset, where $r=1.5$.}
		\small
		\tabcolsep=0.1cm
		\setlength{\belowcaptionskip}{0.5pt}%
		\scalebox{0.8}{
			\begin{tabular}{cccccccc|cccccc}
				\toprule
				Target Model  & Method & \multicolumn{6}{c|}{Targeted Attack} & \multicolumn{6}{c}{Untargeted Attack} \\
				& & @300 & @1K & @2K & @5K & @8K & @10K &@300 & @1K & @2K & @5K & @8K & @10K \\
				\midrule
				\multirow{6}{*}{PyramidNet-272} & BA \cite{brendel2018decisionbased}& 8.651  & 8.073 & 8.013 & 6.387 & 4.189 & 3.333 & - & 5.636 & 4.725 & 4.414 & 2.750 & 1.696  \\
				& Sign-OPT \cite{cheng2019sign}& 8.279 & 6.331 & 4.250 & 1.718 & 0.960 & 0.718 &  4.387 & 2.334 & 1.178 & 0.403 & 0.267 & 0.226  \\
				& SVM-OPT \cite{cheng2019sign}& 9.207  & 6.801 & 4.530 & 2.010 & 1.207 & 0.947 & 4.481 & 2.318 & 1.093 & 0.414 & 0.276 & 0.236  \\
				& HSJA \cite{chen2019hopskipjumpattack}&  7.917  & 4.329 & 2.523 & \B 0.793 & \B 0.489 & \B 0.397 &  4.505 & 1.279 & 0.713 & 0.333 & 0.255 & 0.227  \\
				& TA  &  7.943  & \B 4.267 & 2.488 & 0.809 & 0.503 & 0.406 &  \B 4.256  & 1.275 & 0.710 & \B 0.329 & 0.253 & 0.226  \\
				& G-TA  &  \B 7.816  & 4.277 & \B 2.469 &  0.803 & 0.505 & 0.412 &  4.432  & \B 1.270 & \B 0.702 & \B 0.329 & \B 0.252 & \B 0.225  \\
				\Xcline{1-14}{0.1pt}
				\multirow{6}{*}{GDAS} &  BA \cite{brendel2018decisionbased}& 8.487  & 7.885 & 7.821 & 6.034 & 3.632 & 2.703 &  - & 2.717 & 2.514 & 2.373 & 1.642 & 1.106  \\
				& Sign-OPT \cite{cheng2019sign}& 8.372 & 6.514 & 4.351 & 1.827 & 0.987 & 0.711 & 4.917  & 4.159 & 3.260 & 1.352 & 0.452 & 0.250  \\
				& SVM-OPT \cite{cheng2019sign}& 9.529  & 7.243 & 5.092 & 2.347 & 1.317 & 0.958 &4.909 & 3.950 & 2.736 & 1.082 & 0.371 & 0.234  \\
				& HSJA \cite{chen2019hopskipjumpattack}&  7.714 & 3.566 & 1.966 & 0.591 & 0.365 & 0.301 & 2.188  & 0.756 & 0.483 & 0.261 & 0.208 & 0.189  \\
				& TA  &  \B 7.674  & \B 3.529 & \B 1.946 & 0.585 & 0.366 & 0.302 & 2.190 & 0.774 & 0.485 & 0.257 & 0.206 & 0.187  \\
				& G-TA  &  7.697  & 3.558 & 1.959 & \B 0.583 & \B 0.361 & \B 0.298 & \B 2.161 & \B 0.745 & \B 0.476 & \B 0.255 & \B 0.204 & \B 0.185  \\
				\Xcline{1-14}{0.1pt}
				\multirow{6}{*}{WRN-28} & BA \cite{brendel2018decisionbased}& 8.688  &  8.046 & 7.984 & 5.786 & 2.486 & 1.555 &  - & 4.425 & 3.648 & 3.435 & 1.543 & 0.832  \\
				& Sign-OPT \cite{cheng2019sign}& 8.258 & 5.576 & 3.260 & 1.087 & 0.593 & 0.459 &  3.093 & 1.494 & 0.828 & 0.319 & \B 0.239 & \B 0.213  \\
				& SVM-OPT \cite{cheng2019sign}&  9.516 & 5.968 & 3.744 & 1.367 & 0.728 & 0.553 &  2.977  & 1.466 & 0.723 & 0.325 & 0.245 & 0.221  \\
				& HSJA \cite{chen2019hopskipjumpattack}& 6.810 & 2.603 & 1.326 & 0.518 & 0.389 & 0.347 & 3.052 & 0.797 & 0.508 & 0.299 & 0.250 & 0.232  \\
				& TA  &  6.802 & 2.556 & 1.311 & 0.519 & 0.394 & 0.353 & \B 2.974 & 0.785 & \B 0.496 & \B 0.293 & 0.249 & 0.233  \\
				& G-TA  &  \B 6.755 & \B 2.543 & \B 1.281 & \B 0.513 & \B 0.387 & \B 0.345 &  2.995 & \B 0.782 & 0.502 & 0.298 & 0.250 & 0.232  \\
				\Xcline{1-14}{0.1pt}
				\multirow{6}{*}{WRN-40} & BA \cite{brendel2018decisionbased}&  8.658 & 8.014 & 7.953 & 5.738 & 2.484 & 1.566 & -  &4.377 & 3.586 & 3.367 & 1.487 & 0.821  \\
				& Sign-OPT \cite{cheng2019sign}&  8.156 & 5.579 & 3.300 & 1.186 & 0.646 & 0.501 &  4.754  & 3.239 & 1.885 & 0.311 & \B 0.226 & \B0.201  \\
				& SVM-OPT \cite{cheng2019sign}&  9.339  & 6.061 & 3.840 & 1.445 & 0.800 & 0.605 &  4.457  & 2.756 & 0.739 & 0.310 & 0.229 & 0.206  \\
				& HSJA \cite{chen2019hopskipjumpattack}&  6.909 & 2.648 & 1.330 & 0.528 & 0.400 & \B 0.357 &2.992 & 0.777 & 0.498 & 0.290 & 0.242 & 0.225  \\
				& TA  & 6.944  & \B 2.579 & \B 1.295 & \B 0.523 & \B 0.398 & 0.358 &  \B 2.926  & \B 0.770 & \B 0.490 &  \B 0.288 & 0.243 & 0.227  \\
				& G-TA  & \B 6.783  &2.605 & 1.320 & 0.535 & 0.403 & 0.361 &  2.952  & 0.772 & 0.492 & \B 0.288 & 0.241 &  0.223  \\
				\bottomrule
		\end{tabular}}
		\label{tab:CIFAR-10_normal_models_result}
	\end{table}
	\begin{table}[!ht]
		\caption{Experimental results of the combined method of QEBA-S and TA.}
		\small
		\centering
		\scalebox{0.8}{
			\begin{tabular}{ccccccccHHHHHH}
				\toprule
				Target Model  & Method & \multicolumn{6}{c}{Targeted Attack} & \multicolumn{6}{H}{Untargeted Attack} \\
				& & @300 & @1K & @2K & @5K & @8K & @10K &@300 & @1K & @2K & @5K & @8K & @10K \\
				\midrule
				\multirow{2}{*}{Inception-v3} 
				& QEBA-S \cite{li2020qeba} &  \B 100.295 & 79.604 & 63.621 & 35.194 & 22.773 & 18.414 & 66.643 & 25.418 & 15.201 & 6.291 & 4.019 & 3.487\\
				& QEBA-S + TA &  104.490 & \B 75.622 & \B 59.836 & \B 33.112 & \B 22.329 & \B 17.799 & 71.862 & 26.007 & 15.495 & 6.317 & 4.220 & 3.436\\
				\Xcline{1-8}{0.1pt}
				\multirow{2}{*}{Inception-v4}
				& QEBA-S \cite{li2020qeba} &  \B 97.772 & 77.347 & 62.451 & 35.275 & 23.204 & 19.002 & 68.847 & 27.508 & 16.537 & 6.782 & 4.256 & 3.663\\
				& QEBA-S + TA &  101.845 & \B 73.838 & \B 58.554 &\B 33.288 & \B 23.160 & \B 18.736 & 74.298 & 27.954 & 16.866 & 6.782 & 4.468 & 3.633\\
				\Xcline{1-8}{0.1pt}
				\multirow{2}{*}{SENet-154}
				& QEBA-S \cite{li2020qeba} &  \B 72.831 & 55.367 & 42.674 & 21.988 & 13.888 & 11.210 & 48.884 & 18.766 & 11.148 & 4.502 & 2.839 & 2.456\\
				& QEBA-S + TA &  76.547 & \B 52.269 & \B 39.740 & \B 20.608 & \B 13.873 & \B 11.016 & 51.312 & 19.288 & 11.434 & 4.491 & 2.951 & 2.405\\
				\Xcline{1-8}{0.1pt}
				\multirow{2}{*}{ResNet-101}
				& QEBA-S \cite{li2020qeba} &  \B 75.567 & 57.929 & 44.983 & 23.209 & 14.402 & 11.467 & 47.370 & 13.153 & 7.264 & 3.398 & 2.368 & 2.129\\
				& QEBA-S + TA &  78.709 & \B 53.917 & \B 41.245 & \B 21.198 & \B 13.856 & \B 10.773 & 43.106 & 15.695 & 9.351 & 3.846 & 2.616 & 2.183\\
				\bottomrule
		\end{tabular}}
		\label{tab:QEBA_tangent_attack_targeted_attack}
	\end{table}
	
	\subsection{Comparisons with State-of-the-Art Methods}
	\textbf{Results of Attacks against Undefended Models.} Tables~\ref{tab:ImageNet_normal_models_result} and \ref{tab:CIFAR-10_normal_models_result} show the experimental results on the ImageNet and CIFAR-10 datasets. We derive two conclusions based on the results:
	
	(1) We found that the experiments of CIFAR-10 requires a larger radius ratio $r$ than that of ImageNet to achieve the satisfactory performance of G-TA. We speculate that the reason is that the target models of ImageNet have relatively flat decision boundaries. 
	
	(2) TA is more effective in targeted attacks, while the G-TA performs better in untargeted ones. 
	This is because the adversarial region of the target class is narrower and more scattered, making the local classification decision boundary smoother, so that Theorem \ref{theorm:1} holds and TA performs better.
	
	In addition, one unique benefit of our approach is that it can be used as a performance enhanced plug-in when combining it with other HSJA-based approaches (\textit{e.g.}, QEBA-S). Specifically, the method is to change the jump directions of boundary samples of QEBA-S to the directions of the optimal tangent points.
	We present that ``QEBA-S + TA'' can further improve the performance of QEBA-S, as shown in Table \ref{tab:QEBA_tangent_attack_targeted_attack}.
	
	\begin{figure}[t]
		\captionsetup[sub]{font=small}
		\centering 
		\begin{minipage}[t]{.24\textwidth}
			\includegraphics[width=\linewidth]{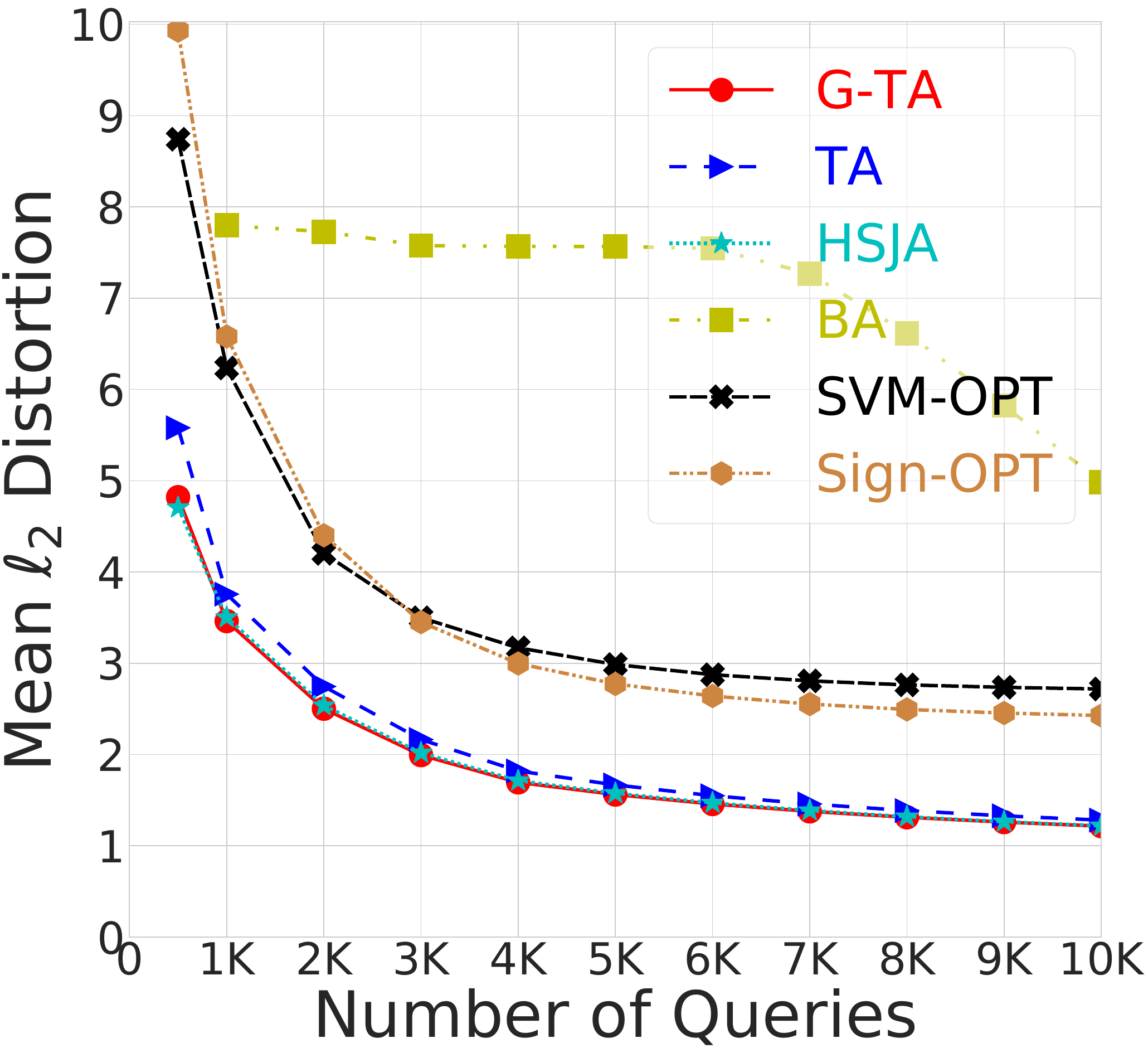}
			\subcaption{AT ($\epsilon = 8/255$)}
			\label{fig:CIFAR_untargeted_adv_train}
		\end{minipage}
		\begin{minipage}[t]{.24\textwidth}
			\includegraphics[width=\linewidth]{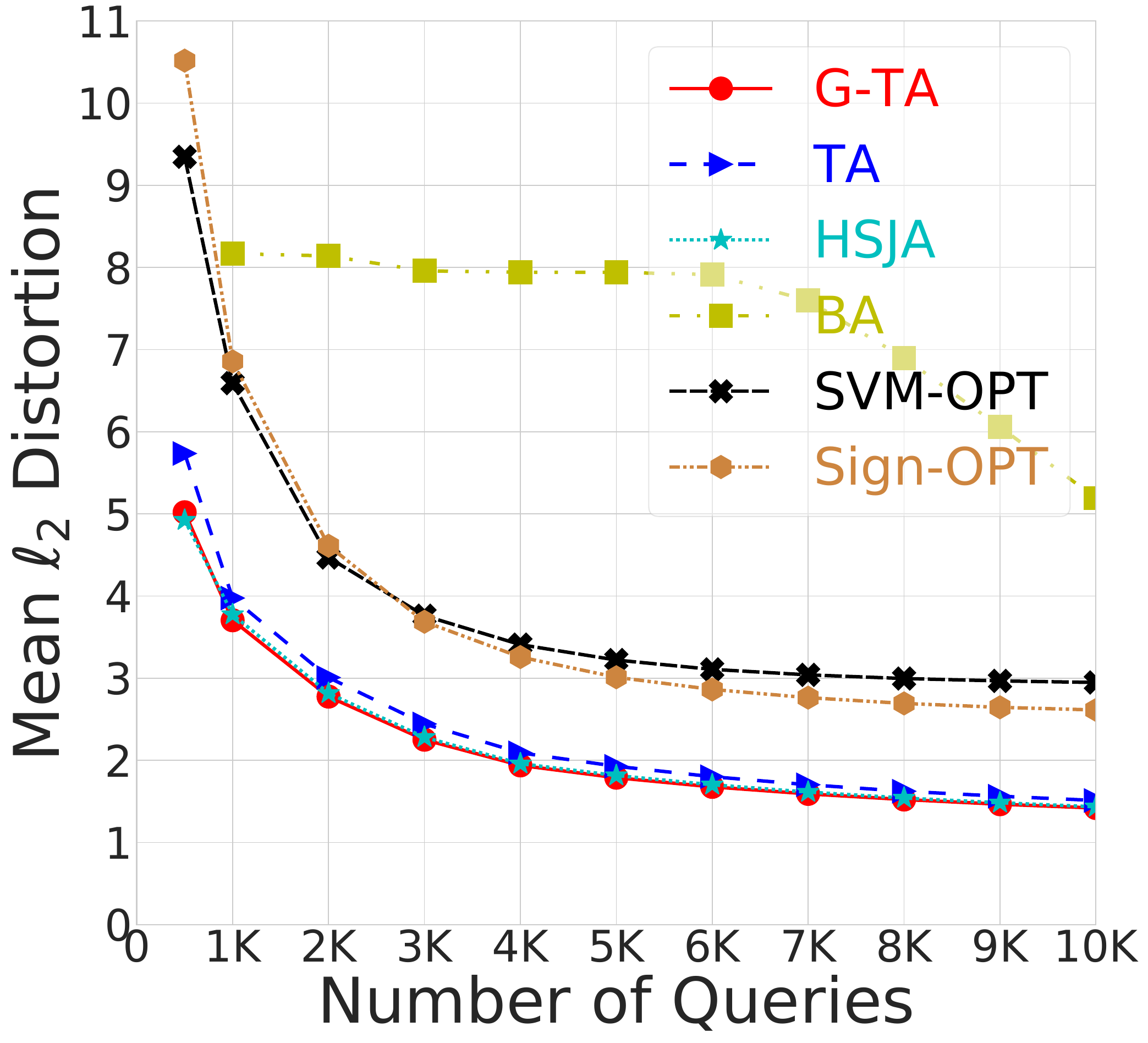}
			\subcaption{TRADES ($\epsilon = 8/255$)}
			\label{fig:CIFAR_untargeted_TRADES}
		\end{minipage}
		\begin{minipage}[t]{.24\textwidth}
			\includegraphics[width=\linewidth]{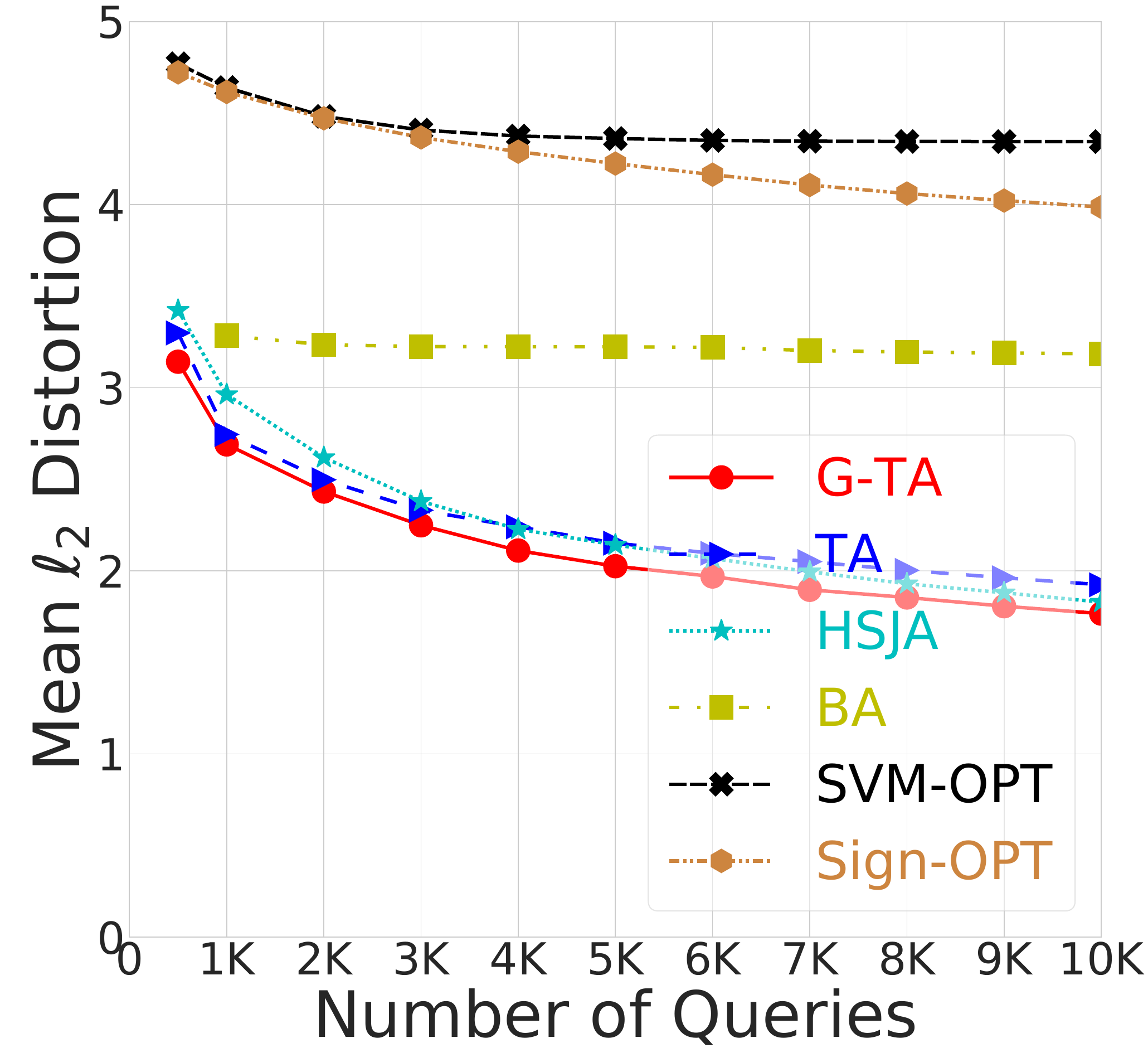}
			\subcaption{JPEG}
			\label{fig:CIFAR_untargeted_ComDefend}
		\end{minipage}
		\begin{minipage}[t]{.24\textwidth}
			\includegraphics[width=\linewidth]{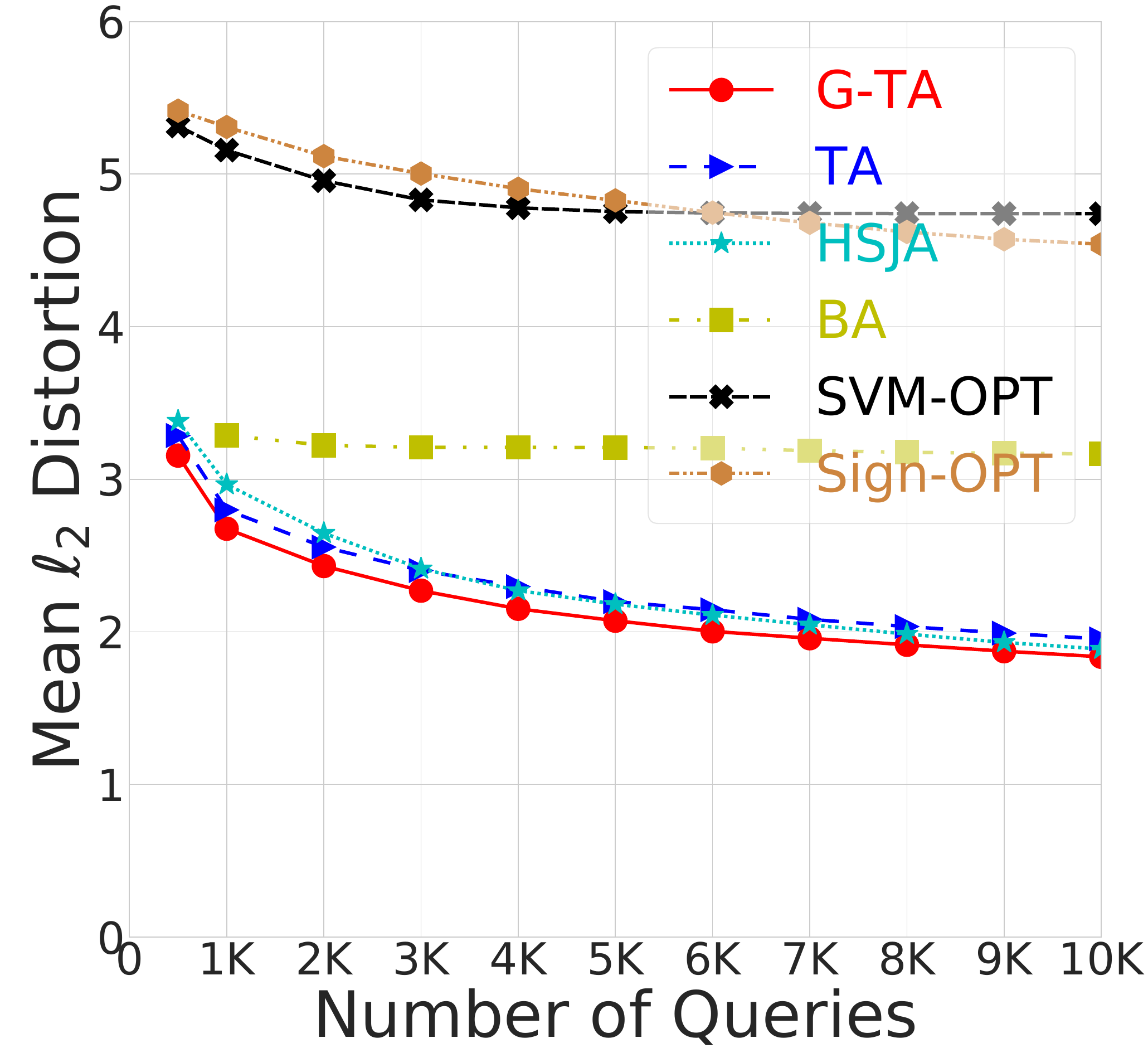}
			\subcaption{Feature Distillation}
			\label{fig:CIFAR_untargeted_FeatureDistillation}
		\end{minipage}
		\caption{Experimental results of the attacks against defense models with the backbone of ResNet-50.}
		\label{fig:ImageNet_defense_models}
		
	\end{figure}
	
	\textbf{Results of Attacks against Defense Models.} 
	We conduct the experiments of untargeted attacks on defense models. Fig. \ref{fig:ImageNet_defense_models} shows the experimental results on the CIFAR-10 dataset. 
	We select 4 defense models: (1) a model obtained through adversarial training, abbreviated as AT \cite{madry2018towards}; (2) an improved AT that optimizes a regularized surrogate loss, named as TRADES \cite{zhang2019theoretically}; (3) an image-transformation-based defense called JPEG \cite{guo2018countering}; and (4) a DNN-oriented compression defense called Feature Distillation \cite{liu2019feature}. All defense models adopt ResNet-50 as the backbone. Previous studies \cite{zhang2021geometryaware} have shown that AT and TRADES have the issue of robust overfitting, which leads to a significant increase in the curvature of the classification decision boundary. Figs. \ref{fig:CIFAR_untargeted_adv_train} and \ref{fig:CIFAR_untargeted_TRADES} show that G-TA outperforms TA when attacking AT and TRADES. This advantage is also demonstrated in attacking other defenses (\textit{e.g.,} Figs.~\ref{fig:CIFAR_untargeted_FeatureDistillation} and \ref{fig:CIFAR_untargeted_ComDefend}), proving the effectiveness of G-TA in attacking defense models.
	\subsection{Comprehensive Understanding of Tangent Attack}
	In the ablation studies, we conduct the experiments of the targeted attacks on the ImageNet dataset to understand our approach in depth, and the target model is ResNet-50. The results are shown in Fig. \ref{fig:ablation_study}.
	
	\textbf{Initialization.} Our algorithm starts with an image $\tilde{\mathbf{x}}_0$ selected from the target class, and we study three selection strategies: (1) a randomly selected image, (2) the image with the shortest distance to the original image, and (3) the image with the longest distance to the original image. Fig. \ref{fig:initial_sample_study} shows that the strategy of (2) achieves the best performance.
	
	\textbf{Radius Ratio.} 
	Fig. \ref{fig:radius_ratio_targeted_study} shows that the performance of G-TA is not sensitive to the radius ratio $r$.
	
	
	\textbf{Jump Direction.} Fig.~\ref{fig:update_direction_study} shows the effects of different jump directions. RandomHSJA is a variant of HSJA which adopts a random vector $\mathbf{r}$ that satisfies $\vecprod{\mathbf{r},\mathbf{u}} > 0$ as the jump direction. Fig.~\ref{fig:update_direction_study} verifies the benefit of jumping to the optimal tangent point.
	
	\textbf{Initial Batch Size.} In general, Fig.~\ref{fig:initial_batch_size_study} shows that a smaller $B_0$ performs better since it saves queries. But $B_0 = 5$ performs worse than $B_0=30$ because it uses too few samples for gradient estimation.

	\begin{figure}[!ht]
		\captionsetup[sub]{font=small}
		\setlength{\abovecaptionskip}{0pt}
		\setlength{\belowcaptionskip}{0pt}
		\centering 
		\begin{minipage}[t]{.24\textwidth}
			\includegraphics[width=\linewidth]{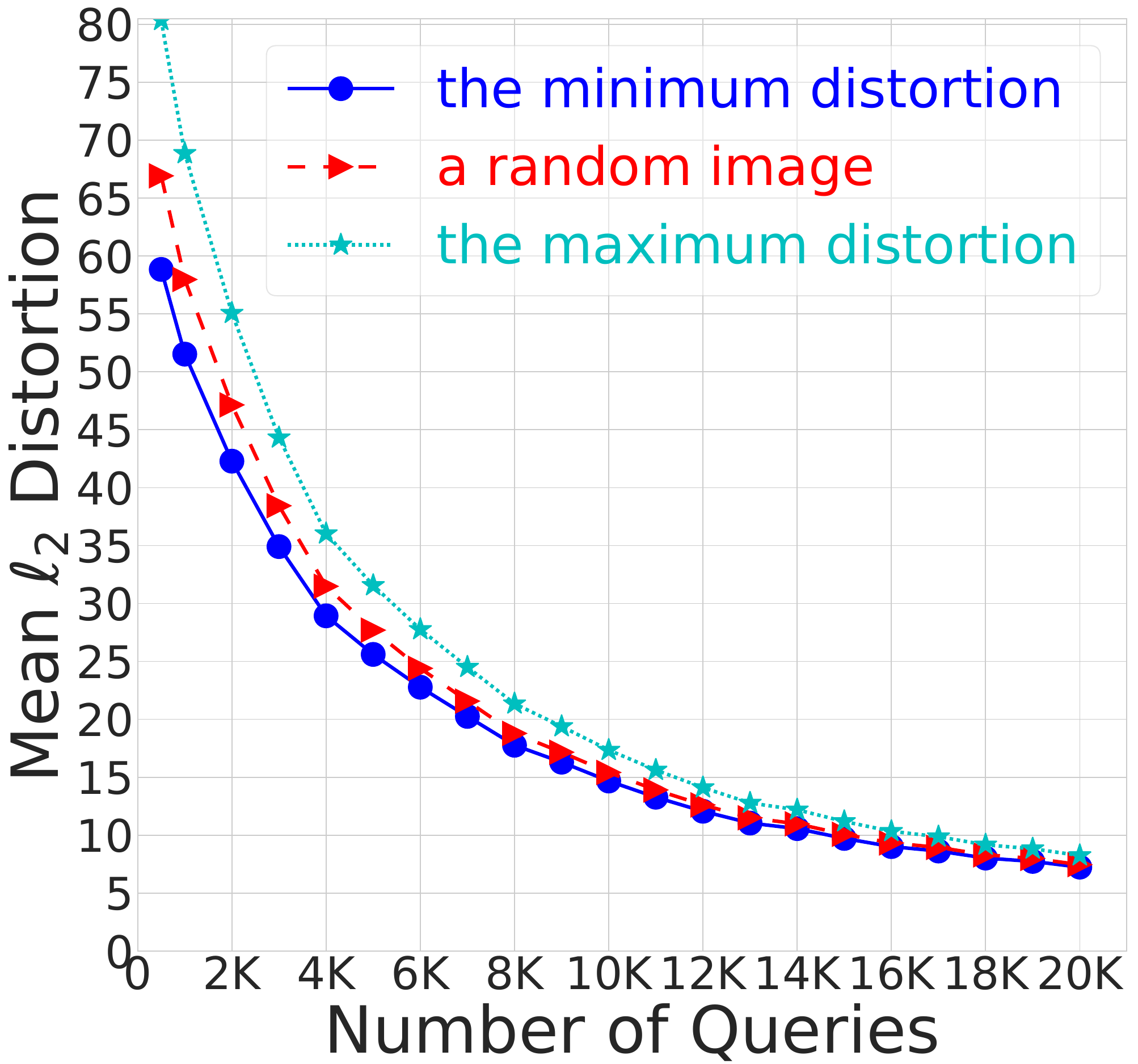}
			\subcaption{initial $\tilde{\mathbf{x}}_0$ selection}
			\label{fig:initial_sample_study}
		\end{minipage}
		\begin{minipage}[t]{.24\textwidth}
			\includegraphics[width=\linewidth]{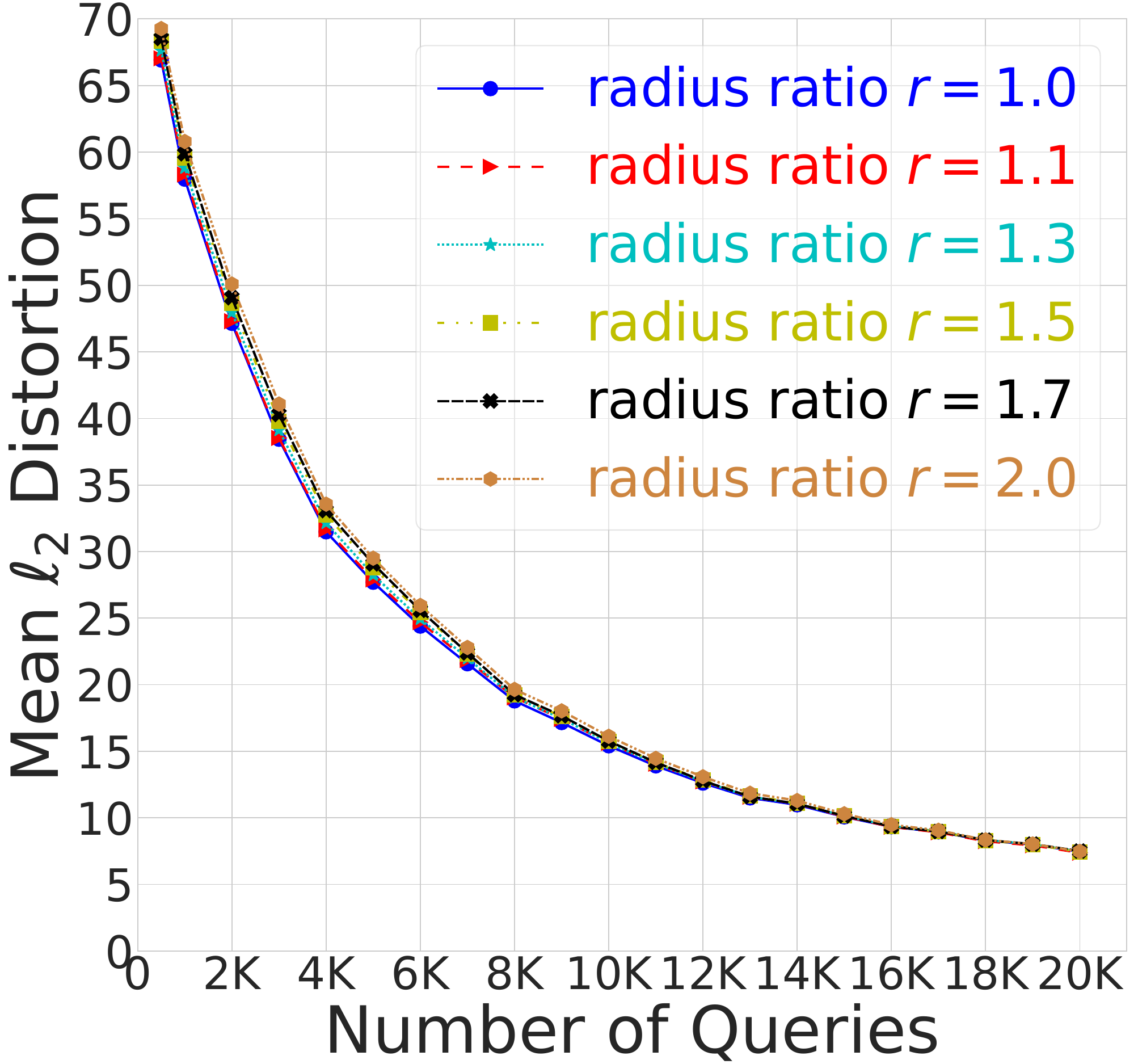}
			\subcaption{radius ratio}
			\label{fig:radius_ratio_targeted_study}
		\end{minipage}
		\begin{minipage}[t]{.24\textwidth}
			\includegraphics[width=\linewidth]{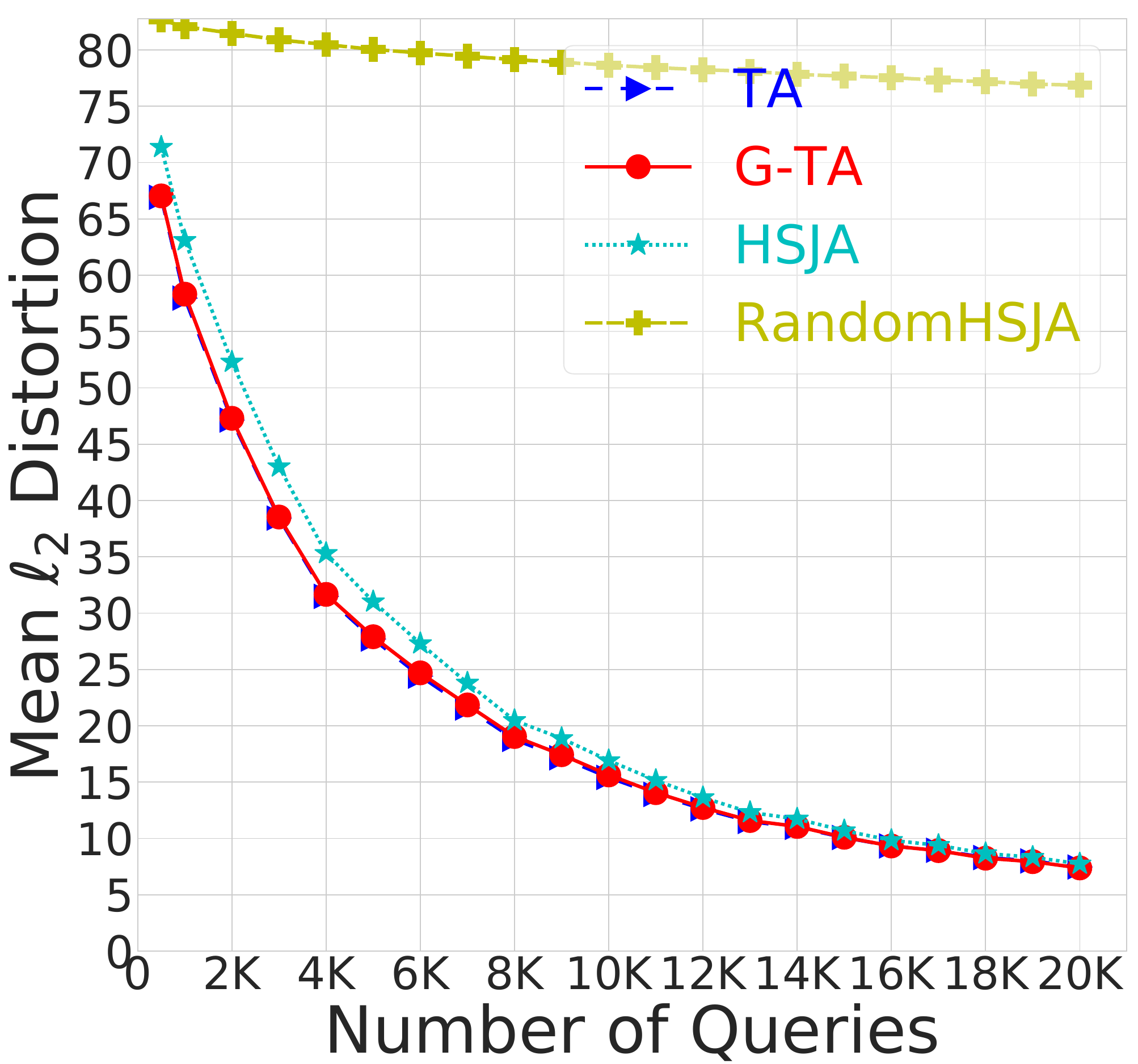}
			\subcaption{jump directions}
			\label{fig:update_direction_study}
		\end{minipage}
		\begin{minipage}[t]{.24\textwidth}
			\includegraphics[width=\linewidth]{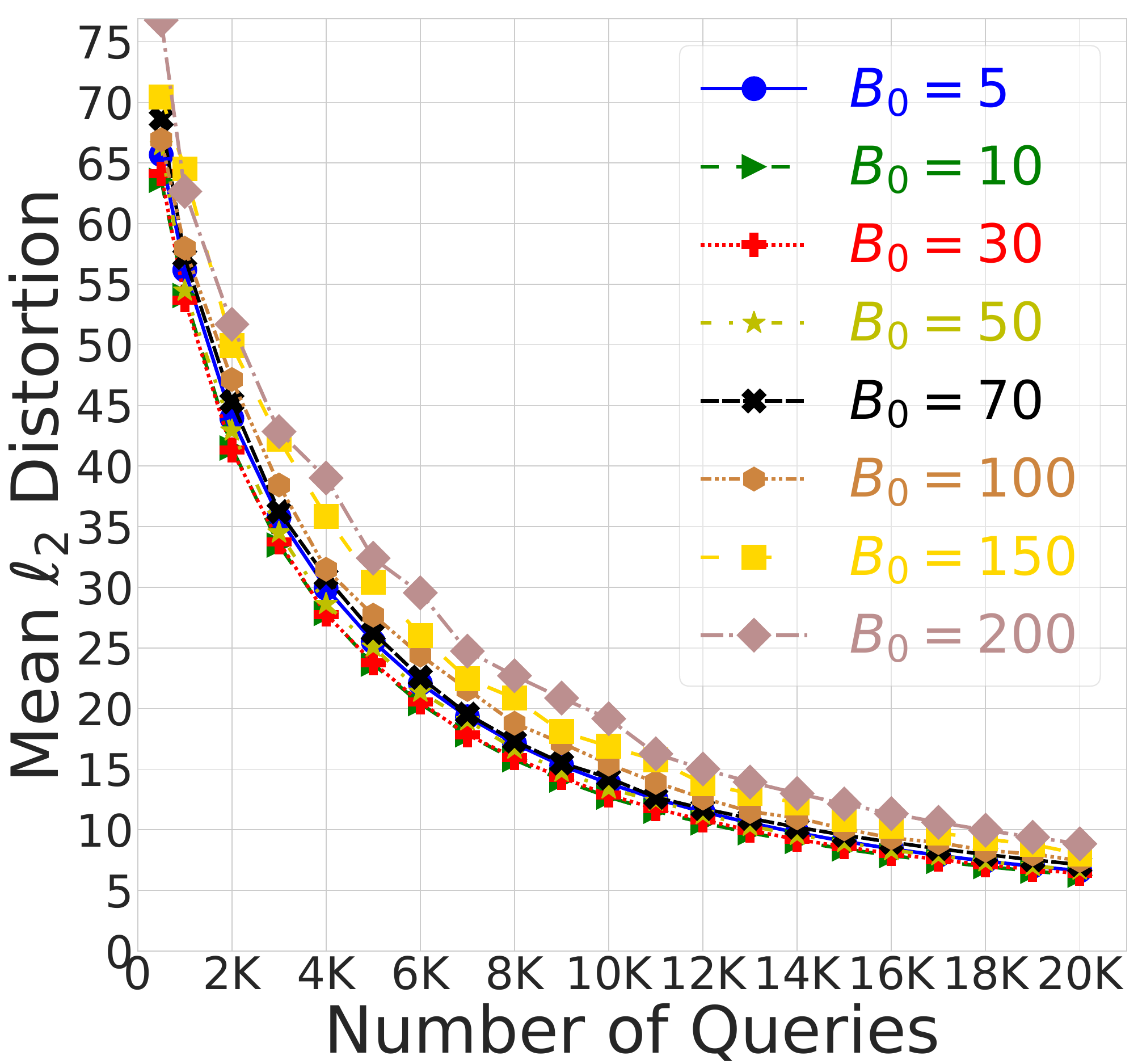}
			\subcaption{initial batch size $B_0$}
			\label{fig:initial_batch_size_study}
		\end{minipage}
		\caption{Experimental results of ablation studies.}
		\label{fig:ablation_study}
	\end{figure}
	\section{Conclusion}
	In this paper, we propose a new geometric-based method for query-efficient hard-label black-box attacks. Our method relies on the observation that the minimum $\ell_2$ distortion can be obtained by searching a boundary point along a tangent line of a virtual hemisphere.
	We offer a closed-form solution for computing the optimal tangent point and provide a formal proof of its correctness. We further propose a generalized method that replaces the hemisphere with a semi-ellipsoid to adapt to the target models with curved decision boundaries. Lastly, we evaluate our approach through extensive experiments and show its superior performance compared with baseline methods.

	\section*{Acknowledgments}
	This research is partially supported by the National Key R\&D Program of China (2019YFB1405703) and TC190A4DA/3, the National Natural Science Foundation of China (Grant Nos. 62021002, 61972221).
	Yisen Wang is partially supported by the National Natural Science Foundation of China under Grant 62006153, and Project 2020BD006 supported by PKU-Baidu Fund.

	\medskip
	
	{
		\small
		\bibliography{main}
		\bibliographystyle{plain}
	}

	\newpage
	\appendix
	\section*{\centering Appendix}


	\newcommand{\set}[1]{{\left\{#1\right\}}}

	\newcommand{\bfa}{{\mathbf{a}}}
	\newcommand{\bfb}{{\mathbf{b}}}

	\newcommand{\bfz}{{\mathbf{z}}}
	\newcommand{\bfy}{{\mathbf{y}}}

	\newcommand{\calS}{{\mathcal{S}}}
	\newcommand{\calC}{{\mathcal{C}}}
	\newcommand{\calP}{{\mathcal{P}}}
	\newcommand{\sfC}{\mathsf{C}}
	\newcommand{\sfF}{\mathsf{F}}
	
	\newcommand{\bbR}{{\mathbb{R}}}
	\newcommand{\bbZ}{{\mathbb{Z}}}
	
	\tikzset{
		reuse path/.code={\pgfsyssoftpath@setcurrentpath{#1}}
	}
	\tikzset{even odd clip/.code={\pgfseteorule},
		protect/.code={
			\clip[overlay,even odd clip,reuse path=#1]
			(-6383.99999pt,-6383.99999pt) rectangle (6383.99999pt,6383.99999pt);
	}}
	\tikzset{
		dot/.style={circle, fill, minimum size=#1, inner sep=0pt, outer sep=0pt},
		dot/.default = 4.5pt,
		hemispherebehind/.style={ball color=gray!20!white, fill=none, opacity=0.3},
		hemispherefront/.style={ball color=gray!65!white, fill=none, opacity=0.4},
		ellipsoidfront/.style={ball color=gray!50!white, fill=none, opacity=0.5},
		circlearc/.style={thick,color=gray!90},
		circlearchidden/.style={thick,dashed,color=gray!90},
		equator/.style = {thick, black},
		diameter/.style = {thick, black},
		axis/.style={thick, -stealth,black!60, every node/.style={text=black, at={([turn]1mm,0mm)}},
		},
	}
	
\section{Potential Negative Societal Impacts}
The adversarial attack is a major security concern in the real-world machine learning system, because the generated adversarial perturbation could be used for the malicious purpose. Our study relies only on the top-1 predicted label to craft the adversarial examples which is applicable to most real-world systems, making it more useful and practical. Although the experiments in this paper are about attacking the image classifier, this method can be used in other settings, such as the object detection, the recommender system, the facial recognition system, and the autonomous driving. In summary, this study could be used in harmful ways by malicious users.

In a broader perspective, the adversarial example is not restricted to malicious applications, and it can be used in the positive side, \textit{e.g.,} the generation of CAPTCHA and the privacy protection. In particular, the study of adversarial attacks can promote the defense techniques. In recent years, many proposed defenses are broken by the latest attacks, which stimulates the development of defenses. 

Our results also point to the potential defense techniques against hard-label attacks. For example, the defense can prohibit queries near the decision boundary, then the approximate gradient cannot be estimated, making Tangent Attack ineffective. Another possible defense is to add random perturbations to the input image to prevent effective gradient estimation, or predict random classification labels for samples near the classification decision boundary.

\section{Proof of Theorem~\ref{theorm:1}}
\subsection{Notations and Assumption}
Before we formally prove Theorem~\ref{theorm:1}, let us first define the notations that will be used in the proof. Let $\mathbf{x}$ denote the original image, and w.l.o.g. we assume the boundary sample $\mathbf{x}_{t-1} = \mathbf{0}$ be the origin of the coordinate axis. Let $B$ denote a $n$-dimensional ball centered at $\mathbf{x}_{t-1}$ with the radius of $R$, and its surface is denoted as $S :=\partial B$. Note that $B$ denotes a complete ball in this proof. However, $B$ denotes the hemisphere in the main text of the paper.
Theorem~\ref{theorm:1} assumes that the classification decision boundary of the target model is the hyperplane $H$, which is defined by its unit normal vector $\mathbf{u}$. Then, the hyperplane $H$ divides $\mathbb{R}^n$ into two half-spaces:
\begin{equation}
\begin{aligned}
H_{\geq 0} &=\{\mathbf{v}\in\mathbb{R}^n \mathbin{|} \vecprod{\mathbf{v}, \mathbf{u}}\geq 0\}, \\
H_{\leq 0} &=\{\mathbf{v}\in\mathbb{R}^n \mathbin{|} \vecprod{\mathbf{v}, \mathbf{u}}\leq 0\}.
\end{aligned}
\end{equation}
In the attack, $H_{\geq 0}$ mainly contains the adversarial region, and $H_{\leq 0}$ represents the non-adversarial region. 
In Fig. \ref{fig:proof_notations}, we visually represent the hyperplane $H$ and two half-spaces in $\mathbb{R}^3$.

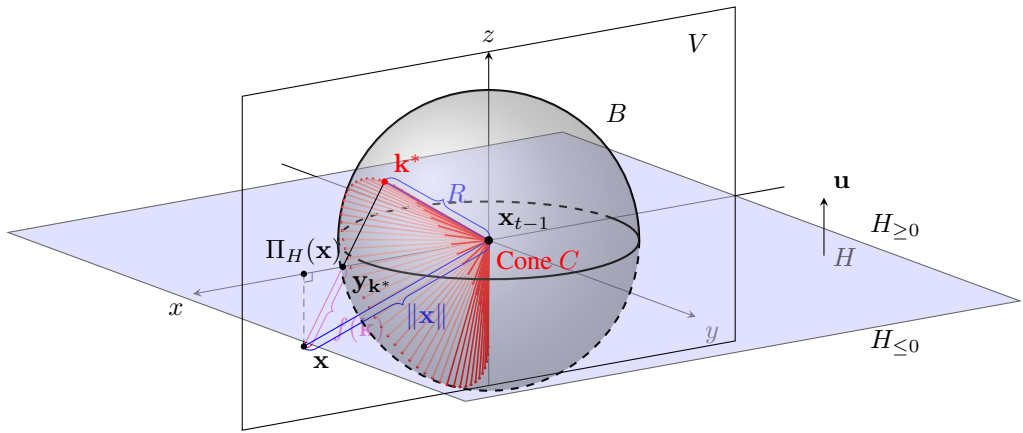
\begin{figure}[b]
	\centering
	\def\r{2.5}
	\tdplotsetmaincoords{75}{145}
	\tikzset{zxplane/.style={canvas is zx plane at y=#1}}
	\tikzset{yxplane/.style={canvas is yx plane at z=#1}}
	\tikzset{yzplane/.style={canvas is yz plane at x=#1}}
	\makeatletter
	\tikzoption{canvas is plane}[]{\@setOxy#1}
	\def\@setOxy O(#1,#2,#3)x(#4,#5,#6)y(#7,#8,#9)%
	{\def\tikz@plane@origin{\pgfpointxyz{#1}{#2}{#3}}%
		\def\tikz@plane@x{\pgfpointxyz{#4}{#5}{#6}}%
		\def\tikz@plane@y{\pgfpointxyz{#7}{#8}{#9}}%
		\tikz@canvas@is@plane
	}
	\makeatother

	\begin{tikzpicture}[tdplot_main_coords,scale=0.8]
	
	\draw[-stealth] (-2.4*\r,0,0) -- (2.4*\r,0,0) node[anchor=north east,thin] {$x$};
	\draw[-stealth] (0,-2.4*\r,0) -- (0,2.4*\r,0) node[anchor=north west,thin] {$y$};
	\draw[-stealth] (0,0,-\r) -- (0,0,1.3*\r) node[anchor=south,thin] {$z$};
	\coordinate (O) at (0,0,0);
	\coordinate (x) at  (1.5*\r,0,-0.5*\r);
	\coordinate (x') at  (1.5*\r,0,0);
	\coordinate (x_k) at  (2.9588, 0.0000, 0.0000);

	\zjbj[4pt]{O}{x'}{x};  
	\draw[densely dashed] (x) -- (x');
	\node[black] (H+) at (-1.9*\r,2*\r,0.6) {$H_{\ge 0}$};
	\node[black] (H-) at (-1.9*\r,2*\r,-1.4) {$H_{\le 0}$};

	\draw[magenta] (x) -- (x_k);
	\draw[decorate,decoration={brace,raise=1pt,amplitude=3pt,mirror},draw=magenta] (x) --node[below right,magenta] {$f(\mathbf{k})$}  (x_k);
	
	\filldraw[fill=blue!20,opacity=0.6,draw,thin] (-2*\r,-2*\r,0) node[black] at (-1.7*\r,1.7*\r,0) {$H$} -- (-2*\r,3.3*\r,0) -- (2.5*\r,3.3*\r,0) -- (2.5*\r,-2*\r,0) -- cycle ;
	\draw[thin] (-2*\r,0,-\r) -- (-2*\r,0,1.3*\r) -- (2*\r,0,1.3*\r) -- (2*\r,0,-\r) -- cycle node[black] at (-1.7*\r,0,1.1*\r) {$V$};
	\draw[tdplot_screen_coords,thick] (-\r,0,0) arc (180:0:\r) node at (0,0,-0.85*\r) {$B$}; 
	\draw[tdplot_screen_coords,thick,dashed] (\r,0,0) arc (0:-180:\r); 
	
	\draw[thick,dashed,canvas is xy plane at z = 0] (\r,0) arc (0:-180:\r);
	
	\fill (x') circle[radius=1.5pt] node[anchor=south] {$\Pi_H(\mathbf{x})$};


	\coordinate (k*) at (2.1124, 0.0000, 1.3371);
	
	\coordinate (y_k) at  (2.9588, 0.0000, 0.0000);
	\coordinate (k'*) at (0.8876275643042051, 0.0, -2.3371173070873836);
	
	\foreach \t in {
		(1.4999999999999998, -1.9364916731037085, -0.49999999999999983),
		(1.558026898085369, -1.9277782006719293, -0.32591930574389155),
		(1.615644899411993, -1.9016473312367626, -0.15306530176402072),
		(1.672436339172634, -1.8581326888377487, 0.01730901751790337),
		(1.727966735976753, -1.7973067831845102, 0.1839002079302603),
		(1.7817782139717986, -1.7193052030714704, 0.3453346419153964),
		(1.833385194997704, -1.6243592944799572, 0.5001555849931143),
		(1.8822731951543525, -1.5128357619599095, 0.6468195854630585),
		(1.9279015388668483, -1.3852807406276293, 0.7837046166005444),
		(1.9697106746724962, -1.2424648168004933, 0.9091320240174889),
		(2.0071344917078835, -1.0854243747041374, 1.0214034751236498),
		(2.0396175650718544, -0.915493765494462, 1.1188526952155635),
		(2.0666366138240595, -0.7343224623692498, 1.1999098414721785),
		(2.0877247091794486, -0.5438719170901557, 1.263174127538347),
		(2.1024960591240482, -0.34638849201425864, 1.307488177372145),
		(2.1106687006546907, -0.14435158609143212, 1.3320061019640732),
		(2.11208232992713, 0.05960045226134553, 1.336246989781391),
		(2.1067089022694705, 0.262760497544315, 1.320126706808413),
		(2.0946545208887875, 0.46245000580091183, 1.2839635626663635),
		(2.0761523448712387, 0.6561133705341866, 1.2284570346137167),
		(2.051547519544043, 0.8413996296933663, 1.1546425586321296),
		(2.02127618120204, 1.01622410378527, 1.0638285436061214),
		(1.9858412021922451, 1.1788058629494256, 0.9575236065767356),
		(1.9457874435293288, 1.3276805157551475, 0.8373623305879869),
		(1.9016789319467287, 1.4616909236570337, 0.7050367958401867),
		(1.8540797337114319, 1.5799605760108109, 0.5622392011342965),
		(1.8035395458873662, 1.6818553566894257, 0.4106186376620991),
		(1.7505843264231888, 1.7669394311946198, 0.2517529792695673),
		(1.6957117368765653, 1.8349302876369933, 0.08713521062969687),
		(1.6393908108522561, 1.8856569196170088, -0.08182756744323066),
		(1.5820650725872825, 1.9190240328386825, -0.25380478223815256),
		(1.5241582704834218, 1.9349841807292634, -0.4275251885497332),		
		(1.4999999999999998, -1.9364916731037085, -0.5000000000000001), 
		(1.4419731019146305, -1.9277782006719295, -0.6740806942561084),
		(1.3843551005880068, -1.9016473312367628, -0.8469346982359791),
		(1.3275636608273653, -1.8581326888377492, -1.0173090175179036),
		(1.2720332640232461, -1.7973067831845102, -1.18390020793026),
		(1.2182217860282007, -1.7193052030714706, -1.3453346419153962),
		(1.1666148050022949, -1.6243592944799572, -1.5001555849931139),
		(1.1177268048456468, -1.5128357619599098, -1.6468195854630585),
		(1.0720984611331517, -1.3852807406276293, -1.7837046166005444),
		(1.0302893253275034, -1.242464816800493, -1.9091320240174887),
		(0.9928655082921165, -1.0854243747041377, -2.02140347512365),
		(0.9603824349281452, -0.9154937654944622, -2.1188526952155637),
		(0.9333633861759404, -0.7343224623692498, -2.199909841472178),
		(0.9122752908205507, -0.5438719170901557, -2.263174127538347),
		(0.8975039408759515, -0.34638849201425864, -2.3074881773721447),
		(0.8893312993453083, -0.14435158609143214, -2.332006101964073),
		(0.8879176700728694, 0.05960045226134554, -2.336246989781391),
		(0.893291097730529, 0.262760497544315, -2.320126706808413),
		(0.9053454791112121, 0.46245000580091183, -2.2839635626663632),
		(0.9238476551287609, 0.6561133705341867, -2.2284570346137165),
		(0.9484524804559566, 0.8413996296933665, -2.15464255863213),
		(0.9787238187979593, 1.01622410378527, -2.063828543606121),
		(1.0141587978077546, 1.1788058629494256, -1.9575236065767354),
		(1.0542125564706708, 1.327680515755148, -1.837362330587987),
		(1.0983210680532711, 1.4616909236570337, -1.7050367958401869),
		(1.1459202662885677, 1.579960576010811, -1.5622392011342963),
		(1.1964604541126334, 1.6818553566894254, -1.4106186376620986),
		(1.249415673576811, 1.7669394311946196, -1.2517529792695672),
		(1.304288263123434, 1.8349302876369933, -1.0871352106296968),
		(1.3606091891477434, 1.8856569196170088, -0.9181724325567693),
		(1.4179349274127173, 1.9190240328386825, -0.7461952177618473),
		(1.4758417295165773, 1.9349841807292631, -0.5724748114502667)		
	}
	{\fill[fill=red] \t circle[radius=0.7pt];
	}
	\foreach \t/\c in {
		(1.4999999999999998, -1.9364916731037085, -0.49999999999999983)/red!30!white,
		(1.558026898085369, -1.9277782006719293, -0.32591930574389155)/red!30!white,
		(1.615644899411993, -1.9016473312367626, -0.15306530176402072)/red!30!white,
		(1.672436339172634, -1.8581326888377487, 0.01730901751790337)/red!30!white,
		(1.727966735976753, -1.7973067831845102, 0.1839002079302603)/red!40!white,
		(1.7817782139717986, -1.7193052030714704, 0.3453346419153964)/red!40!white,
		(1.833385194997704, -1.6243592944799572, 0.5001555849931143)/red!40!white,
		(1.8822731951543525, -1.5128357619599095, 0.6468195854630585)/red!40!white,
		(1.9279015388668483, -1.3852807406276293, 0.7837046166005444)/red!40!white,
		(1.9697106746724962, -1.2424648168004933, 0.9091320240174889)/red!50!white,
		(2.0071344917078835, -1.0854243747041374, 1.0214034751236498)/red!50!white,
		(2.0396175650718544, -0.915493765494462, 1.1188526952155635)/red!50!white,
		(2.0666366138240595, -0.7343224623692498, 1.1999098414721785)/red!50!white,
		(2.0877247091794486, -0.5438719170901557, 1.263174127538347)/red!50!white,
		(2.1024960591240482, -0.34638849201425864, 1.307488177372145)/red!50!white,
		(2.1106687006546907, -0.14435158609143212, 1.3320061019640732)/red!60!white,
		(1.4999999999999998, -1.9364916731037085, -0.5000000000000001)/red!30!white, 
		(1.4419731019146305, -1.9277782006719295, -0.6740806942561084)/red!30!white,
		(1.3843551005880068, -1.9016473312367628, -0.8469346982359791)/red!30!white,
		(1.3275636608273653, -1.8581326888377492, -1.0173090175179036)/red!30!white,
		(1.2720332640232461, -1.7973067831845102, -1.18390020793026)/red!30!white,
		(1.2182217860282007, -1.7193052030714706, -1.3453346419153962)/red!30!white,
		(1.1666148050022949, -1.6243592944799572, -1.5001555849931139)/red!30!white,
		(1.1177268048456468, -1.5128357619599098, -1.6468195854630585)/red!30!white,
		(1.0720984611331517, -1.3852807406276293, -1.7837046166005444)/red!30!white,
		(1.0302893253275034, -1.242464816800493, -1.9091320240174887)/red!40!white,
		(0.9928655082921165, -1.0854243747041377, -2.02140347512365)/red!40!white,
		(0.9603824349281452, -0.9154937654944622, -2.1188526952155637)/red!40!white,
		(0.9333633861759404, -0.7343224623692498, -2.199909841472178)/red!40!white,
		(0.9122752908205507, -0.5438719170901557, -2.263174127538347)/red!40!white,
		(0.8975039408759515, -0.34638849201425864, -2.3074881773721447)/red!40!white,
		(0.8893312993453083, -0.14435158609143214, -2.332006101964073)/red!40!white,
		(2.11208232992713, 0.05960045226134553, 1.336246989781391)/red,
		(2.1067089022694705, 0.262760497544315, 1.320126706808413)/red,
		(2.0946545208887875, 0.46245000580091183, 1.2839635626663635)/red,
		(2.0761523448712387, 0.6561133705341866, 1.2284570346137167)/red,
		(2.051547519544043, 0.8413996296933663, 1.1546425586321296)/red,
		(2.02127618120204, 1.01622410378527, 1.0638285436061214)/red,
		(1.9858412021922451, 1.1788058629494256, 0.9575236065767356)/red,
		(1.9457874435293288, 1.3276805157551475, 0.8373623305879869)/red,
		(1.9016789319467287, 1.4616909236570337, 0.7050367958401867)/red,
		(1.8540797337114319, 1.5799605760108109, 0.5622392011342965)/red,
		(1.8035395458873662, 1.6818553566894257, 0.4106186376620991)/red,
		(1.7505843264231888, 1.7669394311946198, 0.2517529792695673)/red,
		(1.6957117368765653, 1.8349302876369933, 0.08713521062969687)/red,
		(1.6393908108522561, 1.8856569196170088, -0.08182756744323066)/red,
		(1.5820650725872825, 1.9190240328386825, -0.25380478223815256)/red,
		(1.5241582704834218, 1.9349841807292634, -0.4275251885497332)/red,	
		(0.8879176700728694, 0.05960045226134554, -2.336246989781391)/red,
		(0.893291097730529, 0.262760497544315, -2.320126706808413)/red,
		(0.9053454791112121, 0.46245000580091183, -2.2839635626663632)/red,
		(0.9238476551287609, 0.6561133705341867, -2.2284570346137165)/red,
		(0.9484524804559566, 0.8413996296933665, -2.15464255863213)/red,
		(0.9787238187979593, 1.01622410378527, -2.063828543606121)/red,
		(1.0141587978077546, 1.1788058629494256, -1.9575236065767354)/red,
		(1.0542125564706708, 1.327680515755148, -1.837362330587987)/red,
		(1.0983210680532711, 1.4616909236570337, -1.7050367958401869)/red,
		(1.1459202662885677, 1.579960576010811, -1.5622392011342963)/red,
		(1.1964604541126334, 1.6818553566894254, -1.4106186376620986)/red,
		(1.249415673576811, 1.7669394311946196, -1.2517529792695672)/red,
		(1.304288263123434, 1.8349302876369933, -1.0871352106296968)/red,
		(1.3606091891477434, 1.8856569196170088, -0.9181724325567693)/red,
		(1.4179349274127173, 1.9190240328386825, -0.7461952177618473)/red,
		(1.4758417295165773, 1.9349841807292631, -0.5724748114502667)/red}
	{\draw[\c,semithick] (O) -- \t;
	}
	
	\draw[thick,canvas is xy plane at z = 0] (\r,0) arc (0:180:\r);
	\coordinate (cone_edge) at  ({1.5*\r-0.633 * cos(30)},  {0.633 * sin(30)}, 0); 

	\coordinate (k)  at (-1.7677669529663689,1.7677669529663689,0);
	\draw[blue] (O) -- (k*);
	\draw[decorate,decoration={brace,raise=1pt,amplitude=4pt,mirror},draw=blue] (O) -- node[above right,blue] {$R$}  (k*);
	
	\draw[blue] (O) --  (x);
	\draw[decorate,decoration={brace,raise=1pt,amplitude=4pt,mirror},draw=blue] (x) --  node[below right,blue] {$\|\mathbf{x} \|$} (O);
	
	\shade[
	tdplot_screen_coords,hemispherefront,
	delta angle=-180,
	x radius=\r,
	] (0,0)
	circle [radius=\r];
	
	\node[red] (C) at (-0.4*\r,0,-0.2*\r) {Cone $C$};

	\draw[black] (k*) -- (x_k);
	\fill (x_k) circle[radius=1.5pt] node[anchor=north west] {$\mathbf{y}_{\mathbf{k}^*}$};
	\fill (x) circle[radius=1.5pt] node[anchor=north west] {$\mathbf{x}$};
	\fill[red] (k*) circle[radius=1.5pt] node[anchor=south west,red] {$\mathbf{k}^*$};
	\fill (O) circle[radius=2pt] node[anchor=south west] {$\mathbf{x}_{t-1}$};
	\draw[black,-stealth] (-1.6*\r,1.6*\r,0) -- (-1.6*\r,1.6*\r,1.0) node[anchor=south west] {$\mathbf{u}$}; 
	
	\end{tikzpicture}
	\caption{Illustration of the entities defined in the proof, where $C$ is a convex cone whose boundary intersects with the circle formed by all the tangent points from $\mathbf{x}$ to the ball $B$.}
	\label{fig:proof_notations}
\end{figure}
Suppose $\bfx\in H_{\leq 0}\setminus B$ is a fixed point outside $B$ such that $\vecprod{\bfx, \bfu}< 0$. 
Now, let us define the cosine function $\cos(\bfa,\bfb):=\frac{\vecprod{\bfa,\bfb}}{\|\bfa\|\|\bfb\|}$ to represent the cosine of the angle between two vectors, then we can define the convex cone $C$ with $\mathbf{x}_{t-1}$ as its vertex, as shown below:
\begin{equation}
C:=\set{\mathbf{v}\in\mathbb{R}^n \mathbin{|} \cos(\mathbf{v}, \mathbf{x}) \geq \frac{R}{\|\mathbf{x}\|}}.
\end{equation}
Fig. \ref{fig:proof_notations} demonstrates the convex cone $C$ in $\mathbb{R}^3$.
For $\bfv\in S\cap C$ that satisfies $\cos(\mathbf{v}, \mathbf{x})=R \mathbin{/} \|\mathbf{x}\|$,
the equation $\|\mathbf{v}-\mathbf{x}\|^2=\|\mathbf{x}\|^2-\|\mathbf{v}\|^2$ holds,
\textit{i.e.,} $\mathbf{v}$ is the tangent point of the tangent line from $\mathbf{x}$ to the surface of $B$.

To make the feasible region of the optimization problem (3) in Theorem~\ref{theorm:1} nonempty,
we need to make an assumption about the positional relationship between $\mathbf{x}$ and the ball $B$.
Let $\Pi_H:\mathbb{R}^n\mapsto H$ denote the orthogonal projection from $\mathbb{R}^n$ onto the hyperplane $H$, we make the following assumption:
\begin{assumption}\label{asspt:x-position}
	$\Pi_H(\bfx)\in C$
\end{assumption}
Note that Assumption \ref{asspt:x-position} is not really an ``assumption'': it essentially means that there is a tangent point on $S\cap H_{\ge 0}$, which is in the adversarial region.
Assumption \ref{asspt:x-position} means the feasible region of the optimization problem \eqref{eqn:objective} is a nonempty set.
By repeatedly reducing the radius $R$, the algorithm guarantees that the optimal tangent point is in the adversarial region, thereby making Assumption~\ref{asspt:x-position} always hold.
In addition, according to Assumption \ref{asspt:x-position}, $\|\Pi_H(\mathbf{x})\|\ge R$ holds.

In Theorem~\ref{theorm:1}, $\mathbf{k}$ is an arbitrary point on the surface of the hemisphere $B \cap H_{\geq 0}$, so this proof mainly focuses on points in this region.
In the following text, the hemisphere is denoted as $B':= B\cap H_{\geq 0}$, and its surface is denoted as $S':=S\cap H_{\geq 0}$ for brevity.
Now, let us pick up any $\mathbf{k}\in S'$\footnote{Note that $\mathbf{k}$ defined here may not be a tangent point on the ball.}, 
and then the intersection point between the hyperplane $H$ and the line passing through $\mathbf{x}$ and $\mathbf{k}$ is denoted as $\mathbf{y}_{\mathbf{k}}$.
Then, $(\mathbf{y}_{\mathbf{k}},\lambda)$ is the unique solution of the following equation system:
\begin{equation}
\label{eqn:y_k}
\begin{aligned}
& \mathbf{y}_{\mathbf{k}}=\lambda \mathbf{k} + (1-\lambda)\mathbf{x}, \\
& \vecprod{\mathbf{y}_{\mathbf{k}},\mathbf{u}} = 0,\\
& 0\leq\lambda\leq 1.
\end{aligned}
\end{equation}

Because the position of $\mathbf{k}$ determines the distance between $\mathbf{y}_{\mathbf{k}}$ and $\mathbf{x}$, 
we can define the function $f(\mathbf{k}):=\|\mathbf{y}_{\mathbf{k}}-\mathbf{x}\|$ to represent the distance between $\mathbf{x}$ and $\mathbf{y}_{\mathbf{k}}$.

\subsection{Proof}

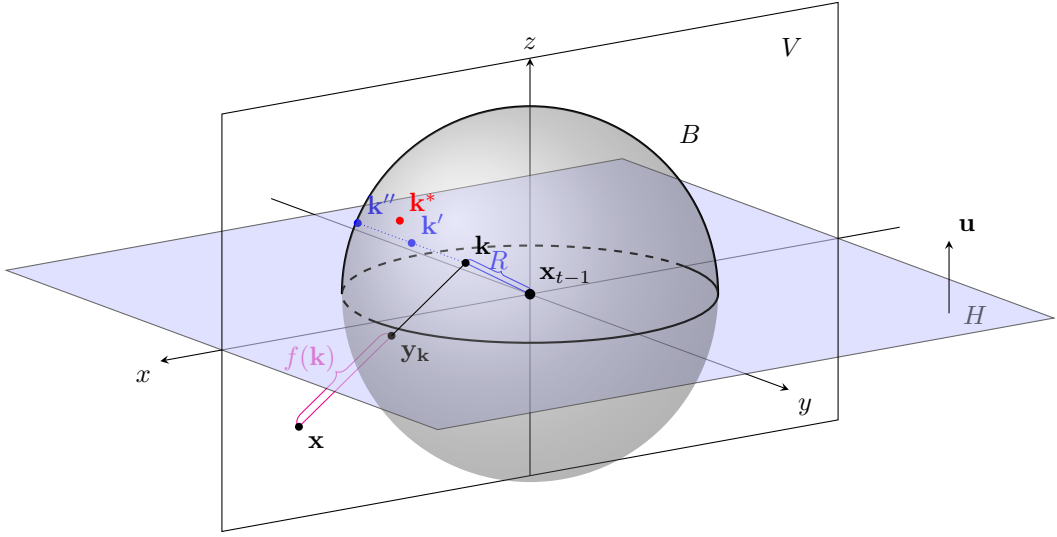
\begin{figure}[t]
	\centering
	\def\r{2.5}
	\tdplotsetmaincoords{75}{145}
	\tikzset{zxplane/.style={canvas is zx plane at y=#1}}
	\tikzset{yxplane/.style={canvas is yx plane at z=#1}}
	\tikzset{yzplane/.style={canvas is yz plane at x=#1}}
	\begin{tikzpicture}[tdplot_main_coords,scale=1]
	\begin{scope}[thin]
	\draw[-stealth] (-2.4*\r,0,0) -- (2.4*\r,0,0) node[anchor=north east,thin] {$x$};
	\draw[-stealth] (0,-2.4*\r,0) -- (0,2.4*\r,0) node[anchor=north west,thin] {$y$};
	\draw[-stealth] (0,0,-\r) -- (0,0,1.3*\r) node[anchor=south,thin] {$z$};
	\coordinate (O) at (0,0,0);
	\coordinate (x) at  (1.5*\r,0,-0.5*\r);
	
	\coordinate (y_k) at  (2.7335, 0.6944, 0.0000);
	\draw[magenta] (x) --node[above left] {$f(\mathbf{k})$}  (y_k);
	\draw[decorate,decoration={brace,raise=1pt,amplitude=4pt,mirror},draw=magenta] (y_k) -- (x);
	\filldraw[fill=blue!20,opacity=0.6,draw,thin] (-2*\r,-2*\r,0) -- (-2*\r,2*\r,0) -- (2*\r,2*\r,0) -- (2*\r,-2*\r,0) -- cycle node[black] at (-1.7*\r,1.7*\r,0) {$H$};
	\draw[thin] (-2*\r,0,-\r) -- (-2*\r,0,1.3*\r) -- (2*\r,0,1.3*\r) -- (2*\r,0,-\r) -- cycle node[black] at (-1.7*\r,0,1.1*\r) {$V$};
	\draw[tdplot_screen_coords,thick] (-\r,0,0) arc (180:0:\r) node at (0,0,-0.85*\r) {$B$}; 
	
	\draw[thick,dashed,canvas is xy plane at z = 0] (\r,0) arc (0:-180:\r);
	\draw[thick,canvas is xy plane at z = 0] (\r,0) arc (0:180:\r);
	
	\end{scope}

	\coordinate  (k) at (1.920286436967152, 0.5*\r, 1.0);
	\coordinate  (k') at (1.920286436967152, 0, 1.0);
	\coordinate  (k'') at (1.920286436967152, -0.5*\r, 1.0);
	\coordinate (k*) at (2.1124, 0.0000, 1.3371);
	
	\fill[blue] (k') circle[radius=1.5pt] node[anchor=south west,blue] {$\mathbf{k}'$};
	\fill[blue] (k'') circle[radius=1.5pt] node[anchor=south west,blue] {$\mathbf{k}''$};
	\draw[blue] (O) -- node[above] {$R$}  (k);
	\draw[decorate,decoration={brace,raise=1pt,amplitude=2pt,mirror},draw=blue] (O) --(k);
	\fill (x) circle[radius=1.5pt] node[anchor=north west] {$\mathbf{x}$};
	\fill (y_k) circle[radius=1.5pt] node[anchor=north west] {$\mathbf{y}_{\mathbf{k}}$};
	\draw[blue,densely dotted] (k) -- (k'');
	\shade[
	tdplot_screen_coords,hemispherefront,
	delta angle=-180,
	x radius=\r,
	] (0,0)
	circle [radius=\r];
	
	\fill (O) circle[radius=2pt] node[anchor=south west] {$\mathbf{x}_{t-1}$};
	\fill[red] (k*) circle[radius=1.5pt] node[anchor=south west,red] {$\mathbf{k}^*$};
	\fill (k) circle[radius=1.5pt] node[anchor=south west] {$\mathbf{k}$};
	\draw[black] (k) -- (y_k);
	\draw[black,-stealth] (-1.6*\r,1.6*\r,0) -- (-1.6*\r,1.6*\r,1.0) node[anchor=south west] {$\mathbf{u}$}; 
	\end{tikzpicture}
	\caption{Illustration of the points used in proving Lemma \ref{lemma}, where $\mathbf{k}''$ is the mirror point of $\mathbf{k}$ with respect to the plane $V$, and $\mathbf{k}'$ is the projection of $\mathbf{k}$ onto the plane $V$.}
	\label{fig:3d_notations}
\end{figure}

To prove Theorem~\ref{theorm:1}, we turn to prove the following lemma, which is equivalent to Theorem~\ref{theorm:1}. 
\begin{lemma}
	\label{lemma}
	Let $S', f$ be defined as above, then minimizing $f$ over the feasible region $S'$ is equivalent to finding the point $\mathbf{k}$ from the set $S'\cap C$ that is farthest away from $H$, \textit{i.e.,}
	\begin{equation}
	\label{eqn:lemma}
	\argmin_{\mathbf{k}\in S'}f(\mathbf{k})=\argmax_{\mathbf{k}\in (S'\cap C)}\vecprod{\mathbf{k}, \mathbf{u}}.
	\end{equation}
	In addition, we can replace $S'$ with $B'$ in the above equation, and the optimal solution of $f(\mathbf{k})$ does not change. In other words, when the feasible region is $B'$, the optimal solution can be always obtained at the surface of $B'$. Thus, the following equation holds:
	\begin{equation}
	\label{eqn:chain_eq}
	\argmin_{\mathbf{k}\in S'}f(\mathbf{k})=\argmin_{\mathbf{k}\in B'}f(\mathbf{k})=\argmax_{\mathbf{k}\in (S'\cap C)}\vecprod{\mathbf{k}, \mathbf{u}}=\argmax_{\mathbf{k}\in (B'\cap C)}\vecprod{\mathbf{k}, \mathbf{u}}.
	\end{equation}
\end{lemma}
\begin{proof}
	By simplifying the original problem to a two-dimensional plane, the proof of Lemma \ref{lemma} will be readily apparent.
	Let $V:=\mathrm{span}(\{\bfx,\bfu\})$ be the plane spanned by $\mathbf{x}$ and $\mathbf{u}$. 
	It is easy to observe $B$, $S$, $C$, $B'$, and $S'$ are symmetrical about the plane $V$.
	Next, we will show that for any $\mathbf{k}\in B'$, there must exist a point $\mathbf{k}^*\in S' \cap V$ such that $f(\mathbf{k}^*) \leq f(\mathbf{k})$. 
	To find the $\mathbf{k}^*$ that satisfies the condition, we introduce the notation $\Pi_V:\mathbb{R}^n\mapsto V$ to denote the projection from $\mathbb{R}^n$ to $V$.
	
	Now, take any $\mathbf{k}\in B'$, and use $\mathbf{k}''$ to denote the mirror point of $\bfk$ with respect to $V$, as shown in Fig.~\ref{fig:3d_notations}.
	The projection point $\Pi_V(\mathbf{k})$ is the midpoint of the line between $\mathbf{k}$ and $\mathbf{k}''$, 
	\textit{i.e.,} $\mathbf{k}''=2\Pi_V(\mathbf{k}) - \mathbf{k}$. Note that if $\mathbf{k} \in V$, then $\mathbf{k}$, $\mathbf{k}'$ and $\mathbf{k}''$ coincide.
	Since $B'$ is symmetrical about the plane $V$, we have $\mathbf{k}'' \in B'$. 
	Now since $B'$ is the intersection of two convex sets $B$ and $H_{\geq 0}$, we know that $B'$ is also a convex set.
	Notice that $\Pi_V(\mathbf{k}) = \frac{1}{2} \cdot (\mathbf{k} + \mathbf{k}'')$ is a convex combination of $\mathbf{k}$ and $\mathbf{k}''$, and $B'$ is a convex set, thus we conclude $\Pi_V(\mathbf{k}) \in B'$.
	
	Now, we will show that we can ignore any point outside of $V$, thus restricting the problem to the two-dimensional plane $V$. Formally, the following inequality holds for any $\mathbf{k}$:
	\begin{equation}
	\label{eqn:proof_last_goal}
	f\left(\Pi_V(\mathbf{k})\right)\leq f(\mathbf{k}).
	\end{equation}
	The above inequality is easy to prove.
	Because $\mathbf{x}\in V$, we have $\| \Pi_V(\mathbf{y}_{\mathbf{k}} - \mathbf{x}) \| = \| \Pi_V(\mathbf{y}_{\mathbf{k}}) - \mathbf{x} \|$. Therefore,
	\begin{equation}
	\label{eqn:substitute}
	f\left(\Pi_V(\mathbf{k})\right)= \| \Pi_V(\mathbf{y}_{\mathbf{k}}) - \mathbf{x} \| = \|\Pi_V(\mathbf{y}_{\mathbf{k}}-\mathbf{x})\| \leq \|\mathbf{y}_{\mathbf{k}}-\mathbf{x}\| =f(\mathbf{k}).
	\end{equation}
	Now, we can focus on the plane $V$ and find the optimal $\mathbf{k}^*$ on it such that $f(\mathbf{k}^*)\leq f\left(\Pi_V(\mathbf{k})\right)$.
	Let us define $C_0$ to denote the convex cone with the point $\mathbf{x}$ as the vertex, and its boundary is formed by all tangent lines from $\mathbf{x}$ to $B$:
	\begin{equation}
	C_0:=\set{\mathbf{v}\in\mathbb{R}^n \mathbin{|} \cos(\mathbf{v}-\mathbf{x}, -\mathbf{x})\geq\sqrt{1-\frac{R^2}{\|\mathbf{x}\|^2}}}.
	\label{eqn:C0}
	\end{equation}
	
	Let $\mathbf{k}' := \Pi_V(\mathbf{k})$ be the projection point of $\mathbf{k}$ onto the plane $V$ (see Fig. \ref{fig:2d}).
	Because $\mathbf{k}' \in B'$ and $\mathbf{k}'\in V$, we have $\mathbf{k}' \in \Pi_V(B')$.
	Now, we define $\mathbf{k}^*\in S'\cap C_0\cap V$ to be the tangent point from $\mathbf{x}$ to the semicircle $\Pi_V(B')$. We claim $\mathbf{k}^*$ is the optimal one that attains the minimum $f(\mathbf{k}')$ among all $\mathbf{k}'$ in $\Pi_V(B')$.
	We denote the angle between $-\mathbf{x}$ and $\mathbf{u}$ as $\theta_1$, \textit{i.e.,} $\theta_1 := \arccos\left(\cos(-\mathbf{x},\mathbf{u})\right)$.
	The angle between $-\mathbf{x}$ and $\mathbf{k}' - \mathbf{x}$ is denoted as $\theta_2$, \textit{i.e.,} $\theta_2 := \arccos\left(\cos(-\mathbf{x},\mathbf{k}'-\mathbf{x})\right)$.
	Based on the position of $\mathbf{k}'$ in $\Pi_V(B')$, there are two possible cases for the angle $\theta_2$, as shown in Fig. \ref{subfig:same_side} and Fig.~\ref{subfig:diff_side}, respectively. We discuss them separately below.
	\begin{figure}[t]
		\begin{minipage}[b]{.49\textwidth}
			\centering
			\begin{tikzpicture}[scale=0.8]
			\def\r{2.0}
			\newcommand\zjbjblue[4][7pt]{%
				\draw[color=blue] let \p1=(#2),\p2=(#4),\p0=(#3) in
				(#3)++({atan2(\y1-\y0,\x1-\x0)}:#1)
				--++({atan2(\y2-\y0,\x2-\x0)}:#1)
				--++({atan2(\y1-\y0,\x1-\x0)}:-#1);
			}
			
			\draw[very thick] (-1,3) -- (6,3) ;
			\node at (-1.3,2.9) {$H$};
			\coordinate (x) at  (1,2);
			\coordinate (x_projection) at (1,3);
			
			\fill (x_projection) circle[radius=1.5pt] node[anchor=south] {$\Pi_H(\mathbf{x})$};
			\draw[decorate,decoration={brace,raise=1pt,amplitude=3pt}] (x) -- node[left]{$\vert\vecprod{\mathbf{x}, \mathbf{u}}\vert$} (x_projection);
			\draw[thick, densely dashed] (x) -- node[left]{$\vert\vecprod{\mathbf{x}, \mathbf{u}}\vert$} (x_projection);
			\coordinate (O) at  (4,3);
			\draw[draw=black,densely dotted,thick] (O) circle[radius=\r]; 
			

			\coordinate (H) at (3.1339745962155616,3.5);
			
			
			
			\draw[-stealth,very thick] (0,3) -- (0,4) node[anchor=south] {$\mathbf{u}$};
			\coordinate (tangent) at (2.3101, 4.0697);
			\coordinate (tangent_other) at (3.2899, 1.1303);
			\fill (tangent) circle[radius=1.5pt] node
			[anchor=south,color=red] {$\mathbf{k}^*$};
			\fill (tangent_other) circle[radius=1.5pt];
			\coordinate (interset_tangent) at (1.6330, 3.0000);	
			\coordinate (y) at (1.88,3);
			
			\coordinate (k_prime) at (2.4,3.590909090909091);
			\draw[thick] (x) -- (k_prime);
			\fill (k_prime) circle[radius=1.5pt] node[anchor=north]{$\mathbf{k}'$};
			\node at (6.3,4.3) {$\Pi_V(B)$};
			\draw[thick] (k_prime) -- (O);
			\draw[thick] (x) -- (O);
			\draw[decorate,decoration={brace,raise=1pt,amplitude=4pt,mirror},draw=blue] (O) -- (tangent);
			\draw[blue,densely dashed,thick] (O) -- node[above right] {$R$}  (tangent);
			\draw[blue,densely dashed,thick] (O) -- (tangent_other);
			\draw[red,densely dashed,thick] (x) -- (tangent_other);
			\draw[red,densely dashed,thick] (x) -- (tangent);
			\zjbjblue[4pt]{O}{tangent}{x};
			\zjbjblue[4pt]{O}{tangent_other}{x};
			\pic [draw=black!50!green,text=black!50!green, "$\theta_1$", angle eccentricity=1.3,angle radius=0.45cm] {angle = O--x--x_projection};
			\pic [draw=red,text=red, "$\theta_2$", angle eccentricity=1.3,angle radius=0.65cm] {angle = O--x--k_prime};
			\draw[fill=yellow,nearly transparent]  (x) -- (5,8.319212273872225) -- (6.6,-0.12687016900301318) -- cycle;  
			\node[rectangle,rounded corners,text=black,
			fill=yellow!20,
			node font={\bfseries}]
			(C) at (0,1.3) {Cone $\Pi_V(C_0)$};
			\fill[blue] (y) circle[radius=1.5pt] node[anchor=north west]{$\mathbf{y}_{\mathbf{k}'}$};
			\fill (O) circle[radius=1.5pt] node[anchor=north west] {$\mathbf{x}_{t-1}$};
			\fill (x) circle[radius=1.5pt] node[anchor=east] {$\mathbf{x}$};
			\end{tikzpicture}
			\subcaption{$\mathbf{y}_{\mathbf{k}'}$ and $\Pi_H(\mathbf{x})$ are on the same side of $\mathbf{x}_{t-1}$.}
			\label{subfig:same_side}
		\end{minipage}
		\begin{minipage}[b]{.49\textwidth}
			\centering
			\begin{tikzpicture}[scale=0.8]
			\def\r{2.0}
			\newcommand\zjbjblue[4][7pt]{%
				\draw[color=blue] let \p1=(#2),\p2=(#4),\p0=(#3) in
				(#3)++({atan2(\y1-\y0,\x1-\x0)}:#1)
				--++({atan2(\y2-\y0,\x2-\x0)}:#1)
				--++({atan2(\y1-\y0,\x1-\x0)}:-#1);
			}
			
			\draw[very thick] (-1,3) -- (7,3) ;
			\node at (-1.3,2.9) {$H$};
			\coordinate (x) at  (1,2);
			\coordinate (x_projection) at (1,3);
			
			\fill (x_projection) circle[radius=1.5pt] node[anchor=south] {$\Pi_H(\mathbf{x})$};
			\draw[decorate,decoration={brace,raise=1pt,amplitude=3pt}] (x) -- node[left]{$\vert\vecprod{\mathbf{x}, \mathbf{u}}\vert$} (x_projection);
			\draw[thick, densely dashed] (x) -- node[left]{$\vert\vecprod{\mathbf{x}, \mathbf{u}}\vert$} (x_projection);
			\coordinate (O) at  (4,3);
			\draw[draw=black,densely dotted,thick] (O) circle[radius=\r]; 
			

			\coordinate (H) at (3.1339745962155616,3.5);
			
			
			
			\draw[-stealth,very thick] (0,3) -- (0,4) node[anchor=south] {$\mathbf{u}$};
			\coordinate (tangent) at (2.3101, 4.0697);
			\coordinate (tangent_other) at (3.2899, 1.1303);
			\fill (tangent) circle[radius=1.5pt] node
			[anchor=south,color=red] {$\mathbf{k}^*$};
			\fill (tangent_other) circle[radius=1.5pt];
			\coordinate (interset_tangent) at (1.6330, 3.0000);	
			\coordinate (y) at (4.815384615384614,3);
			
			\coordinate (k_prime) at (5.96,3.3);
			\draw[thick] (x) -- (k_prime);
			\fill (k_prime) circle[radius=1.5pt] node[anchor=south west]{$\mathbf{k}'$};
			\node at (6.3,4.3) {$\Pi_V(B)$};
			\draw[thick] (k_prime) -- (O);
			\draw[thick] (x) -- (O);
			\draw[decorate,decoration={brace,raise=1pt,amplitude=4pt,mirror},draw=blue] (O) -- (tangent);
			\draw[blue,densely dashed,thick] (O) -- node[above right] {$R$}  (tangent);
			\draw[blue,densely dashed,thick] (O) -- (tangent_other);
			\draw[red,densely dashed,thick] (x) -- (tangent_other);
			\draw[red,densely dashed,thick] (x) -- (tangent);
			\zjbjblue[4pt]{O}{tangent}{x};
			\zjbjblue[4pt]{O}{tangent_other}{x};
			\pic [draw=black!50!green,text=black!50!green, "$\theta_1$", angle eccentricity=1.3,angle radius=0.5cm] {angle = O--x--x_projection};
			\pic [draw=red,text=red, "$\theta_2$", angle eccentricity=1.2,angle radius=1.2cm] {angle = k_prime--x--O};
			\draw[fill=yellow,nearly transparent]  (x) -- (5,8.319212273872225) -- (6.6,-0.12687016900301318) -- cycle;  
			\node[rectangle,rounded corners,text=black,
			fill=yellow!20,
			node font={\bfseries}]
			(C) at (0,1.3) {Cone $\Pi_V(C_0)$};
			\fill (x) circle[radius=1.5pt] node[anchor=east] {$\mathbf{x}$};
			\fill (O) circle[radius=1.5pt] node[anchor=south west]
			{$\mathbf{x}_{t-1}$};
			\fill[blue] (y) circle[radius=1.5pt] node[anchor=north west]{$\mathbf{y}_{\mathbf{k}'}$};
			
			\end{tikzpicture}
			\subcaption{$\mathbf{y}_{\mathbf{k}'}$ and $\Pi_H(\mathbf{x})$ are on different sides of $\mathbf{x}_{t-1}$.}
			\label{subfig:diff_side}
		\end{minipage}
		\caption{Illustration of the problem reduced to the plane $V$.}
		\label{fig:2d}
	\end{figure}
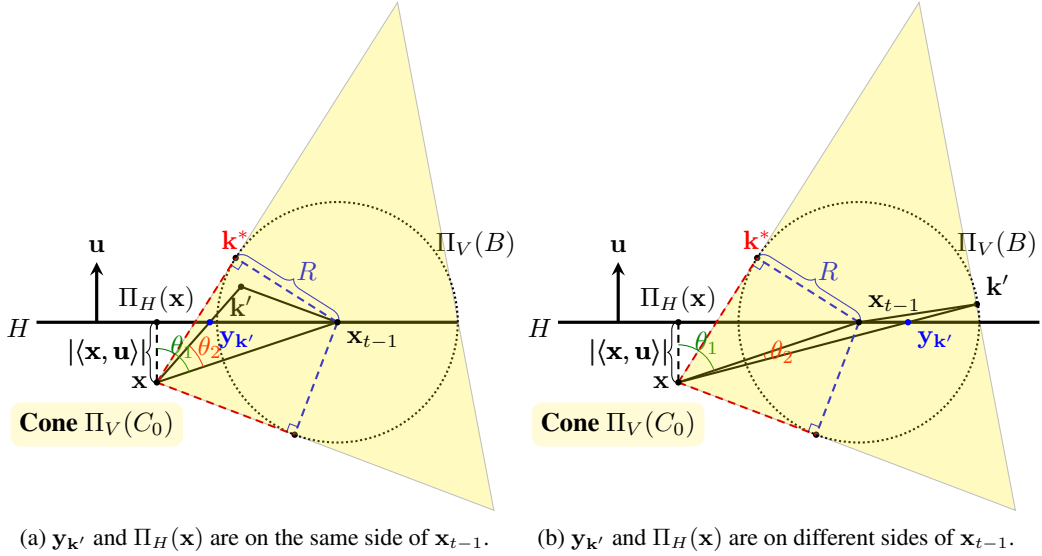

	In the first case (Fig. \ref{subfig:same_side}), $\mathbf{y}_{\mathbf{k}'}$ and $\Pi_H(\mathbf{x})$ are on the same side of $\mathbf{x}_{t-1}$.
	By Assumption \ref{asspt:x-position}, we know that $\|\Pi_H(\mathbf{x})\| \ge R$, so $\cos(-\mathbf{x},\mathbf{u})=\sqrt{1-\|\Pi_H(\mathbf{x})\|^2 \mathbin{/} \|\mathbf{x}\|^2}\leq\sqrt{1-R^2\mathbin{/}\|\mathbf{x}\|^2}$.
	According to the definition of the convex cone $C_0$, $\mathbf{u}$ is outside $C_0$. 
	Notice that $\mathbf{x} \in C_0$ and $\mathbf{k}' \in \Pi_V(B')$, hence $\mathbf{k}'-\mathbf{x}$ is in the convex cone $\Pi_V(C_0)$.
	Therefore, based on the positions of the two vectors $\mathbf{u}$ and $\mathbf{k}'-\mathbf{x}$ with respect to the cone $\Pi_V(C_0)$, we conclude that $\theta_1 \ge \theta_2$.
	In such case, the distance function is $f(\mathbf{k'})=\|\mathbf{y}_{\mathbf{k'}}-\mathbf{x}\|=|\vecprod{\mathbf{x},\mathbf{u}}| \mathbin{/} \cos(\theta_1-\theta_2)$, as shown in Fig. \ref{subfig:same_side}. Because both $\mathbf{x}$ and $\mathbf{u}$ are fixed, the value of $\theta_1$ is fixed. Therefore, the only way to minimize $f(\mathbf{k}')$ is to maximize $\theta_2$.
	Among all possible choices of $\mathbf{k}'$ in $\Pi_V(B')$, the $\mathbf{k}'$ that maximizes the angle $\theta_2$ appears on the boundary of $\Pi_V(C_0) \cap H_{\ge 0}$.
	The only point that satisfies this condition is the tangent point $\mathbf{k}^*$.
	Therefore, in the first case, $\argmin_{\mathbf{k}\in B'}f(\mathbf{k})=\mathbf{k}^*$.
	
	In the second case (Fig. \ref{subfig:diff_side}), $\mathbf{y}_{\mathbf{k}'}$ and $\Pi_H(\mathbf{x})$ are on different sides of $\mathbf{x}_{t-1}$. 
	In this case, $\theta_2 \ge 0$. In particular, when $\theta_2 = 0$, $\mathbf{y}_{\mathbf{k}'}$ and $\mathbf{x}_{t-1}$ coincide.
	According to Assumption \ref{asspt:x-position}, $\theta_1 > 0$.
	The distance function can be defined as $f(\mathbf{k}') = \| \mathbf{y}_{\mathbf{k}'} - \mathbf{x} \| = |\vecprod{\mathbf{x},\mathbf{u}}| \mathbin{/}  \cos(\theta_1 + \theta_2)$ in this case.
	Because $\theta_1 > 0$ and $\theta_2 \ge 0$, the following inequality holds:
	\begin{equation}
	f(\mathbf{k}') = \frac{|\vecprod{\mathbf{x},\mathbf{u}}|}{\cos(\theta_1 + \theta_2)} \geq \frac{|\vecprod{\mathbf{x},\mathbf{u}}|}{\cos(\theta_1)} \geq \frac{|\vecprod{\mathbf{x},\mathbf{u}}|}{\cos(\theta_1 - \theta_2)}.
	\end{equation}
	According to the above inequality, the distance obtained from the second case is greater than or equal to the distance in the first case, and the distances in both cases are equal only if $\theta_2 = 0$.
	Therefore, we can still conclude that $f(\mathbf{k}') \ge f(\mathbf{k}^*)$, \textit{i.e.,} $\argmin_{\mathbf{k}\in B'}f(\mathbf{k})=\mathbf{k}^*$.

	Finally, we need to prove $\argmax_{\mathbf{k}\in (B'\cap C)}\vecprod{\mathbf{k},\mathbf{u}}=\mathbf{k}^*$, so that Eq. \eqref{eqn:chain_eq} holds.
	The overall proof process is similar to the above proof, except that all $f(\mathbf{k})$ in the above proof need to be replaced by $\vecprod{\mathbf{k},\mathbf{u}}$. Correspondingly, Eq. \eqref{eqn:proof_last_goal} needs to be changed to the following formula:
	\begin{equation}
	\label{eqn:substitute_last_goal}
	\vecprod{\mathbf{k}^*,\mathbf{u}} \ge  \vecprod{\Pi_V(\mathbf{k}),\mathbf{u}} = \vecprod{\mathbf{k},\mathbf{u}}.
	\end{equation}
	Firstly, let us prove the equality part of Eq. \eqref{eqn:substitute_last_goal}: when projecting any $\mathbf{k} \in (B'\cap C)$ onto the plane $V$, the value of $\vecprod{\mathbf{k}, \mathbf{u}}$ does not change. Thus, we have $\vecprod{\Pi_V(\mathbf{k}),\mathbf{u}} = \vecprod{\mathbf{k},\mathbf{u}}$.
	Secondly, we prove the inequality part of Eq. \eqref{eqn:substitute_last_goal}: $\vecprod{\mathbf{k}^*,\mathbf{u}} \ge \vecprod{\Pi_V(\mathbf{k}),\mathbf{u}}$.
	Now the problem is reduced to the plane $V$ again. Because $\Pi_V(\mathbf{k}) \in (B'\cap C \cap V)$,
	only the first case mentioned above can happen (Fig.~\ref{subfig:same_side}).  
	By a similar argument, we conclude that $\argmax_{\mathbf{k}\in (B'\cap C)}\vecprod{\mathbf{k},\mathbf{u}}=\mathbf{k}^*$ holds, which proves Lemma \ref{lemma}. Consequently, Theorem~\ref{theorm:1} holds.	
\end{proof}

\section{Experimental Settings}
In this section, we provide the hyperparameter settings of the compared methods, \textit{i.e.,} HSJA \cite{chen2019hopskipjumpattack}, BA~\cite{brendel2018decisionbased}, Sign-OPT \cite{cheng2019sign}, and SVM-OPT \cite{cheng2019sign}.

\begin{table}[h]
	\small
	\tabcolsep=0.1cm
	\setlength{\abovecaptionskip}{0pt}%
	\setlength{\belowcaptionskip}{0pt}%
	\caption{The hyperparameters of HSJA.}
	\label{tab:HSJA}
	\begin{center}
		\begin{tabular}{c|p{8cm}|c}
			\toprule
			Dataset & \multicolumn{1}{c|}{Hyperparameter} & Value \\
			\midrule
			\multirow{5}{*}{CIFAR-10} & $\gamma$, threshold of the binary search & 1.0\\
			& $B_0$, the initial batch size for gradient estimation & 100 \\
			& $B_\text{max}$, the maximum batch size for gradient estimation  & \nn{10000} \\
			& the search method for step size  & geometric progression \\
			& number of iterations & 64 \\
			\Xcline{1-3}{0.1pt}
			\multirow{5}{*}{ImageNet} & $\gamma$, threshold of the binary search & \nn{1000.0}\\
			& $B_0$, the initial batch size for gradient estimation  & 100 \\
			& $B_\text{max}$, the maximum batch size for gradient estimation  & \nn{10000} \\
			& the search method for step size & geometric progression \\
			& number of iterations & 64 \\
			\bottomrule
		\end{tabular}
	\end{center}
\end{table}

\begin{table}[h]
	\small
	\tabcolsep=0.1cm
	\setlength{\abovecaptionskip}{0pt}%
	\setlength{\belowcaptionskip}{0pt}%
	\caption{The hyperparameters of BA.}
	\label{tab:BA}
	\begin{center}
		\begin{tabular}{p{10cm}|c}
			\toprule
			\multicolumn{1}{c|}{Hyperparameter} & Value \\
			\midrule
			maximum number of trials per iteration & 25 \\
			number of iterations & \nn{1200} \\
			spherical step size & 0.01 \\
			source step size & 0.01 \\
			step size adaptation multiplier & 1.5 \\
			disable automatic batch size tuning & False \\
			generate candidates and random numbers without using multithreading  & False \\
			
			\bottomrule
		\end{tabular}
	\end{center}
\end{table}

\begin{table}[!h]
	\small
	\tabcolsep=0.1cm
	\setlength{\abovecaptionskip}{0pt}%
	\setlength{\belowcaptionskip}{0pt}%
	\caption{The hyperparameters of Sign-OPT.}
	\label{tab:SignOPT}
	\begin{center}
		\begin{tabular}{p{12cm}|c}
			\toprule
			\multicolumn{1}{c|}{Hyperparameter} & Value \\
			\midrule
			$k$, number of queries for estimating an approximate gradient & 200 \\
			$\alpha$, the update step size of the direction $\theta$ & 0.2 \\
			$\beta$, used for the gradient estimation of $\theta$ and determining the stopping threshold of binary search & 0.001 \\
			the number of iterations & \nn{1000} \\
			the binary search's stopping threshold of the CIFAR-10 dataset & $\frac{\beta}{500}$ \\
			the binary search's stopping threshold of the ImageNet dataset & $1 \times 10^{-4}$ \\
			\bottomrule
		\end{tabular}
	\end{center}
\end{table}
\begin{table}[!h]
	\small
	\tabcolsep=0.1cm
	\setlength{\abovecaptionskip}{0pt}%
	\setlength{\belowcaptionskip}{0pt}%
	\caption{The hyperparameters of SVM-OPT.}
	\label{tab:SVMOPT}
	\begin{center}
		\begin{tabular}{p{12cm}|c}
			\toprule
			\multicolumn{1}{c|}{Hyperparameter} & Value \\
			\midrule
			$k$, number of queries for estimating gradients & 100 \\
			$\alpha$, the step size of the gradient descent of $\theta$ & 0.2 \\
			$\beta$, used for the gradient estimation of $\theta$ and determining the stopping threshold of binary search & 0.001 \\
			the number of iterations & \nn{1000} \\
			the binary search's stopping threshold of the CIFAR-10 dataset & $\frac{\beta}{500}$ \\
			the binary search's stopping threshold of the ImageNet dataset & $1 \times 10^{-4}$
			\\
			\bottomrule
		\end{tabular}
	\end{center}
\end{table}

\textbf{Experimental Equipment.} The experiments of all compared methods are conducted by using PyTorch 1.7.1 framework on a NVIDIA 1080Ti GPU.

\textbf{HSJA.} Hyperparameters of HSJA \cite{chen2019hopskipjumpattack} are listed in Table \ref{tab:HSJA}. We translate the implementation code into the PyTorch version for the experiments. In the experiments of targeted attacks, we randomly select an image from the target class as the initial adversarial example. For fair comparison, we set the hyperparameters of TA and G-TA to be the same with HSJA, \textit{i.e.,} the same initial batch size $B_0$ and the same $\gamma$.

\textbf{BA.} Hyperparameters of BA \cite{brendel2018decisionbased} are listed in Table \ref{tab:BA}. In the experiments, we directly use the implementation of BA from Foolbox 2.0 \cite{rauber2017foolbox,rauber2017foolboxnative}, and adopt a randomly selected image from the target class as the initialization in the targeted attack.

\textbf{Sign-OPT and SVM-OPT.} Hyperparameters of Sign-OPT \cite{cheng2019sign} and SVM-OPT \cite{cheng2019sign} are listed in Tables~\ref{tab:SignOPT} and \ref{tab:SVMOPT}. We translate the implementation code into the PyTorch version for the experiments. In the experiments of targeted attacks, we set the initial direction $\theta_0$ of Sign-OPT and SVM-OPT to the direction of a randomly selected image of the target class.

\section{Experimental Results}

\subsection{Limitation of Tangent Attack}

The proposed approach supports all types of attacks, including both untargeted and targeted attacks under the both $\ell_2$ and $\ell_\infty$ norm constraints. This is the strength of the proposed approach.
However, in the $\ell_\infty$ norm attack, TA and G-TA obtain the similar performance to the baseline method HSJA. Because under the definition of the $\ell_\infty$ norm distance: $D_{\ell_\infty}(x,y):=\max_i(|x_i - y_i|), i \in \{1,\ldots, d\}$ ($d$ is the image dimension), the intersection of the tangent line and the decision boundary may not be the one with the shortest $\ell_\infty$ norm distance to the original image.
Therefore, searching the boundary sample along the tangent line cannot always outperform HSJA in the $\ell_\infty$ norm attack.



Tables \ref{tab:Linf_ImageNet_normal_models_result} and \ref{tab:Linf_CIFAR_normal_models_result} demonstrate the experimental results of attacking against undefended models on the CIFAR-10 and ImageNet datasets.

\begin{table}[h]
	\centering
	\small
	\tabcolsep=0.1cm
	\setlength{\belowcaptionskip}{0.5pt}%
	\caption{Mean $\ell_\infty$ distortions of different query budgets on the ImageNet dataset, where the radius ratio $r$ is set to 1.1 in G-TA. BA is not applicable to the $\ell_\infty$ norm attack, hence it is not listed.}
	\scalebox{0.8}{
		\begin{tabular}{cccccccc|cccccc}
			\toprule
			Target Model  & Method  & \multicolumn{6}{c|}{Targeted Attack} & \multicolumn{6}{c}{Untargeted Attack} \\
			& & @300&  @1K & @2K & @5K & @8K & @10K & @300&  @1K & @2K & @5K & @8K & @10K \\
			\midrule
			\multirow{5}{*}{Inception-v3}& Sign-OPT \cite{cheng2019sign} & 0.557 & 0.519 & 0.481 & 0.421 & 0.390 & 0.375 & 1.078 & 0.792 & 0.548 & 0.328 & 0.262 & 0.239\\
			& SVM-OPT \cite{cheng2019sign} &  0.558 & 0.512 & 0.476 & 0.423 & 0.397 & 0.385 & 1.079 & 0.763 & 0.526 & 0.336 & 0.280 & 0.260\\
			& HSJA \cite{chen2019hopskipjumpattack} &  0.370 & 0.330 & \B 0.289 & \B 0.211 & \B 0.169 & \B 0.147 & 0.305 & 0.236 & 0.174 & \B 0.093 & 0.069 & \B 0.059 \\                    
			& TA&   0.370 &0.330 & 0.291 & 0.216 & 0.172 & 0.149 & \B 0.304 & \B 0.234 & \B 0.173 & \B 0.093 & \B 0.068 & \B 0.059\\
			& G-TA &  \B 0.364 & \B 0.326 &\B 0.289 & 0.220 & 0.179 & 0.159 & \B 0.304 & 0.238 & 0.174 & \B 0.093 &\B 0.068 & \B 0.059\\
			\Xcline{1-14}{0.1pt}
			\multirow{5}{*}{Inception-v4}  & Sign-OPT \cite{cheng2019sign} &  0.545 & 0.504 & 0.464 & 0.402 & 0.370 & 0.355 & 1.176 & 0.867 & 0.603 & 0.369 & 0.296 & 0.270\\
			& SVM-OPT \cite{cheng2019sign} & 0.547 & 0.498 & 0.460 & 0.406 & 0.379 & 0.367 & 1.181 & 0.842 & 0.588 & 0.381 & 0.319 & 0.296 \\
			& HSJA \cite{chen2019hopskipjumpattack} &  0.357 & \B 0.324 & \B 0.287 & \B 0.215 & \B 0.175 & \B 0.152 & \B 0.336 & \B 0.257 & \B 0.185 & \B 0.091 & \B 0.060 & \B 0.048\\
			& TA & \B 0.354 & 0.328 & 0.294 & 0.221 & 0.182 & 0.161 & 0.337 & 0.264 & 0.196 & 0.103 & 0.073 & 0.062 \\
			& G-TA & \B 0.354 & \B 0.324 & 0.290 & 0.220 & 0.182 & 0.162 & 0.337 & 0.264 & 0.196 & 0.104 & 0.073 & 0.061 \\
			\Xcline{1-14}{0.1pt}
			\multirow{5}{*}{SENet-154} & Sign-OPT \cite{cheng2019sign} &  0.537 & 0.491 & 0.439 & 0.357 & 0.316 & 0.298  & 0.806 & 0.631 & 0.462 & 0.299 & 0.247 & 0.227\\
			& SVM-OPT \cite{cheng2019sign} & 0.538 & 0.480 & 0.429 & 0.357 & 0.322 & 0.307 & 0.807 & 0.608 & 0.444 & 0.306 & 0.262 & 0.246 \\
			& HSJA \cite{chen2019hopskipjumpattack} &  0.347 & \B 0.288 & \B 0.249 & \B 0.176 & \B 0.139 & \B 0.119 & \B 0.253 &  \B 0.195 &  \B 0.141 & 0.071 & 0.048 & 0.039\\
			& TA &  0.346 & 0.289 & 0.251 & 0.181 & 0.142 & 0.123 & \B 0.253 & 0.196 &  \B 0.141 & 0.071 & \B 0.047 & \B 0.038\\
			& G-TA & \B 0.344 & \B 0.288 &  0.252 &  0.179 &  0.142 &  0.123 & \B 0.253 & 0.196 & \B 0.141 & \B 0.070 & \B 0.047 & \B 0.038\\
			\Xcline{1-14}{0.1pt}
			\multirow{5}{*}{ResNet-101} & Sign-OPT \cite{cheng2019sign} &  0.549 & 0.501 & 0.450 & 0.370 & 0.329 & 0.309 & 0.645 & 0.515 & 0.385 & 0.251 & 0.206 & 0.190 \\
			& SVM-OPT \cite{cheng2019sign} &  0.550 & 0.492 & 0.444 & 0.371 & 0.335 & 0.319 & 0.642 & 0.494 & 0.370 & 0.258 & 0.220 & 0.206\\
			& HSJA \cite{chen2019hopskipjumpattack}& 0.340 & 0.283 & 0.247 & \B 0.179 & \B 0.143 & \B 0.125 & 0.197 & \B 0.140 & \B 0.098 & \B 0.049 & \B 0.034 & \B 0.028 \\
			& TA &  0.340 & 0.282 & \B 0.246 & 0.180 & \B 0.143 & \B 0.125 & \B 0.196 & 0.145 & 0.109 & 0.064 & 0.050 & 0.044 \\
			& G-TA & \B 0.337 & \B 0.280 &  \B 0.246 &0.182 & 0.150 &0.132 & \B  0.196 & 0.147 & 0.110 & 0.064 & 0.050 & 0.044 \\
			\bottomrule
	\end{tabular}	}
	\label{tab:Linf_ImageNet_normal_models_result}
\end{table}
\begin{table}[h]
	\centering
	\small
	\tabcolsep=0.1cm
	\setlength{\belowcaptionskip}{0.5pt}%
	\caption{Mean $\ell_\infty$ distortions of different query budgets on the CIFAR-10 dataset, where the radius ratio $r$ is set to $1.5$ in G-TA. BA is not applicable to the $\ell_\infty$ norm attack, and thus it is not listed.}
	\scalebox{0.8}{
		\begin{tabular}{cccccccc|cccccc}
			\toprule
			Target Model  & Method  & \multicolumn{6}{c|}{Targeted Attack} & \multicolumn{6}{c}{Untargeted Attack} \\
			& & @300&  @1K & @2K & @5K & @8K & @10K & @300&  @1K & @2K & @5K & @8K & @10K \\
			\midrule
			\multirow{5}{*}{PyramidNet-272}& Sign-OPT \cite{cheng2019sign} & 0.395 & 0.318 & 0.237 & 0.134 & 0.096 & 0.082 &  0.284 & 0.189 & 0.115 & 0.059 & 0.047 & 0.043 \\
			& SVM-OPT \cite{cheng2019sign} & 0.390  & 0.299 & 0.226 & 0.134 & 0.099 & 0.087 & 0.286 & 0.173 & 0.104 & 0.059 & 0.049 & 0.046  \\
			& HSJA \cite{chen2019hopskipjumpattack}& \B 0.218  & 0.155 & \B 0.112 & \B 0.057 & \B 0.039 &  0.032 &  \B 0.133 & \B 0.056 & \B 0.034 & \B 0.016 & \B 0.012 & 0.011  \\
			& TA &  0.219  & 0.154 & \B 0.112 & \B 0.057 & \B 0.039 &  0.032 &  0.134  & 0.057 & 0.035 & 0.017 & 0.013 & 0.011  \\
			& G-TA &  \B 0.218  & \B 0.153 & 0.113 & \B 0.057 &\B 0.039 & \B 0.031 & \B 0.133  &\B 0.056 &\B 0.034 &\B 0.016 & \B 0.012 & \B 0.010  \\
			\Xcline{1-14}{0.1pt}
			\multirow{5}{*}{GDAS}  & Sign-OPT \cite{cheng2019sign} & 0.398 & 0.332 & 0.266 & 0.153 & 0.107 & 0.089 & 0.305  & 0.269 & 0.231 & 0.185 & 0.167 & 0.160 \\
			& SVM-OPT \cite{cheng2019sign} & 0.389  & 0.325 & 0.267 & 0.164 & 0.118 & 0.100 & 0.304 & 0.257 & 0.219 & 0.181 & 0.168 & 0.163  \\
			& HSJA \cite{chen2019hopskipjumpattack}&  \B 0.210 & \B 0.147 & \B 0.112 & \B 0.060 & \B 0.040 & \B 0.031 & \B 0.049  & \B 0.029 & \B 0.020 & \B 0.011 & \B 0.009 & \B 0.008  \\
			& TA &  0.214  & 0.151 & 0.115 & 0.062 & 0.041 & 0.032 & \B 0.049 & \B 0.029 &\B 0.020 & \B 0.011 & \B 0.009 & \B 0.008  \\
			& G-TA &  0.214  & 0.151 & 0.116 & 0.062 & 0.041 & 0.032 & \B 0.049 & \B 0.029 & \B 0.020 & \B 0.011 &\B 0.009 & \B 0.008  \\
			\Xcline{1-14}{0.1pt}
			\multirow{5}{*}{WRN-28} & Sign-OPT \cite{cheng2019sign} &  0.402 & 0.307 & 0.225 & 0.121 & 0.086 & 0.074 &  0.200 & 0.130 & 0.085 & 0.053 & 0.044 & 0.041   \\
			& SVM-OPT \cite{cheng2019sign} &  0.382 & 0.296 & 0.223 & 0.128 & 0.093 & 0.080 &  0.201  & 0.121 & 0.079 & 0.052 & 0.045 & 0.043  \\
			& HSJA \cite{chen2019hopskipjumpattack}& \B 0.185 & \B 0.106 &  0.070 & 0.032 & \B 0.021 & \B 0.018 & \B 0.090 & 0.031 & \B 0.020 & \B 0.012 & \B 0.010 & \B 0.009  \\
			& TA &  0.186 & 0.107 & 0.070 & \B 0.031 & \B 0.021 & \B 0.018 &  \B 0.090 & \B 0.030 & \B 0.020 & \B 0.012 & \B 0.010 & \B 0.009  \\
			& G-TA & \B 0.185 & \B 0.106 &\B 0.069 & 0.032 & 0.022 & \B 0.018 & \B 0.090 & \B 0.030 &\B 0.020 & \B 0.012 &\B 0.010 & \B 0.009  \\
			\Xcline{1-14}{0.1pt}
			\multirow{5}{*}{WRN-40} & Sign-OPT \cite{cheng2019sign} &  0.397 & 0.305 & 0.220 & 0.120 & 0.085 & 0.073 &  0.284  & 0.208 & 0.125 & 0.051 & 0.042 & 0.039\\
			& SVM-OPT \cite{cheng2019sign} &  0.381  & 0.293 & 0.220 & 0.126 & 0.092 & 0.079 &  0.273  & 0.190 & 0.120 & 0.057 & 0.045 & 0.041  \\
			& HSJA \cite{chen2019hopskipjumpattack}& \B 0.194 & 0.111 & \B 0.072 & \B 0.032 & \B 0.022 & \B 0.019 &0.084 & 0.030 & \B 0.020 & \B 0.012 & \B 0.010 &\B 0.009  \\
			& TA & 0.195  &0.112 & 0.073 & \B 0.032 & \B 0.022 & \B 0.019 & \B 0.082  & \B 0.029 & \B 0.020 & \B 0.012 & \B 0.010 & \B 0.009  \\
			& G-TA & \B 0.194 & \B 0.110 & \B 0.072 & \B 0.032 &\B 0.022 & \B 0.019 & \B 0.082 & \B 0.029 & \B 0.020 & \B 0.012 & \B 0.010 & \B 0.009  \\
			\bottomrule
	\end{tabular}}
	\label{tab:Linf_CIFAR_normal_models_result}
\end{table}
The results of Tables \ref{tab:Linf_ImageNet_normal_models_result} and \ref{tab:Linf_CIFAR_normal_models_result} show that HSJA,
TA and G-TA obtain the similar average $\ell_\infty$ distortions. Therefore, although the proposed approach is applicable to $\ell_\infty$ norm attack, the performance is similar to that of the baseline method HSJA. 

\subsection{Experimental Results of Attacks against Defense Models}

We also conduct experiments by using $\ell_\infty$ norm attacks to break five defense models, and the experimental results are shown in Table \ref{tab:linf_CIFAR10_defenseive_models}. 
The conclusion drawn from this table is the same as that in Tables \ref{tab:Linf_ImageNet_normal_models_result} and \ref{tab:Linf_CIFAR_normal_models_result}: TA and G-TA obtain the similar performance with HSJA in $\ell_\infty$ norm attacks.
\begin{table}[h]
	\centering
	\small
	\tabcolsep=0.1cm
	\setlength{\belowcaptionskip}{0.5pt}%
	\caption{The experimental results of performing $\ell_\infty$ norm attacks against the defense models on the CIFAR-10 dataset, where the radius ratio $r$ is set to $1.5$ in G-TA.}
	\scalebox{0.8}{
		\begin{tabular}{cccccccc}
			\toprule
			Target Model  & Method  & \multicolumn{6}{c}{Untargeted Attack} \\
			& & @300&  @1K & @2K & @5K & @8K & @10K \\
			\midrule
			\multirow{5}{*}{AT \cite{madry2018towards}}& Sign-OPT \cite{cheng2019sign} & 0.731 &  0.519 &  0.395 &  0.288 &  0.255 & 0.243 \\
			& SVM-OPT \cite{cheng2019sign} & 0.719 &  0.498 &  0.382 &  0.287 &  0.261 & 0.251 \\
			&  HSJA \cite{chen2019hopskipjumpattack}& \B 0.181 & \B 0.145 &  \B 0.121 &  \B 0.090 &  0.080 & \B 0.075 \\
			&  TA &  0.184 &  0.147 & \B 0.121 & \B 0.090 & \B 0.079 & \B 0.075 \\
			&  G-TA &  \B 0.181 &  \B 0.145 & \B 0.121 &  \B 0.090 &  0.080 & \B 0.075 \\
			\Xcline{1-8}{0.1pt}
			\multirow{5}{*}{TRADES \cite{zhang2019theoretically}}  & Sign-OPT \cite{cheng2019sign}& 0.748 &  0.562 &  0.419 &  0.304 &  0.269 & 0.257  \\
			& SVM-OPT \cite{cheng2019sign}& 0.743 &  0.534 &  0.409 &  0.308 &  0.281 & 0.271 \\
			&  HSJA \cite{chen2019hopskipjumpattack}& \B 0.194 & \B 0.162 &  \B 0.137 & \B 0.106 & \B 0.095 & \B 0.090 \\
			&  TA &  0.195 &  0.163 &  0.138 &  0.107 & \B 0.095 & \B 0.090  \\
			&  G-TA  &  \B 0.194 &  0.163 &  0.138 &  0.107 & \B 0.095 & \B 0.090 \\
			\Xcline{1-8}{0.1pt}
			\multirow{5}{*}{JPEG \cite{guo2018countering}}& Sign-OPT \cite{cheng2019sign}&0.301 &  0.292 &  0.281 &  0.262 &  0.250 & 0.245  \\
			& SVM-OPT \cite{cheng2019sign}& 0.301 &  0.288 &  0.275 &  0.256 &  0.249 & 0.246 \\
			& HSJA \cite{chen2019hopskipjumpattack}&  0.094 & \B 0.086 & \B 0.078 & \B 0.066 &  \B 0.061 & \B 0.058 \\
			&  TA &  \B 0.093 &  0.087 &  0.080 &  0.067 & \B 0.061 & \B 0.058 \\
			&  G-TA &  0.097 &  0.091 &  0.081 &  0.068 &  0.062 & 0.059 \\
			\Xcline{1-8}{0.1pt}
			\multirow{5}{*}{Feature Distillation \cite{liu2019feature}} &  Sign-OPT \cite{cheng2019sign}&  0.344 &  0.330 &  0.317 &  0.290 &  0.273 & 0.266  \\
			& SVM-OPT \cite{cheng2019sign}& 0.354 &  0.338 &  0.323 &  0.297 &  0.284 & 0.279 \\
			& HSJA \cite{chen2019hopskipjumpattack}&  0.090 &  0.087 &  0.080 &   0.069 &  0.064 & 0.061 \\
			&  TA & \B 0.089 & \B 0.086 &  \B 0.079 &  0.070 &  0.063 & 0.060 \\
			&  G-TA &  0.090 & \B 0.086 &  \B 0.079 &  \B 0.067 &\B  0.062 & \B 0.059 \\
			\Xcline{1-8}{0.1pt}
			\multirow{5}{*}{Feature Scatter \cite{zhang2019defense}} & Sign-OPT \cite{cheng2019sign} & 0.561 &  0.380 &  0.246 &  0.135 &  0.110 & 0.101 \\
			& SVM-OPT \cite{cheng2019sign}& 0.550 &  0.344 &  0.222 &  0.137 &  0.116 & 0.110 \\
			&  HSJA \cite{chen2019hopskipjumpattack}& \B 0.202 & \B 0.137 & \B 0.104 &  \B 0.062 & \B 0.048 & \B 0.042 \\
			&  TA & \B 0.202 &  \B 0.137 & \B 0.104 & \B 0.062 & \B 0.048 & \B 0.042 \\
			&  G-TA &  0.205 &  0.139 &  0.105 & \B 0.062 & \B 0.048 & \B 0.042 \\
			\bottomrule
	\end{tabular}}
	\label{tab:linf_CIFAR10_defenseive_models}
\end{table}

Next, we conduct experiments by using $\ell_2$ norm attack to break different defense models on the CIFAR-10 and ImageNet datasets.
In the CIFAR-10 dataset, we select six types of defense models:
\begin{itemize}
	\item Adversarial Training (AT) \cite{madry2018towards}: the most effective defense method, which uses adversarial examples as the training data to obtain the robust classifier.
	\item TRADES \cite{zhang2019theoretically}: an improved AT that optimizes a regularized surrogate loss.
	\item JPEG \cite{guo2018countering}: a standard image compression algorithm based on the discrete cosine transform, which can remove the adversarial perturbations, thereby providing some degree of defense.
	\item Feature Distillation \cite{liu2019feature}: a defense method based on the improved JPEG image compression. Its defense mechanism is divided into two steps. 
	Firstly, it filters out adversarial perturbations by using a semi-analytical method. Secondly,
	it restores the classification accuracy of benign images by using a DNN-oriented quantization process.
	\item Feature Scatter \cite{zhang2019defense}: a feature scattering-based AT method, which is an unsupervised approach for generating adversarial examples during the training.
	\item ComDefend \cite{jia2019comdefend}: a defense model that consists of a compression CNN and a reconstruction CNN to transform the adversarial image into its clean version to defend against attacks.
\end{itemize}

In the ImageNet dataset, we directly use the publicly available AT models for experiments, all of which use the ResNet-50 networks as their backbones.
The pre-trained weights can be downloaded from  \url{https://github.com/MadryLab/robustness}.
In the experiments, we set the radius ratio $r$ of G-TA to 1.5, and the experimental results are shown in Fig. \ref{fig:attack_adv_train_ImageNet}.
In untargeted attacks (Figs. \ref{fig:ImageNet_untargeted_adv_train_L2}, \ref{fig:ImageNet_untargeted_adv_train_Linf}, \ref{fig:ImageNet_untargeted_adv_train_Linf_8}), the G-TA (the semi-ellipsoid version) outperforms the TA (the hemisphere version), and the baseline method HSJA outperforms TA and G-TA. 
We conjecture that it is because the classification decision boundaries of the AT models on the ImageNet dataset are extremely curved in untargeted attacks, resulting in the better performance of HSJA.
In targeted attacks (Figs. \ref{fig:ImageNet_targeted_adv_train_L2}, \ref{fig:ImageNet_targeted_adv_train_Linf}, \ref{fig:ImageNet_targeted_adv_train_Linf_8}), both TA and G-TA outperform HSJA in the attacks of different AT models. These results indicate that TA and G-TA are more suitable for the targeted attack. 
Another interesting finding is that SVM-OPT performs better in untargeted attacks while Sign-OPT performs better in targeted attacks.
We will explore the reasons for these results in the future work. 

\begin{figure}[t]
	\captionsetup[sub]{font=small}
	\setlength{\abovecaptionskip}{0pt}
	\setlength{\belowcaptionskip}{0pt}
	\centering 
	\begin{minipage}[b]{.3\textwidth}
		\includegraphics[width=\linewidth]{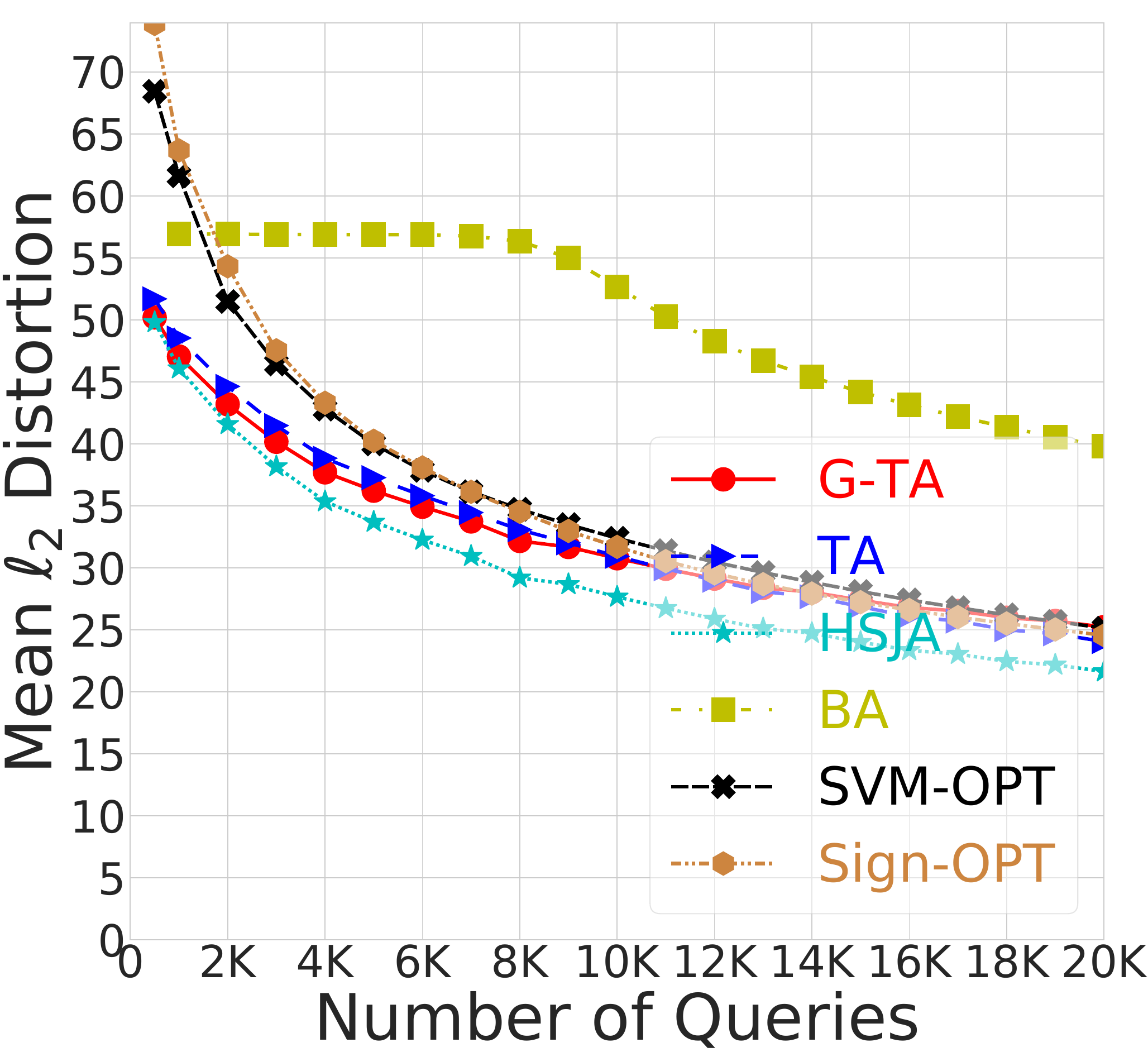}
		\subcaption{AT ($\ell_2$ norm $\epsilon = 3.0$)}
		\label{fig:ImageNet_untargeted_adv_train_L2}
	\end{minipage}
	\begin{minipage}[b]{.3\textwidth}
		\includegraphics[width=\linewidth]{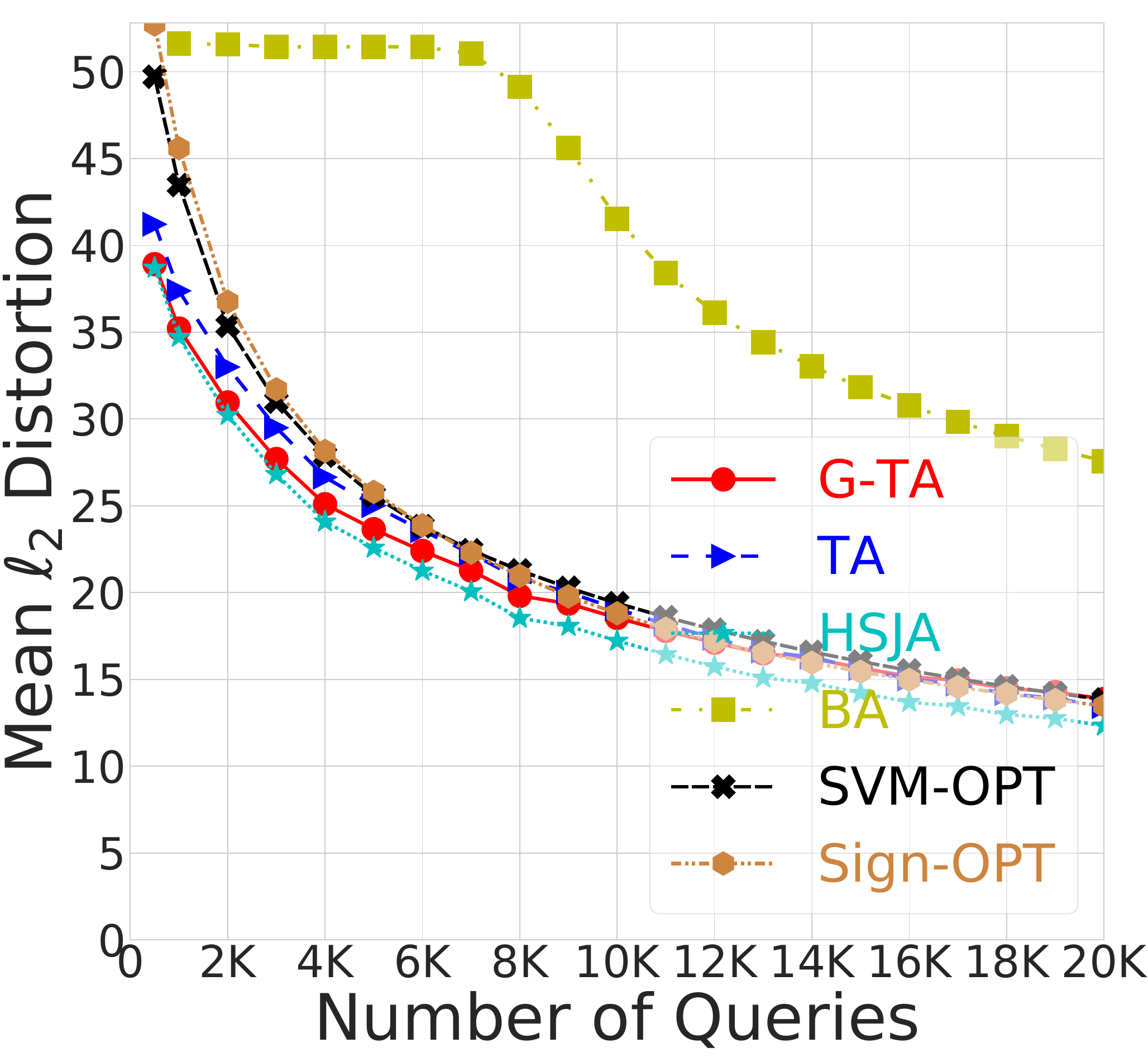}
		\subcaption{AT ($\ell_\infty$ norm $\epsilon = 4/255$)}
		\label{fig:ImageNet_untargeted_adv_train_Linf}
	\end{minipage}
	\begin{minipage}[b]{.3\textwidth}
		\includegraphics[width=\linewidth]{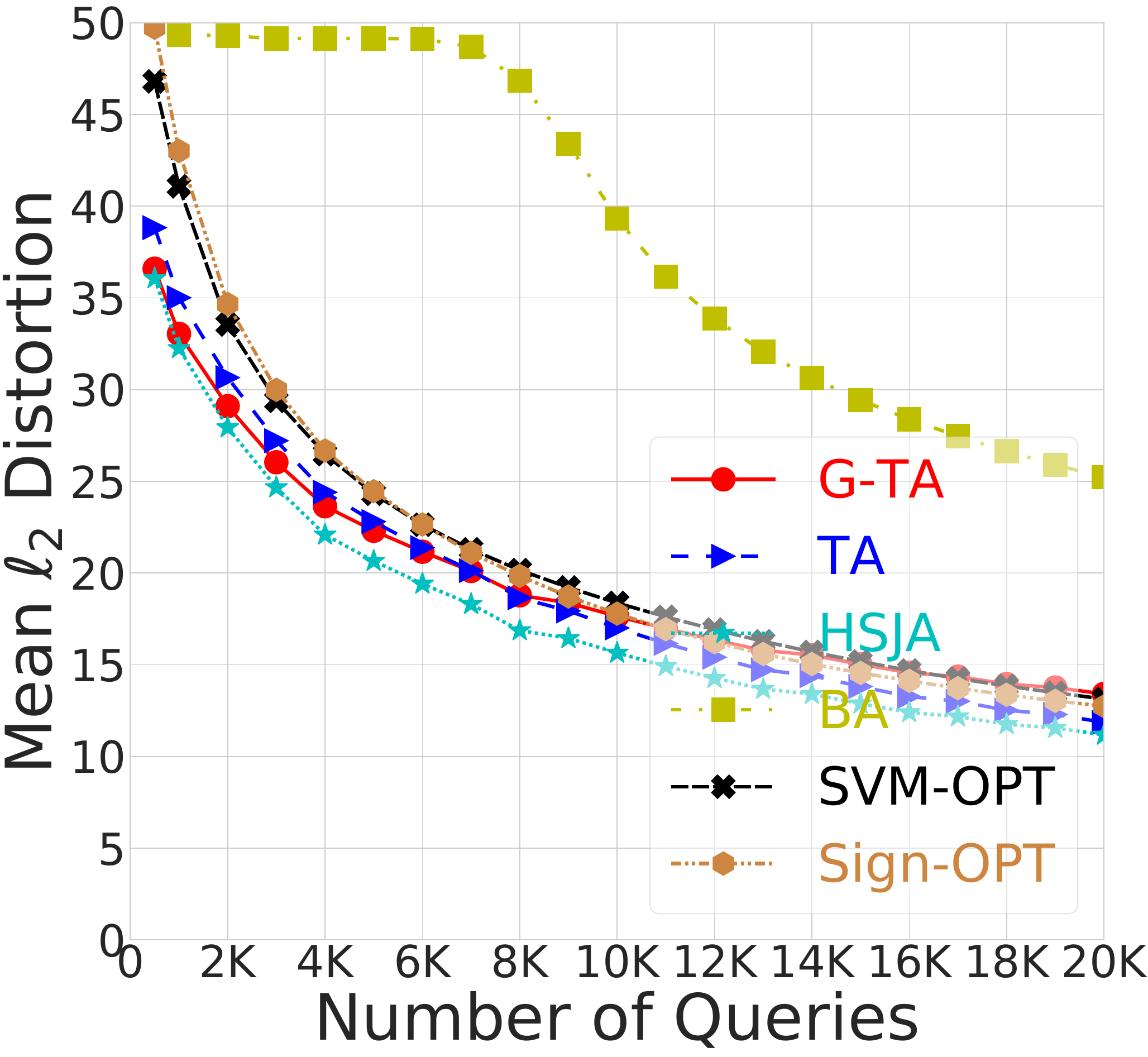}
		\subcaption{AT ($\ell_\infty$ norm $\epsilon = 8/255$)}
		\label{fig:ImageNet_untargeted_adv_train_Linf_8}
	\end{minipage}
	\begin{minipage}[b]{.3\textwidth}
		\includegraphics[width=\linewidth]{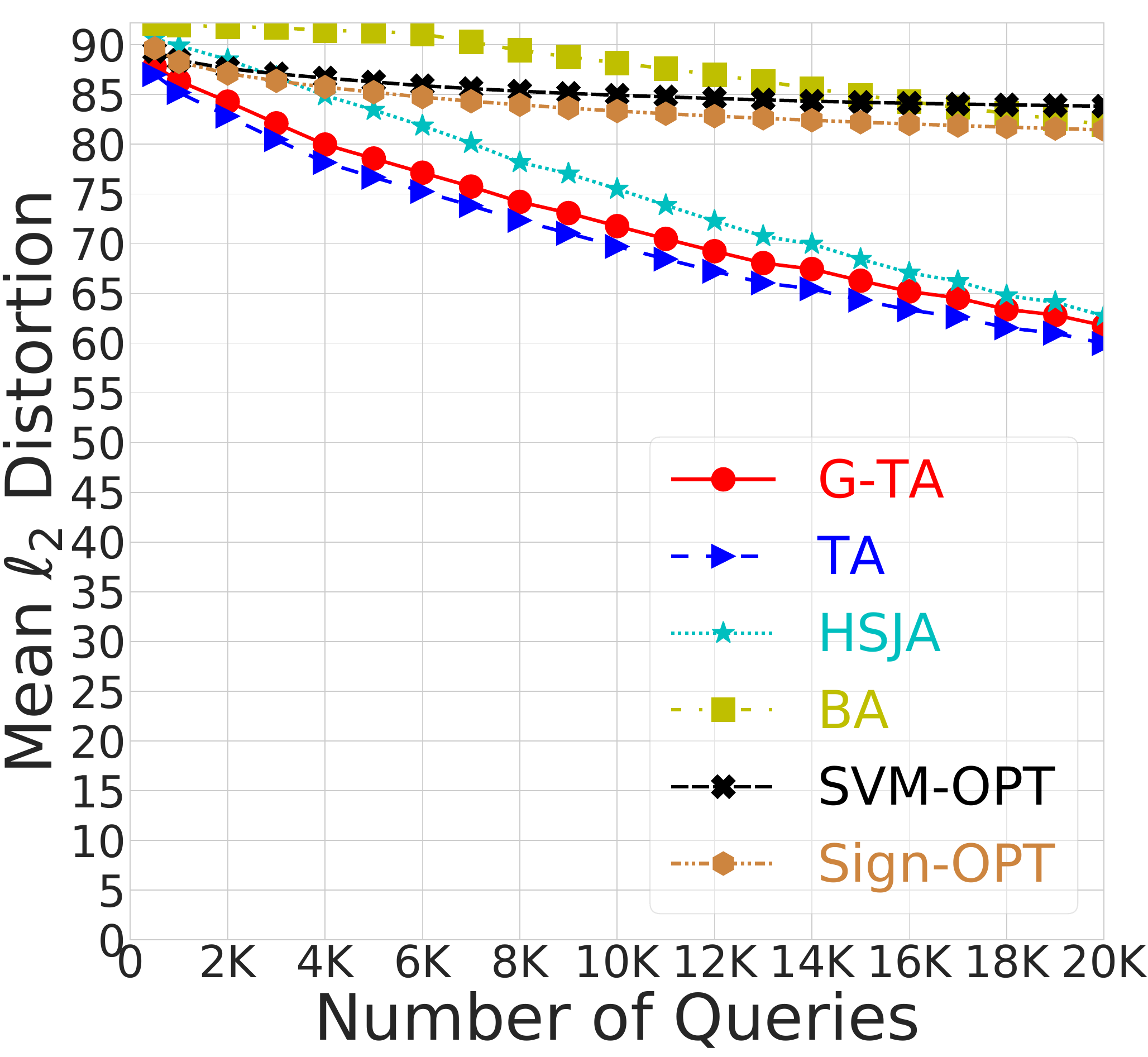}
		\subcaption{AT ($\ell_2$ norm $\epsilon = 3.0$)}
		\label{fig:ImageNet_targeted_adv_train_L2}
	\end{minipage}
	\begin{minipage}[b]{.3\textwidth}
		\includegraphics[width=\linewidth]{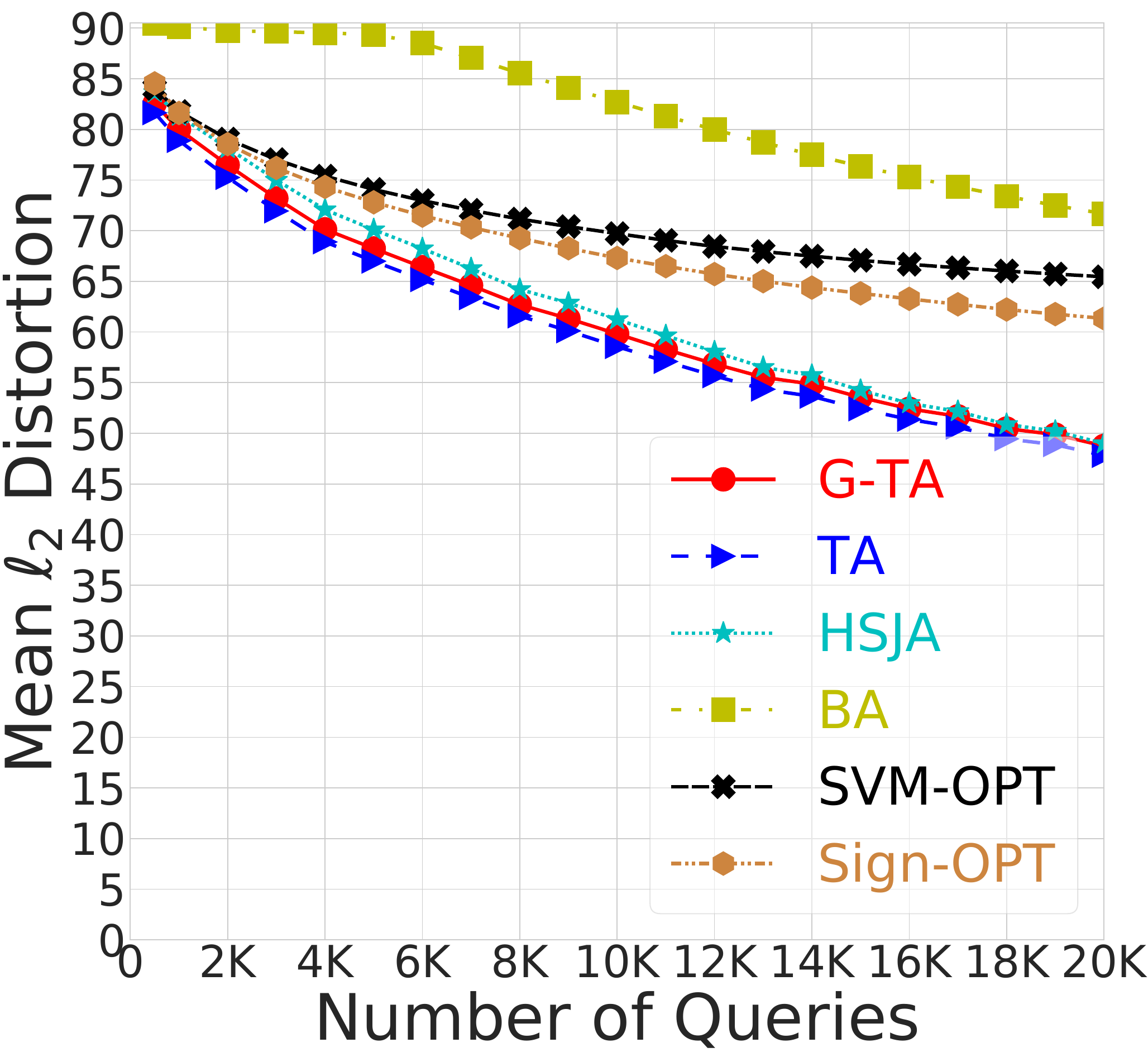}
		\subcaption{AT ($\ell_\infty$ norm $\epsilon = 4/255$)}
		\label{fig:ImageNet_targeted_adv_train_Linf}
	\end{minipage}
	\begin{minipage}[b]{.3\textwidth}
		\includegraphics[width=\linewidth]{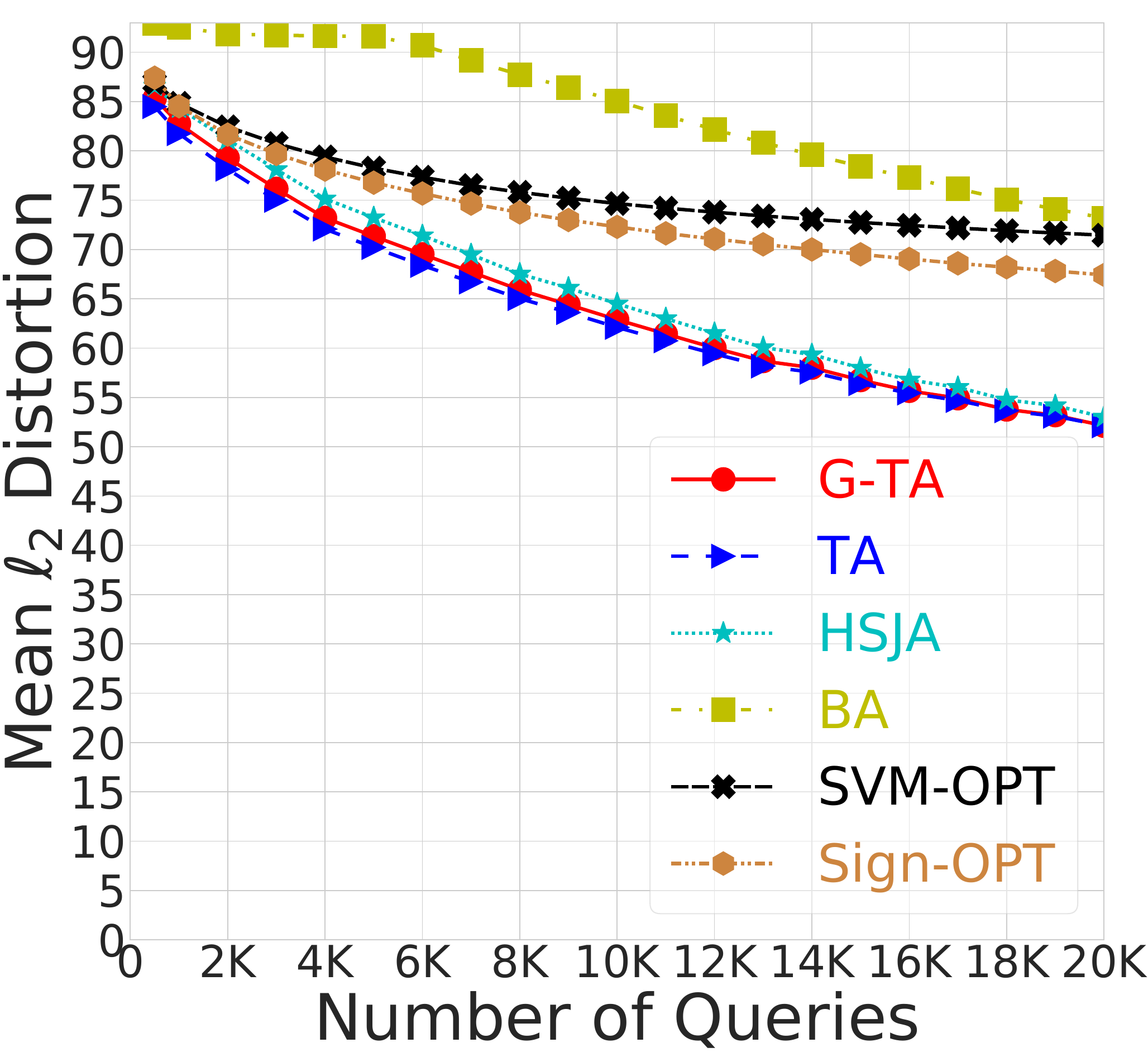}
		\subcaption{AT ($\ell_\infty$ norm $\epsilon = 8/255$)}
		\label{fig:ImageNet_targeted_adv_train_Linf_8}
	\end{minipage}
	\caption{Experimental results of $\ell_2$ norm attacks against adversarial trained ResNet-50 networks on the ImageNet dataset, where the first row (Figs. \ref{fig:ImageNet_untargeted_adv_train_L2}, \ref{fig:ImageNet_untargeted_adv_train_Linf}, \ref{fig:ImageNet_untargeted_adv_train_Linf_8}) shows the results of untargeted attacks, and the second row (Figs. \ref{fig:ImageNet_targeted_adv_train_L2}, \ref{fig:ImageNet_targeted_adv_train_Linf}, \ref{fig:ImageNet_targeted_adv_train_Linf_8}) shows the results of targeted attacks.}
	\label{fig:attack_adv_train_ImageNet}
\end{figure}
\begin{figure}[!t]
	\captionsetup[sub]{font=small}
	\setlength{\abovecaptionskip}{0pt}
	\setlength{\belowcaptionskip}{0pt}
	\centering 
	\begin{minipage}[b]{.3\textwidth}
		\includegraphics[width=\linewidth]{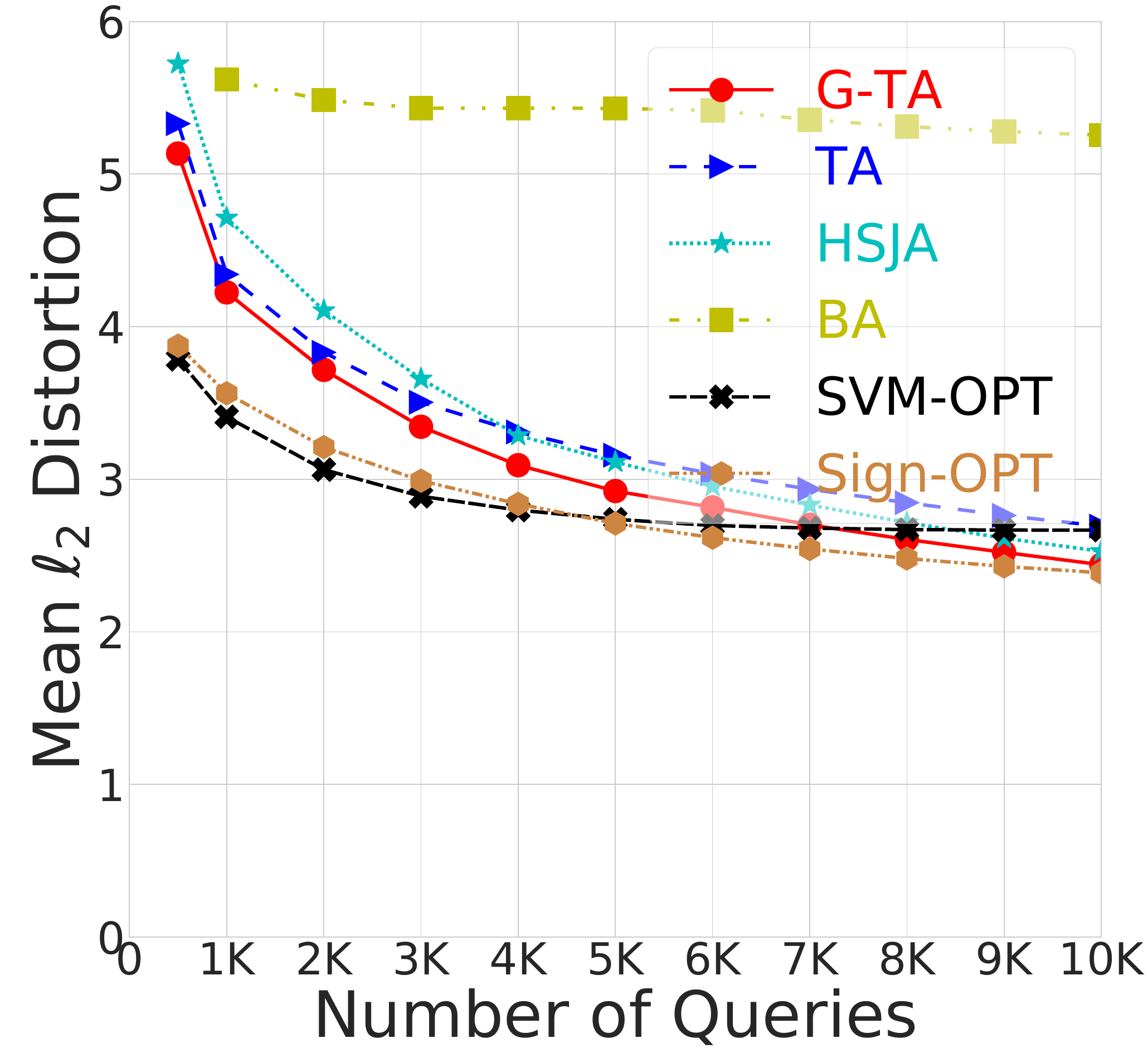}
		\subcaption{ComDefend}
		\label{fig:CIFAR10_untargeted_com_defend}
	\end{minipage}
	\begin{minipage}[b]{.3\textwidth}
		\includegraphics[width=\linewidth]{figures/defense_models/CIFAR-10_resnet-50_adv_train_l2_untargeted_attack.pdf}
		\subcaption{AT (PGD-10 $\epsilon = 8/255$)}
		\label{fig:CIFAR10_untargeted_adv_train}
	\end{minipage}
	\begin{minipage}[b]{.3\textwidth}
		\includegraphics[width=\linewidth]{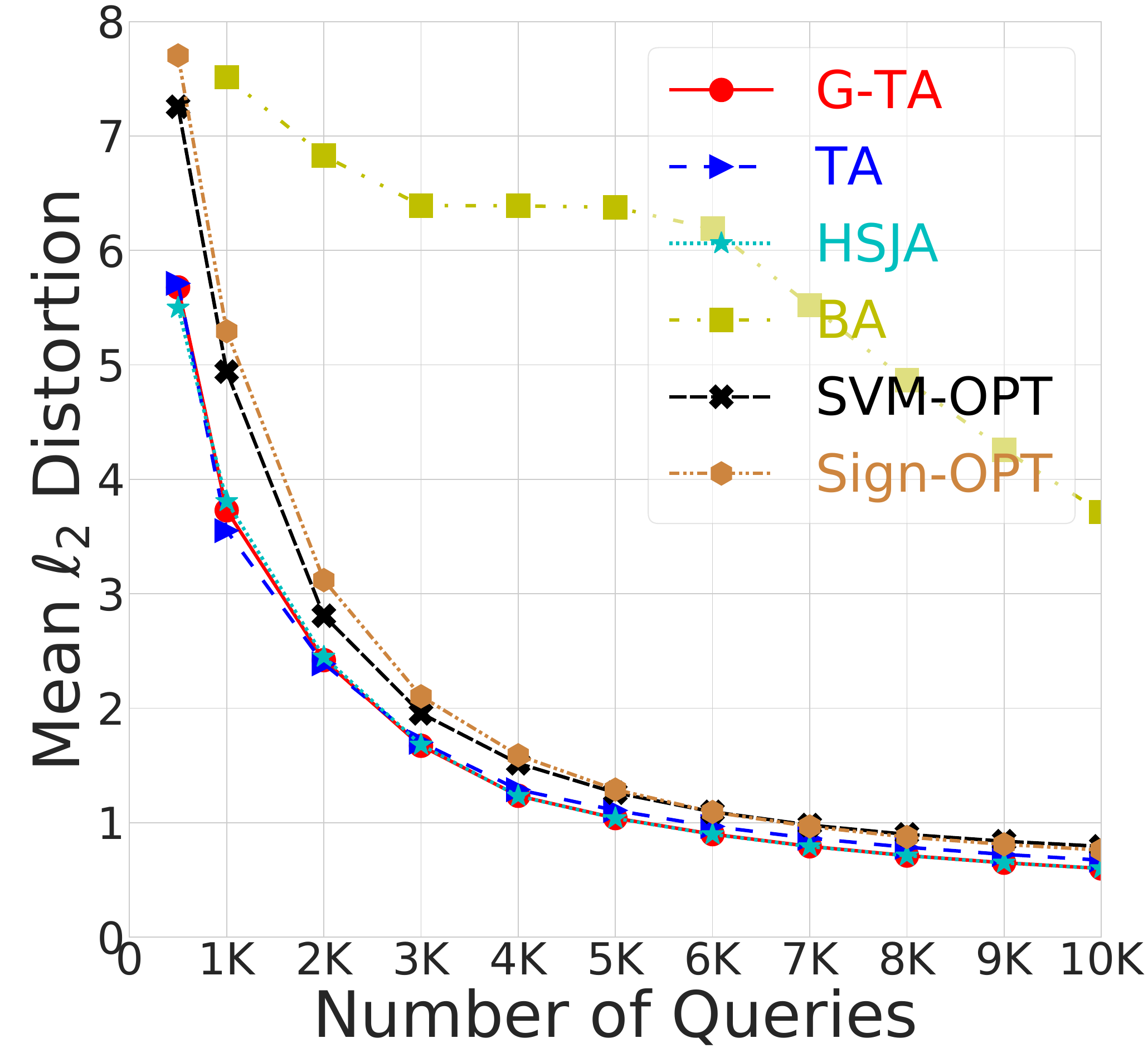}
		\subcaption{Feature Scatter}
		\label{fig:CIFAR10_untargeted_feature_scatter}
	\end{minipage}
	\begin{minipage}[b]{.3\textwidth}
		\includegraphics[width=\linewidth]{figures/defense_models/CIFAR-10_resnet-50_TRADES_l2_untargeted_attack.pdf}
		\subcaption{TRADES ($\epsilon = 8/255$)}
		\label{fig:CIFAR10_untargeted_TRADES}
	\end{minipage}
	\begin{minipage}[b]{.3\textwidth}
		\includegraphics[width=\linewidth]{figures/defense_models/CIFAR-10_resnet-50_jpeg_l2_untargeted_attack.pdf}
		\subcaption{JPEG}
		\label{fig:CIFAR10_untargeted_jpeg}
	\end{minipage}
	\begin{minipage}[b]{.3\textwidth}
		\includegraphics[width=\linewidth]{figures/defense_models/CIFAR-10_resnet-50_feature_distillation_l2_untargeted_attack.pdf}
		\subcaption{Feature Distillation}
		\label{fig:CIFAR10_untargeted_FeatureDistillation}
	\end{minipage}
	\caption{Experimental results of the $\ell_2$ norm untargeted attacks against defense models on the CIFAR-10 dataset, where all defense models adopt the backbone of ResNet-50 network.}
	\label{fig:untargeted_attack_on_defense_models_CIFAR10}
\end{figure}
\begin{figure}[!h]
	\captionsetup[sub]{font=small}
	\setlength{\abovecaptionskip}{0pt}
	\setlength{\belowcaptionskip}{0pt}
	\centering 
	\begin{minipage}[b]{.3\textwidth}
		\includegraphics[width=\linewidth]{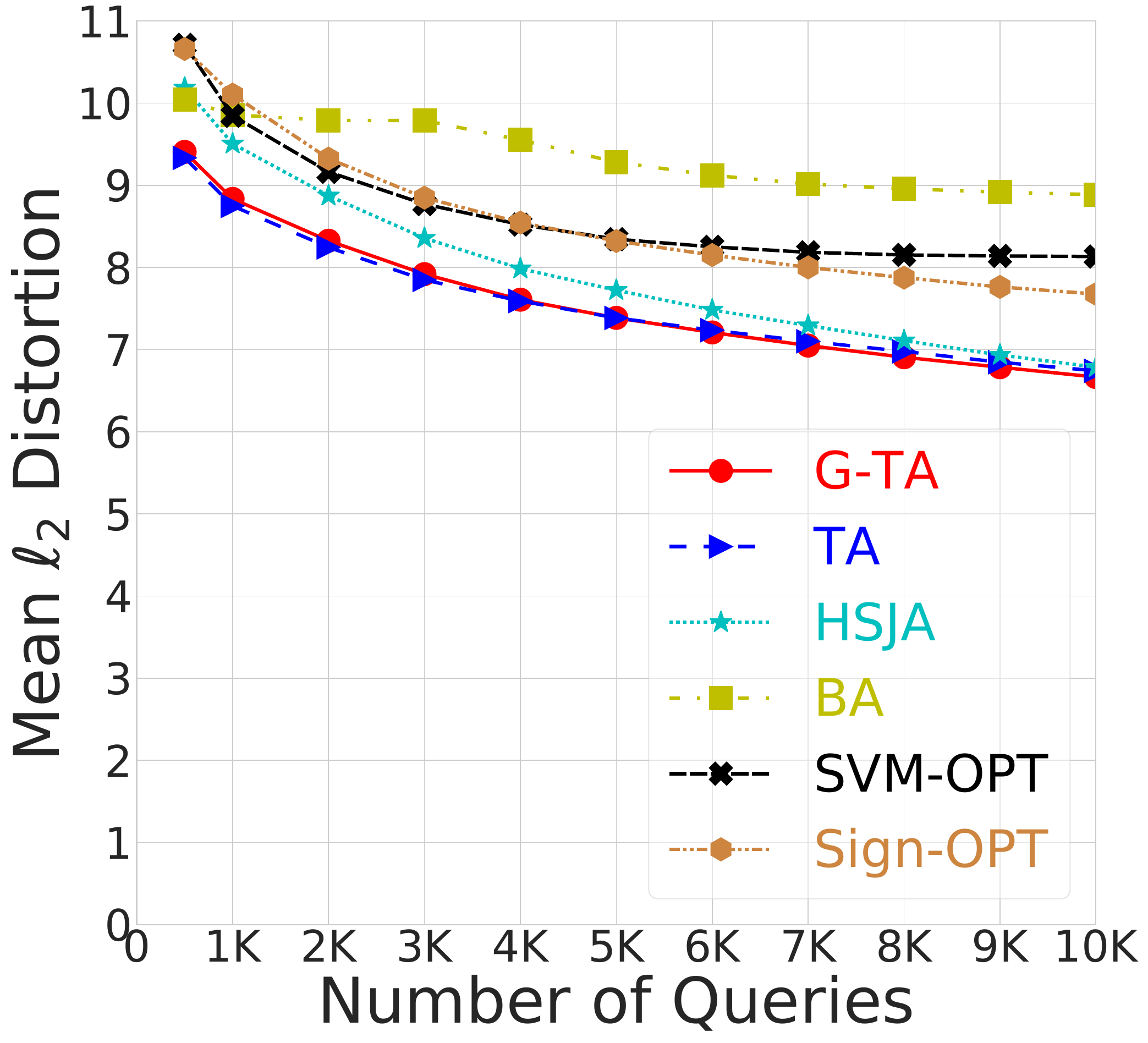}
		\subcaption{ComDefend}
	\end{minipage}
	\begin{minipage}[b]{.3\textwidth}
		\includegraphics[width=\linewidth]{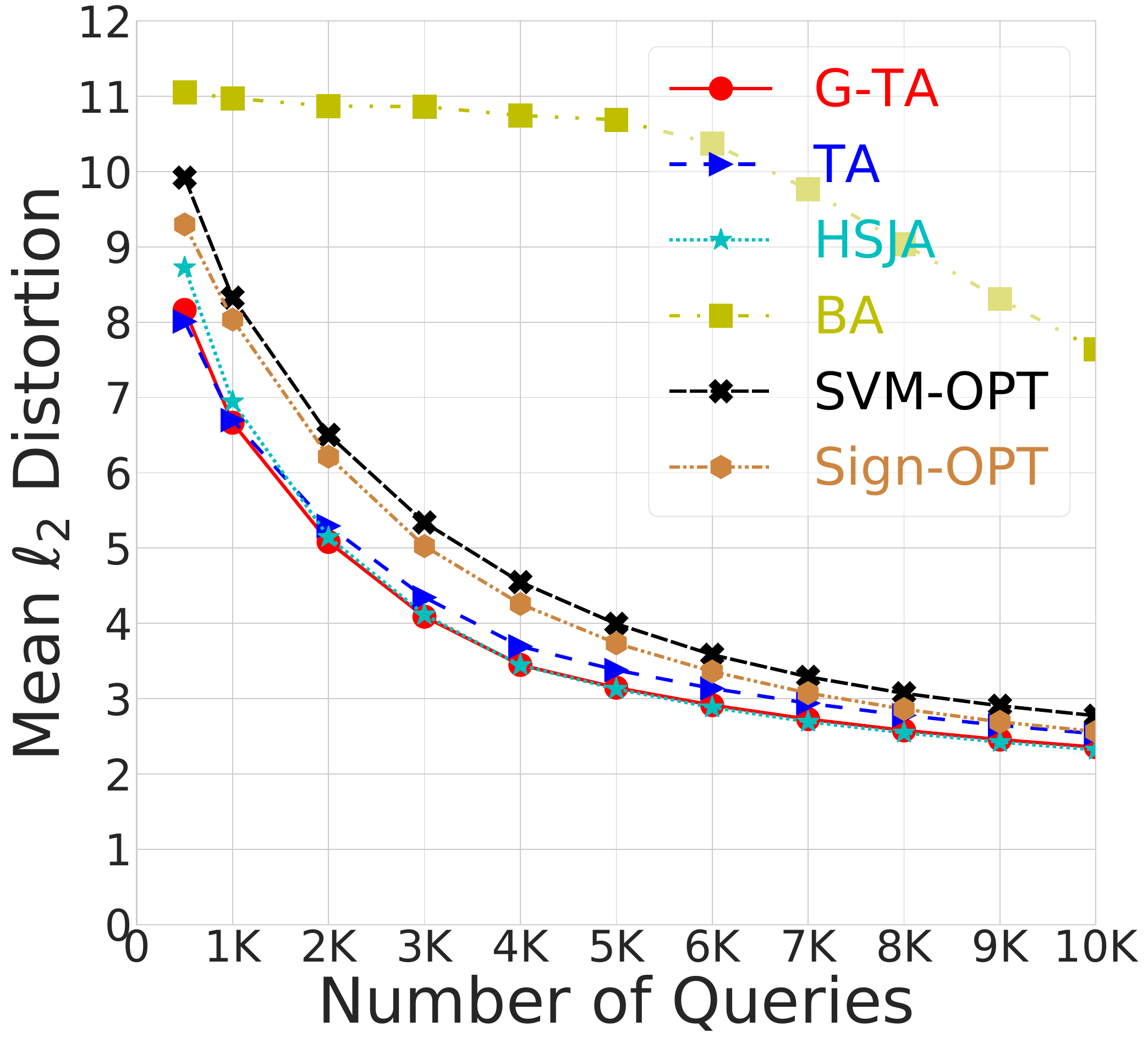}
		\subcaption{AT (PGD-10 $\epsilon = 8/255$)}
	\end{minipage}
	\begin{minipage}[b]{.3\textwidth}
		\includegraphics[width=\linewidth]{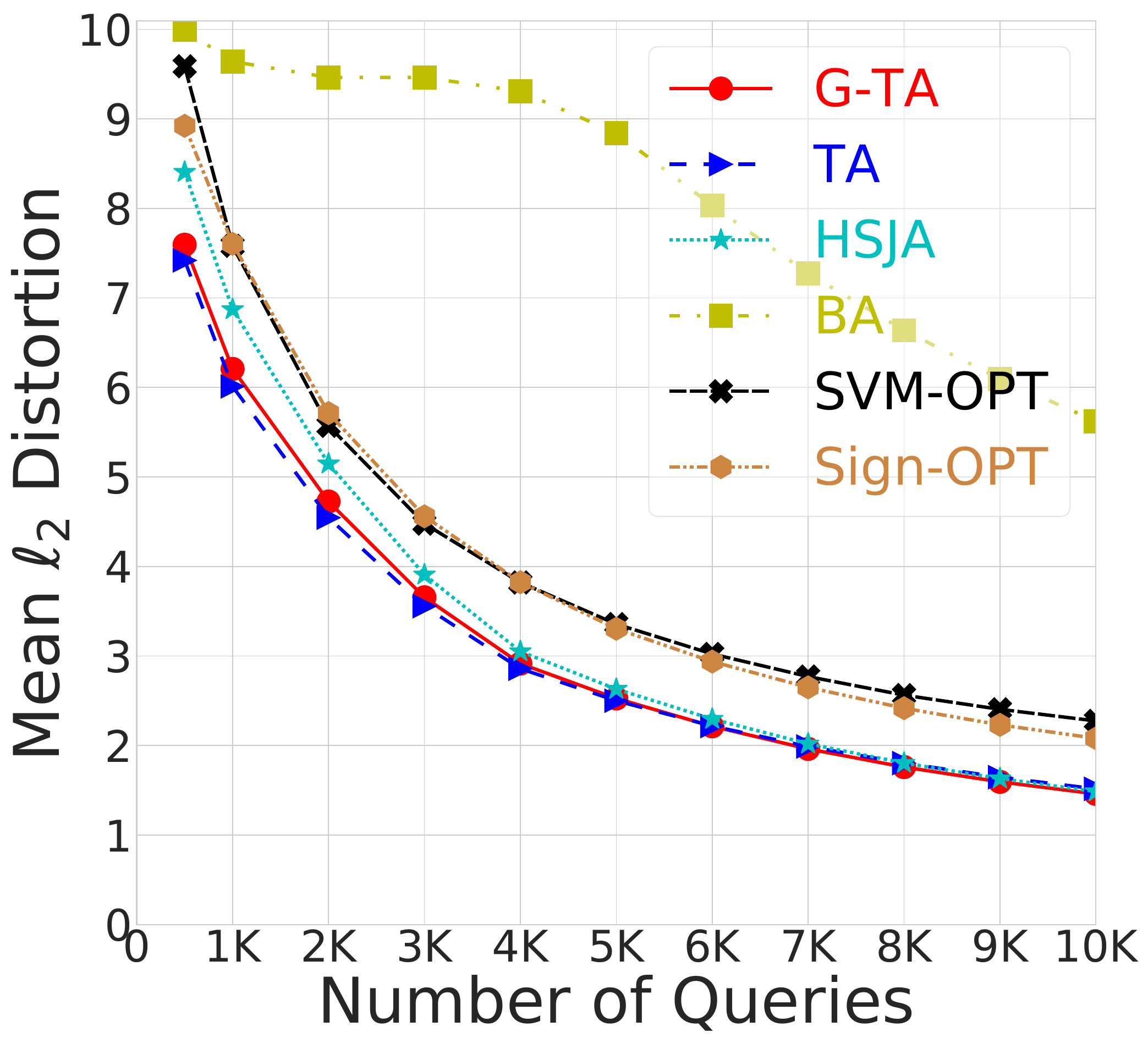}
		\subcaption{Feature Scatter}
		\label{fig:CIFAR10_targeted_feature_scatter}
	\end{minipage}
	\begin{minipage}[b]{.3\textwidth}
		\includegraphics[width=\linewidth]{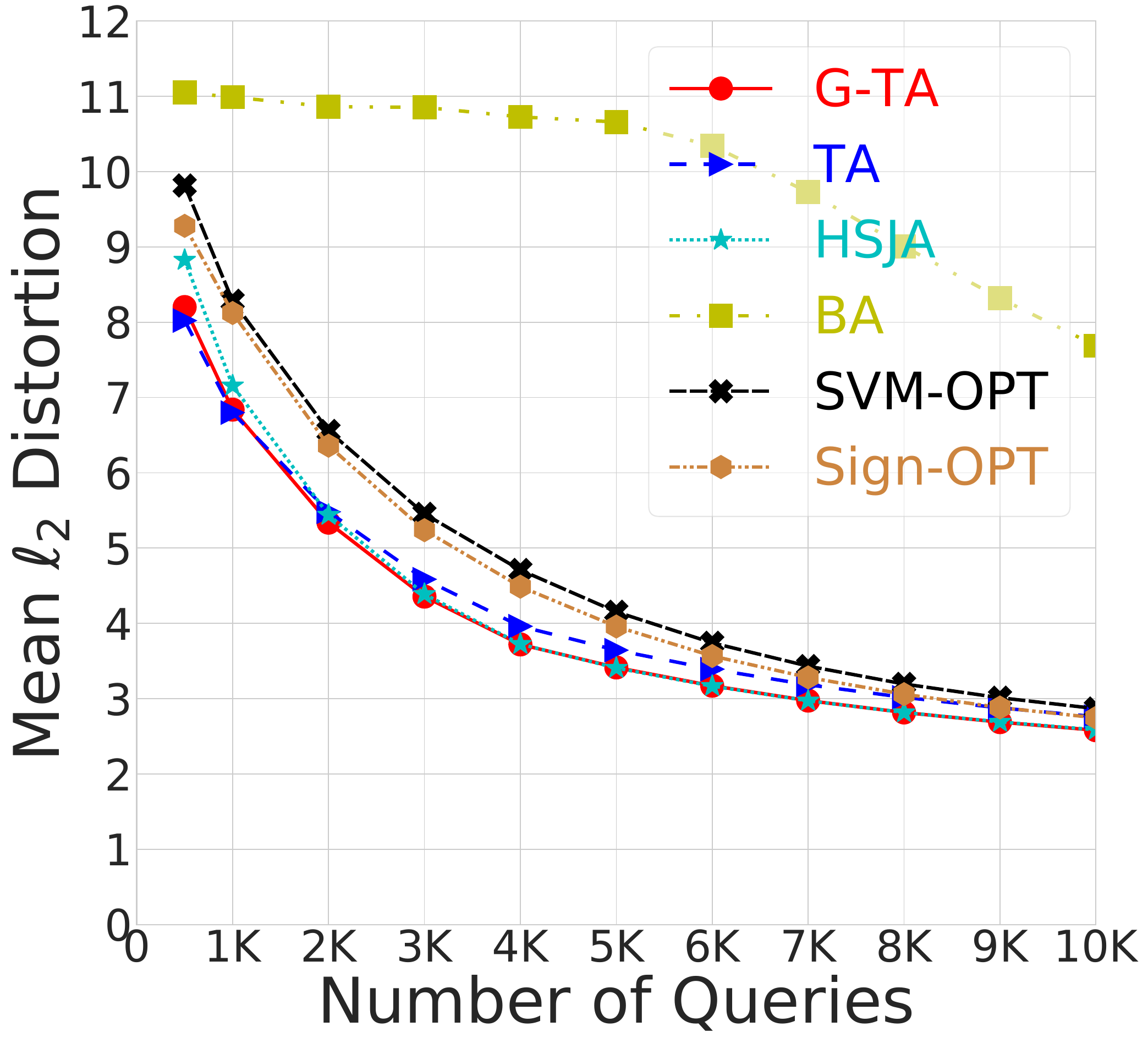}
		\subcaption{TRADES ($\epsilon = 8/255$)}
		\label{fig:CIFAR10_targeted_TRADES}
	\end{minipage}
	\begin{minipage}[b]{.3\textwidth}
		\includegraphics[width=\linewidth]{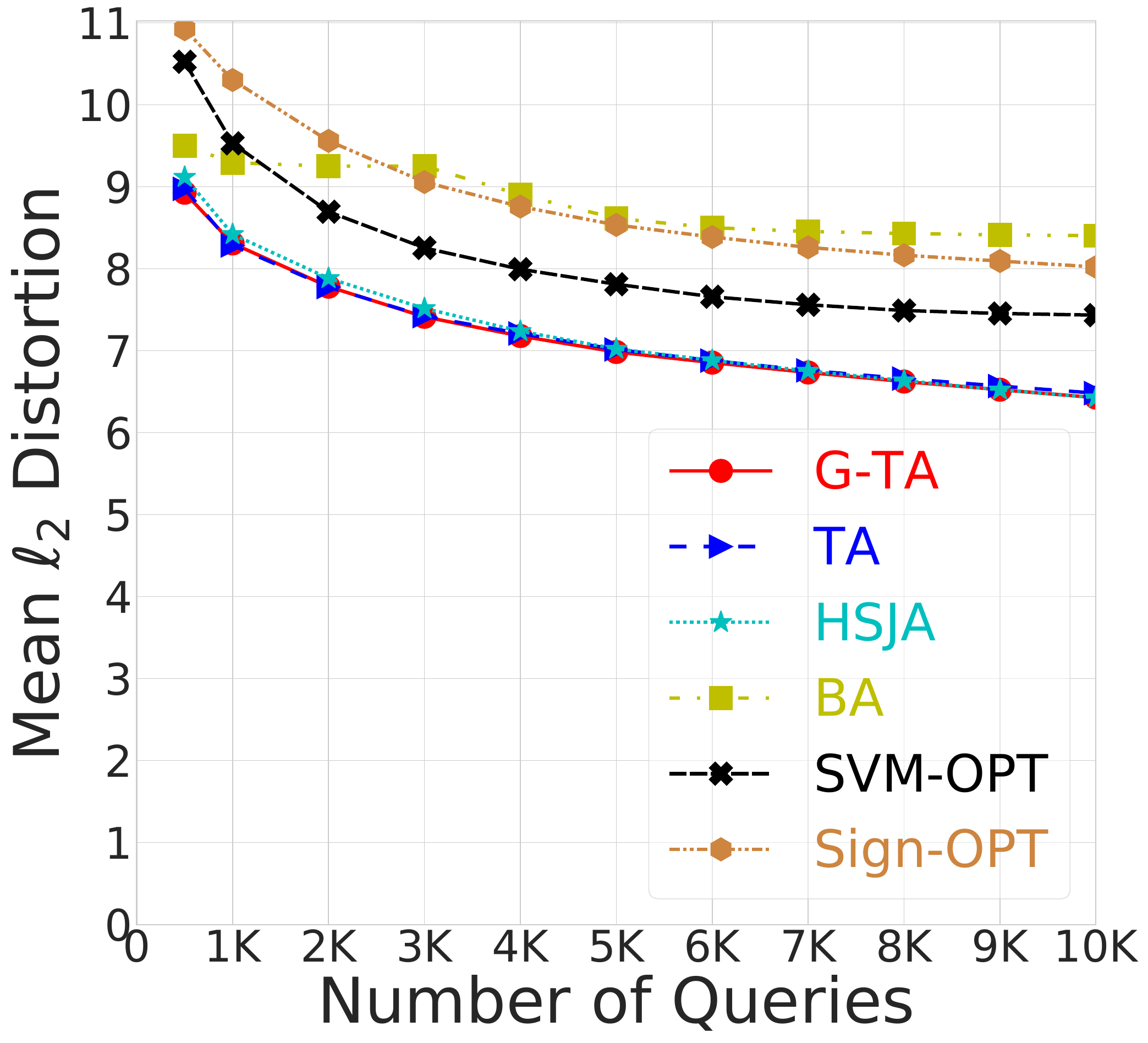}
		\subcaption{JPEG}
		\label{fig:CIFAR10_targeted_ComDefend}
	\end{minipage}
	\begin{minipage}[b]{.3\textwidth}
		\includegraphics[width=\linewidth]{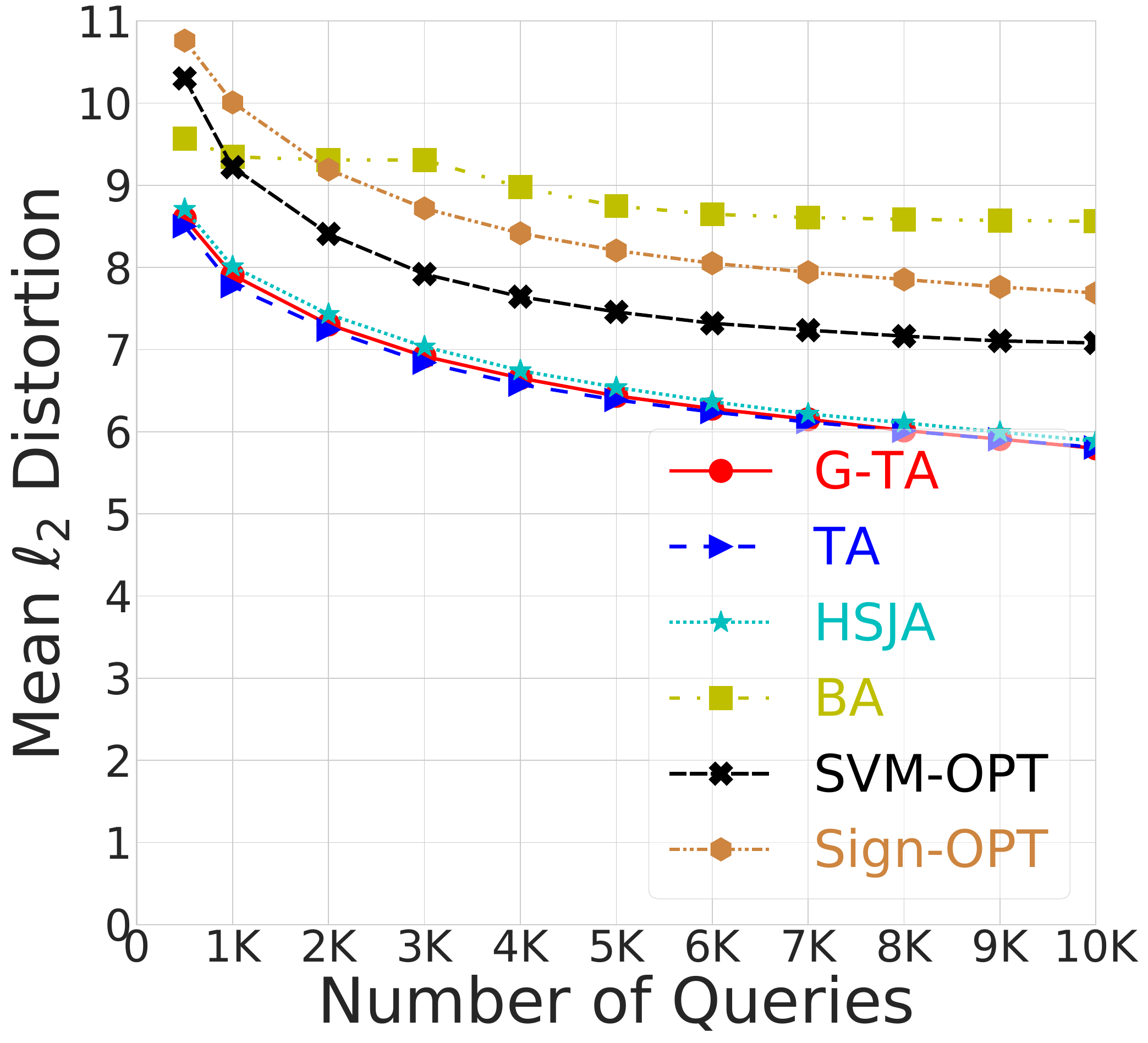}
		\subcaption{Feature Distillation}
		\label{fig:CIFAR10_targeted_FeatureDistillation}
	\end{minipage}
	\caption{Experimental results of the $\ell_2$ norm targeted attacks against defense models on the CIFAR-10 dataset, where all defense models adopt the backbone of ResNet-50 network.}
	\label{fig:targeted_attack_on_defense_models_CIFAR10}
\end{figure}
Figs. \ref{fig:untargeted_attack_on_defense_models_CIFAR10} and \ref{fig:targeted_attack_on_defense_models_CIFAR10} show the experimental results of untargeted and targeted attacks on the CIFAR-10 dataset, respectively. 
In the results of untargeted attacks (Fig. \ref{fig:untargeted_attack_on_defense_models_CIFAR10}), G-TA outperforms HSJA and TA in the attacks of ComDefend, JPEG and Feature Distillation. When the target models are AT, Feature Scatter and TRADES, the performance of G-TA is similar to that of the baseline attack method HSJA.

In addition, in the experimental results of targeted attacks (Fig. \ref{fig:targeted_attack_on_defense_models_CIFAR10}), the performance of G-TA is similar to that of TA when attacking different defense models.

\subsection{Distributions of Distortions across Different Adversarial Examples}
So far, all the experimental results only show the average $\ell_2$ distortion of \nn{1000} adversarial examples. 
To check the distortion of each adversarial example in more detail, we extract the $\ell_2$ distortions of 20 samples from HSJA, TA and G-TA.
These samples are selected from \nn{1000} images in the following way: from the 1st image to the \nn{1000}th image, we select one image for every 50 images.
Fig. \ref{fig:distribution_distortion} shows the distributions of $\ell_2$ distortions across 20 adversarial examples on the ImageNet dataset, where the 1st image's ``image number index'' is 0. The results indicate that the $\ell_2$ distortions obtained by TA and G-TA are uniformly better than that of the baseline method HSJA.
Thus, our approach can obtain better $\ell_2$ distortions on different adversarial examples, not just on specific samples.

\begin{figure}[h]
	\captionsetup[sub]{font=small}
	\setlength{\abovecaptionskip}{0pt}
	\setlength{\belowcaptionskip}{0pt}
	\centering 
	\begin{minipage}[b]{.3\textwidth}
		\includegraphics[width=\linewidth]{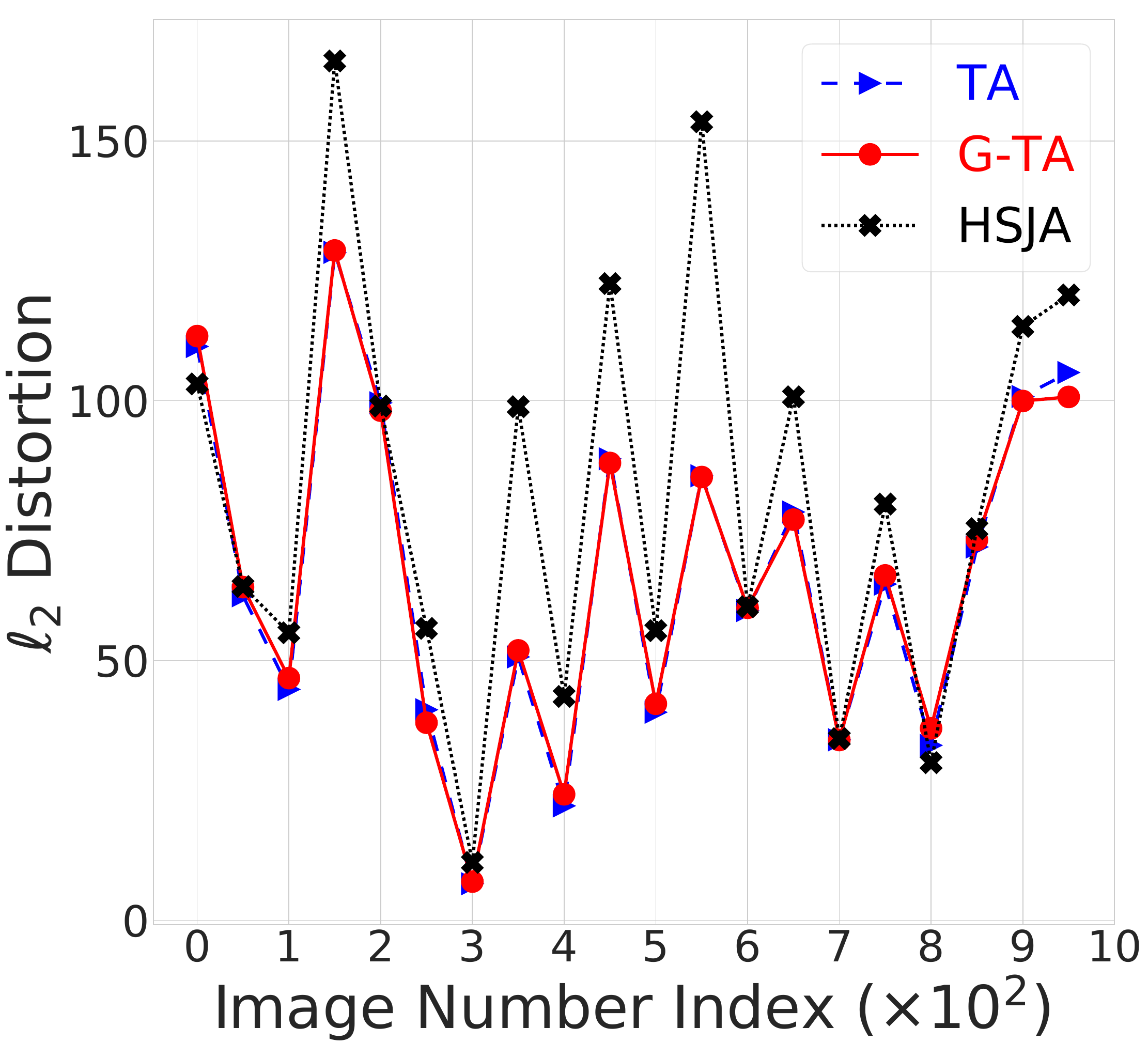}
		\subcaption{Inception-v3 (query: 1K)}
		\label{fig:inceptionv3_1K}
	\end{minipage}
	\begin{minipage}[b]{.3\textwidth}
		\includegraphics[width=\linewidth]{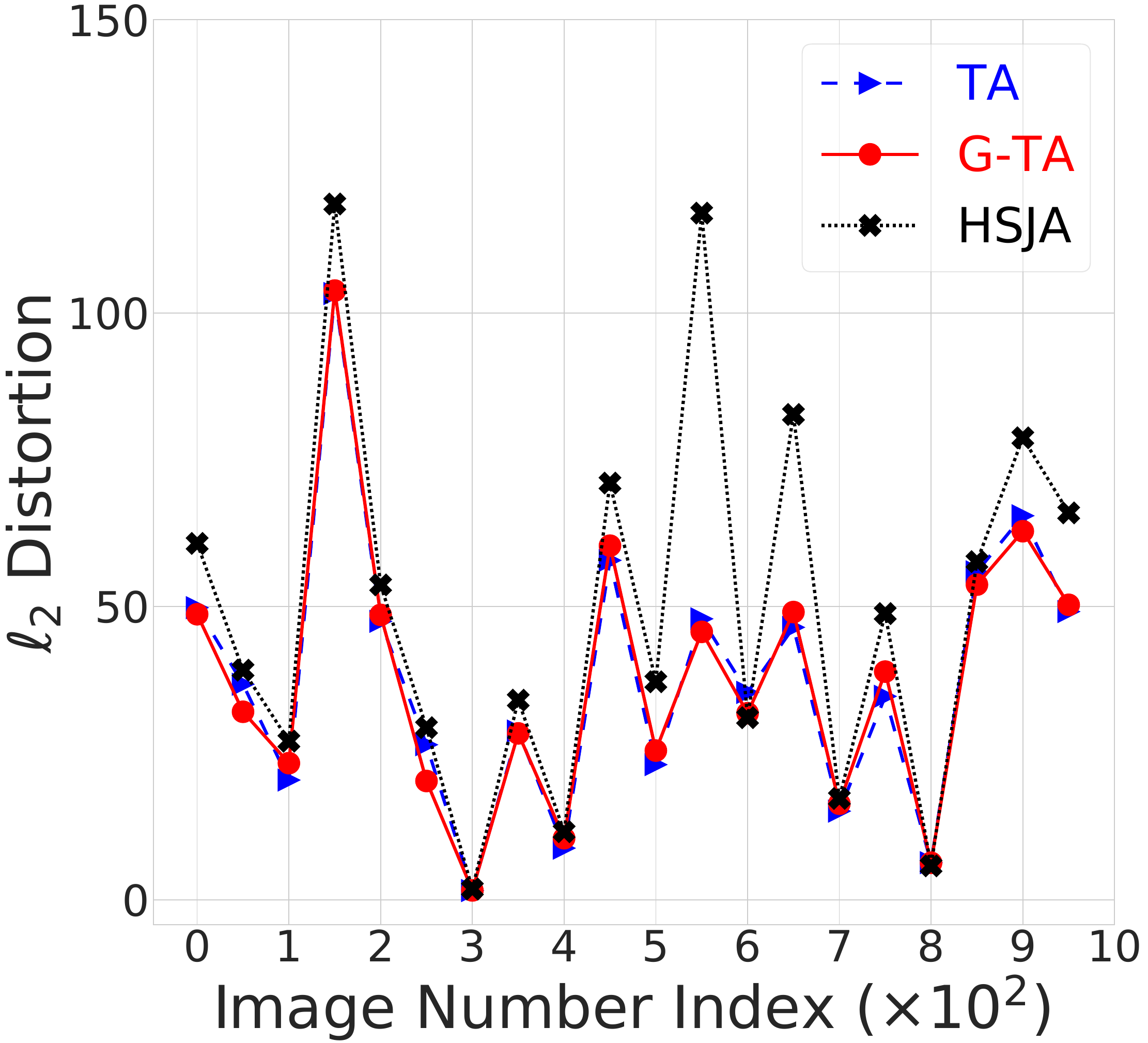}
		\subcaption{Inception-v3 (query: 5K)}
		\label{fig:inceptionv3_5K}
	\end{minipage}
	\begin{minipage}[b]{.3\textwidth}
		\includegraphics[width=\linewidth]{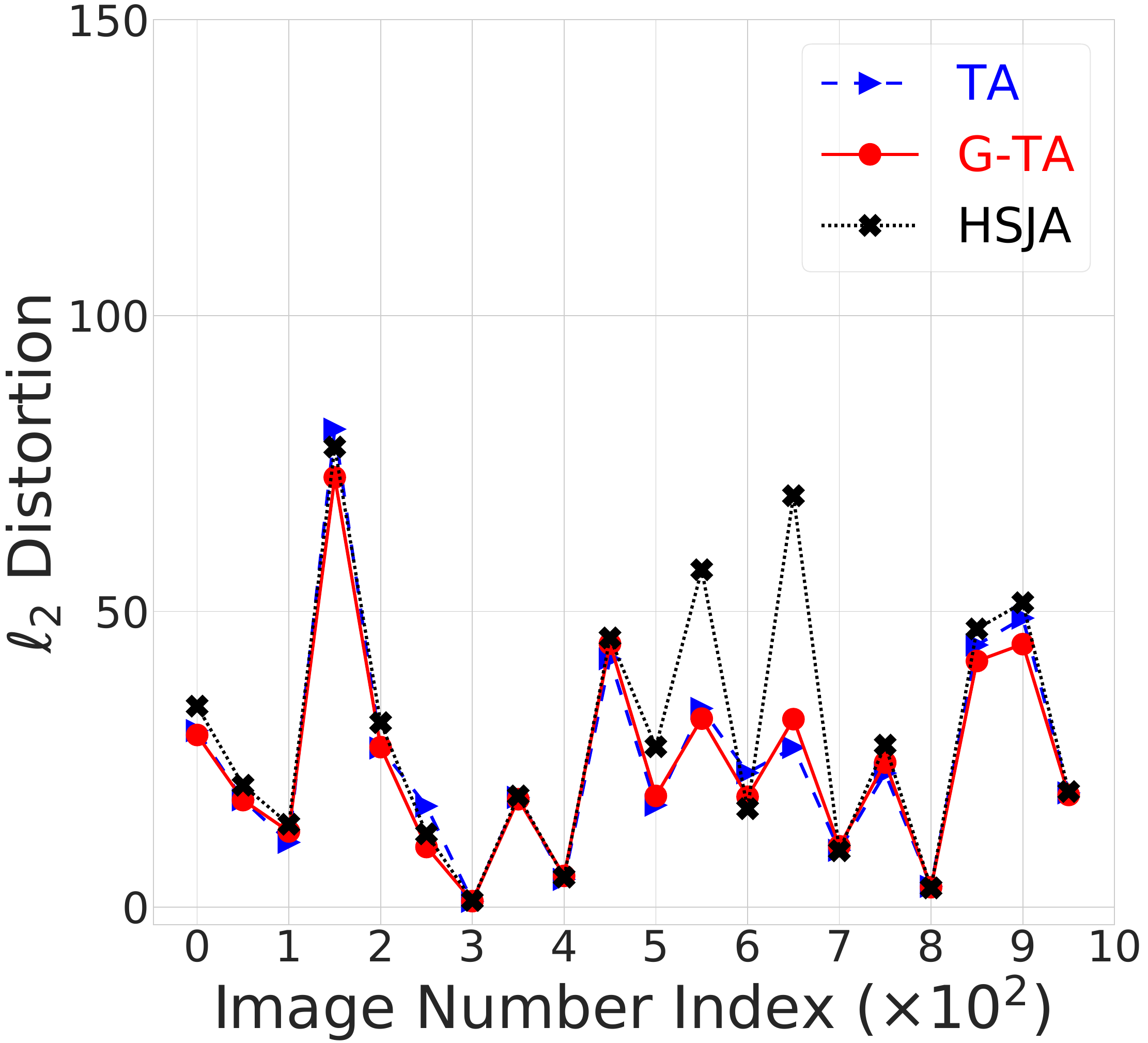}
		\subcaption{Inception-v3 (query: 10K)}
		\label{fig:inceptionv3_10K}
	\end{minipage}
	\begin{minipage}[b]{.3\textwidth}
		\includegraphics[width=\linewidth]{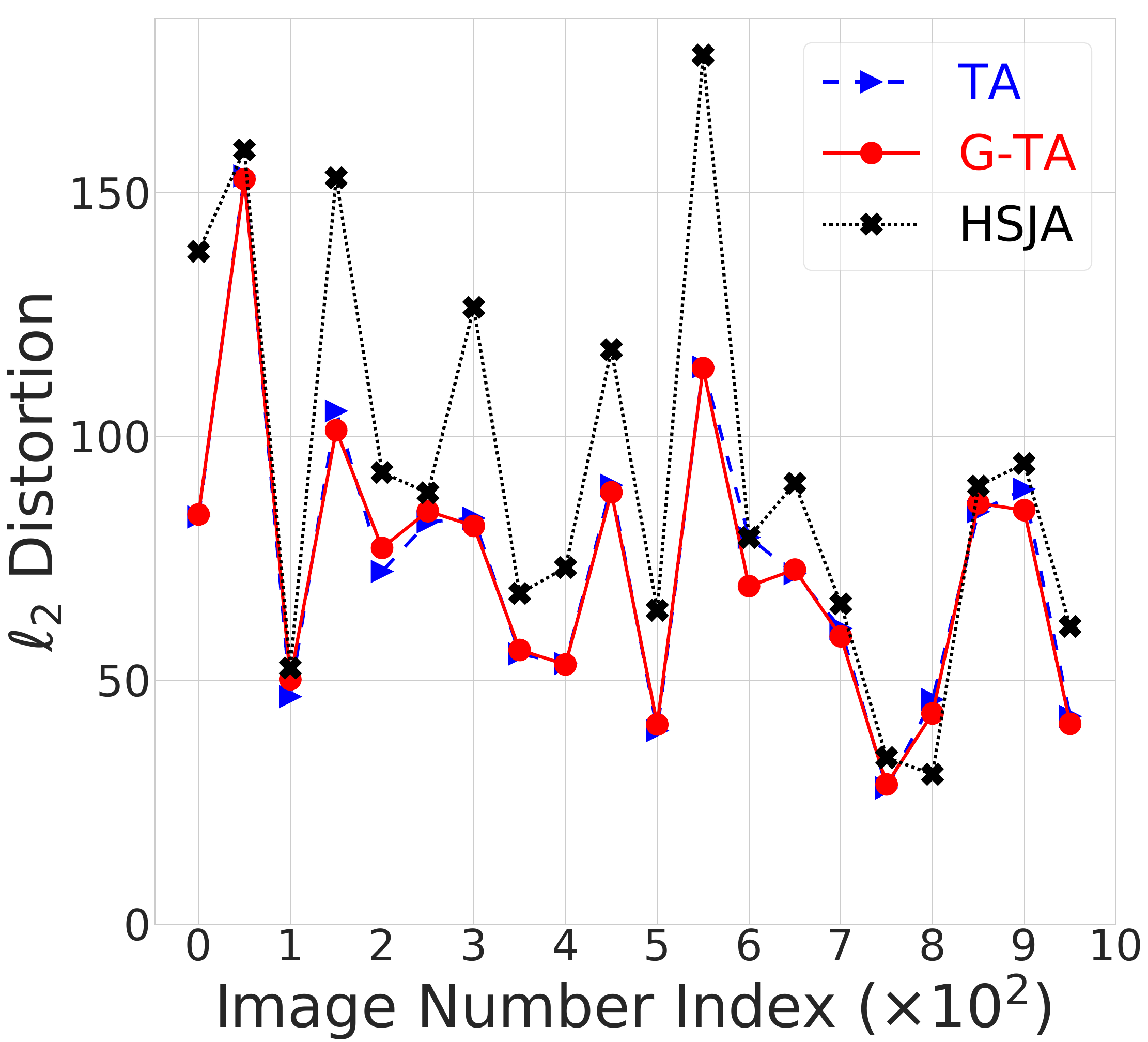}
		\subcaption{Inception-v4 (query: 1K)}
		\label{fig:inceptionv4_1K}
	\end{minipage}
	\begin{minipage}[b]{.3\textwidth}
		\includegraphics[width=\linewidth]{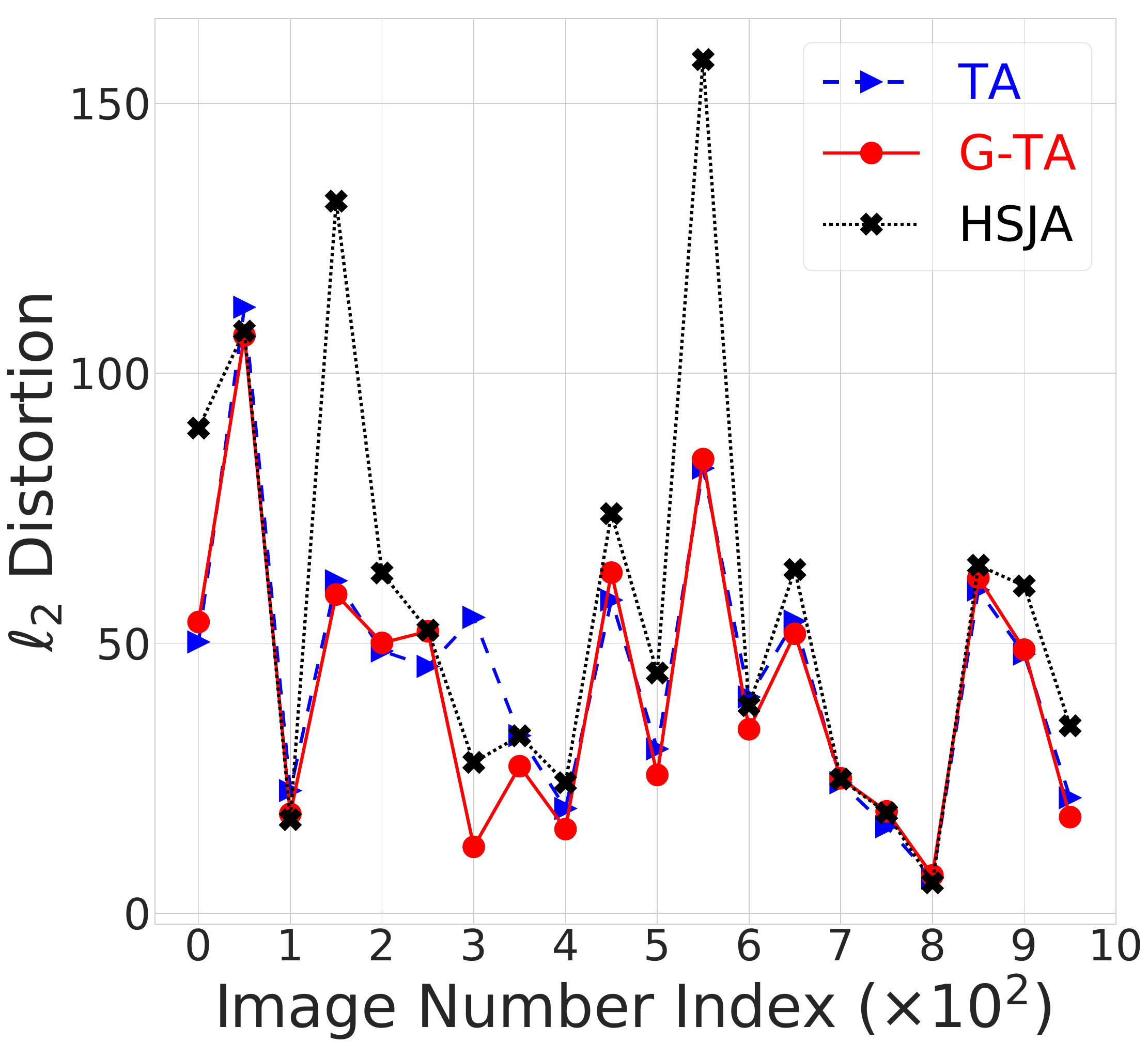}
		\subcaption{Inception-v4 (query: 5K)}
		\label{fig:inceptionv4_5K}
	\end{minipage}
	\begin{minipage}[b]{.3\textwidth}
		\includegraphics[width=\linewidth]{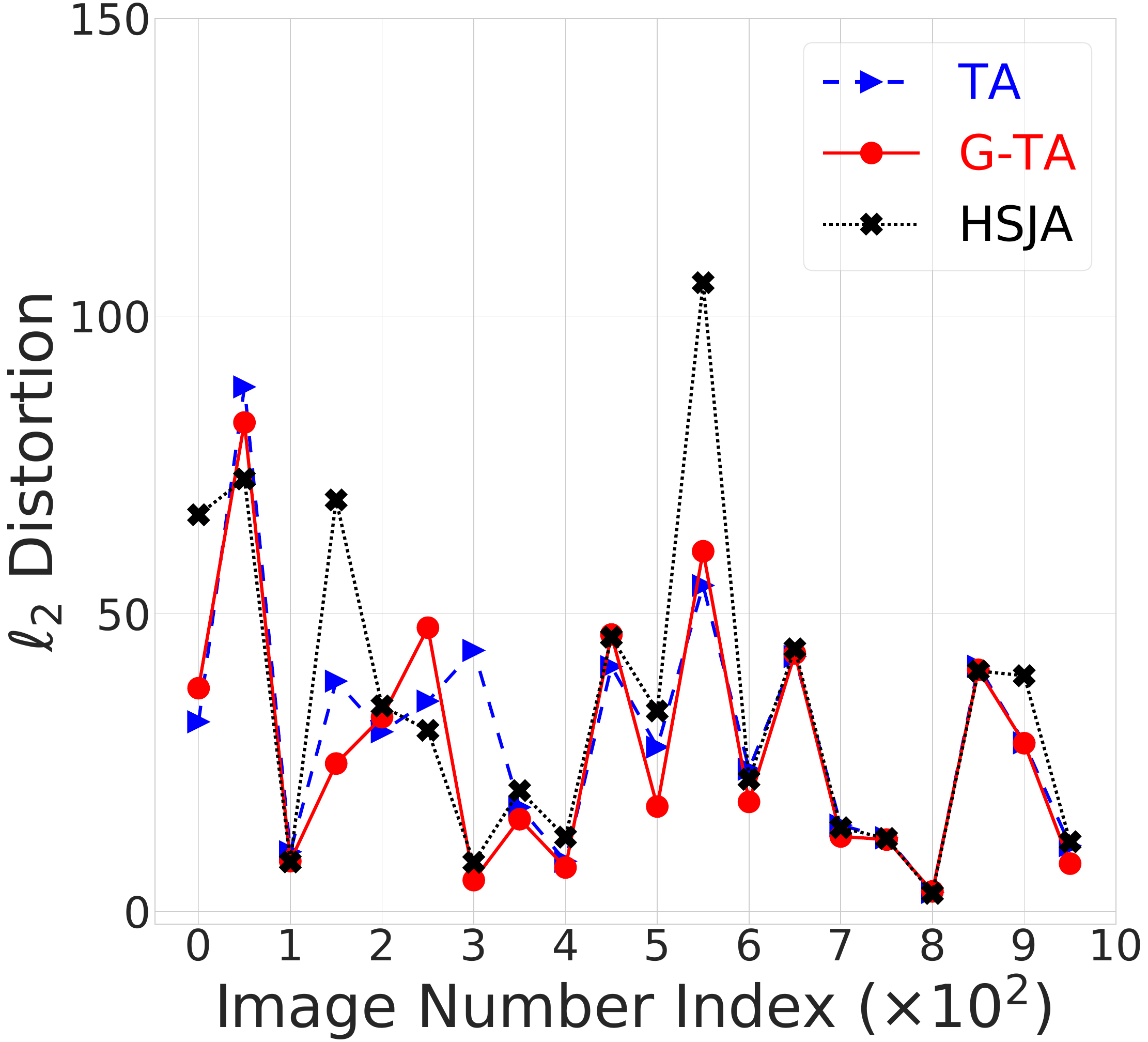}
		\subcaption{Inception-v4 (query: 10K)}
		\label{fig:inceptionv4_10K}
	\end{minipage}
	
	\begin{minipage}[b]{.3\textwidth}
		\includegraphics[width=\linewidth]{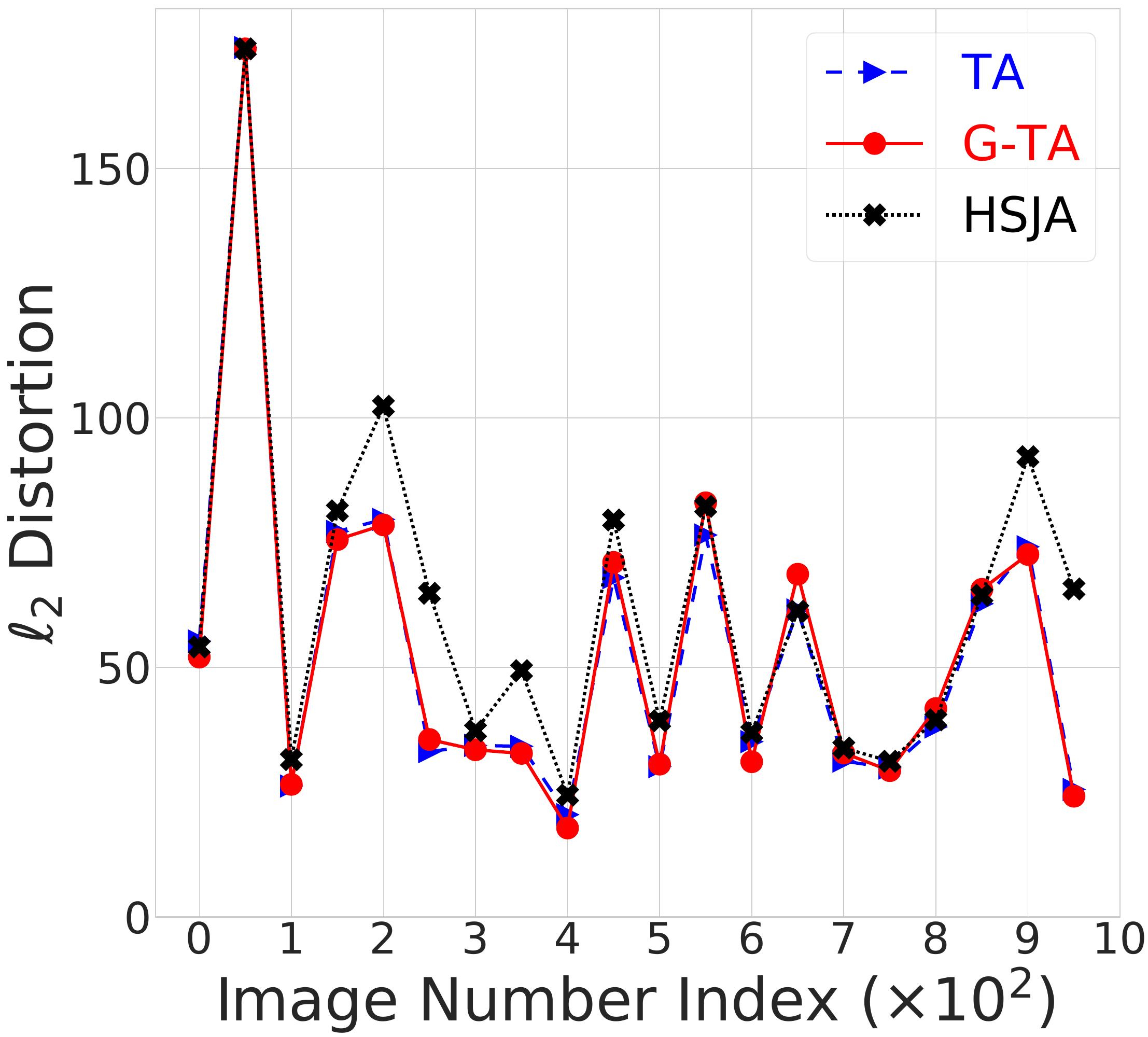}
		\subcaption{SENet-154 (query: 1K)}
		\label{fig:senet154_1K}
	\end{minipage}
	\begin{minipage}[b]{.3\textwidth}
		\includegraphics[width=\linewidth]{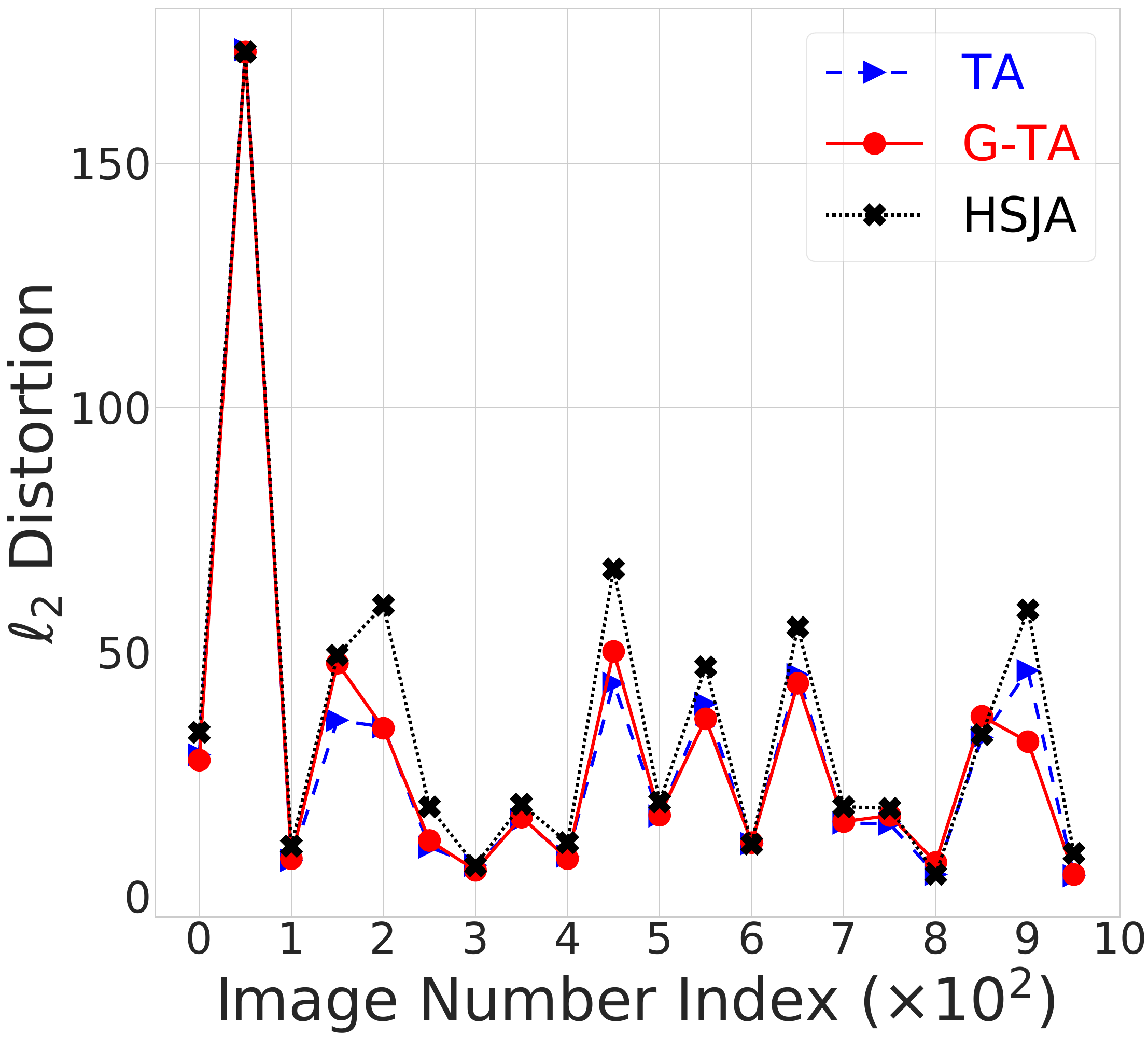}
		\subcaption{SENet-154 (query: 5K)}
		\label{fig:senet154_5K}
	\end{minipage}
	\begin{minipage}[b]{.3\textwidth}
		\includegraphics[width=\linewidth]{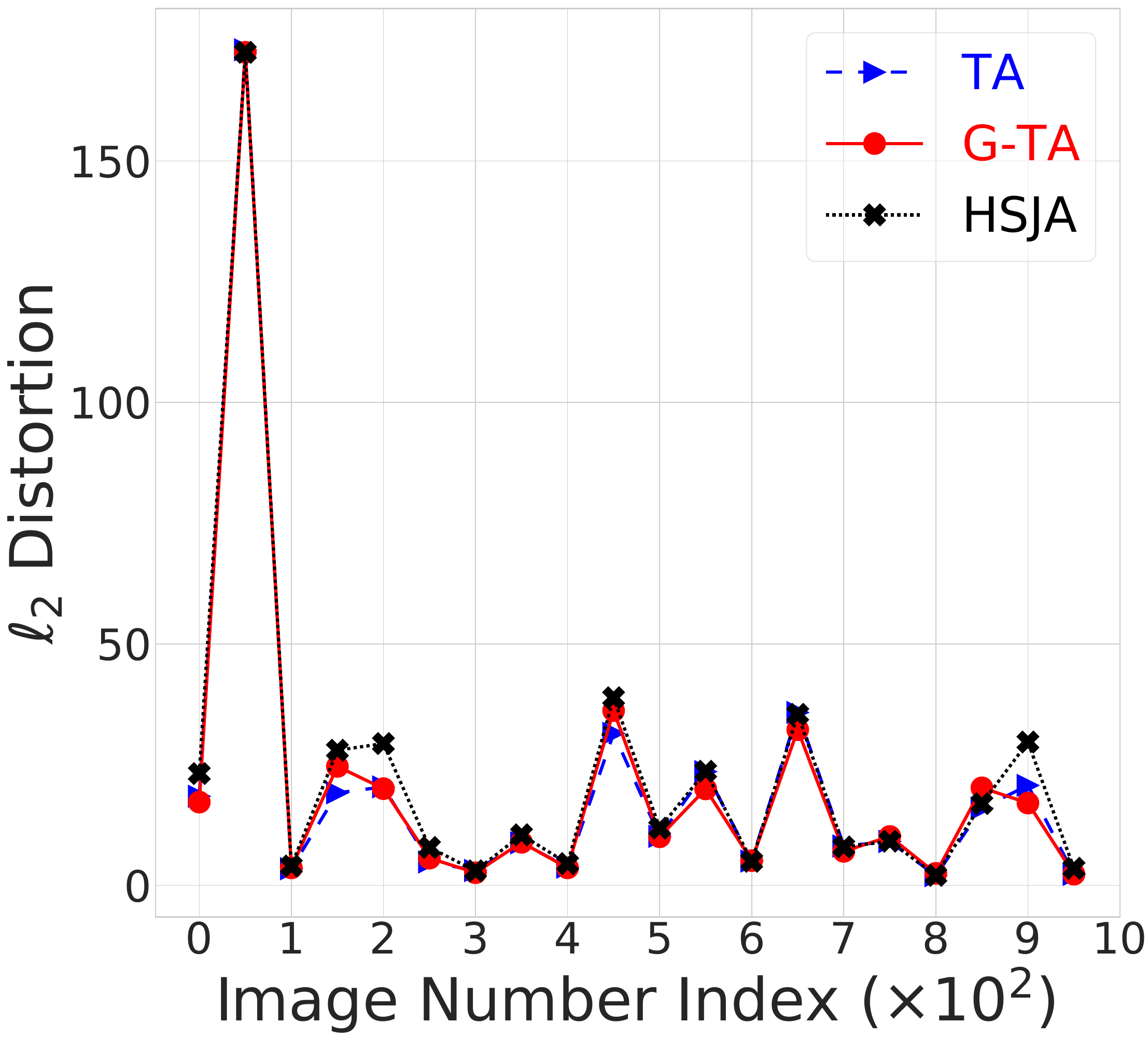}
		\subcaption{SENet-154 (query: 10K)}
		\label{fig:senet154_10K}
	\end{minipage}
	\begin{minipage}[b]{.3\textwidth}
		\includegraphics[width=\linewidth]{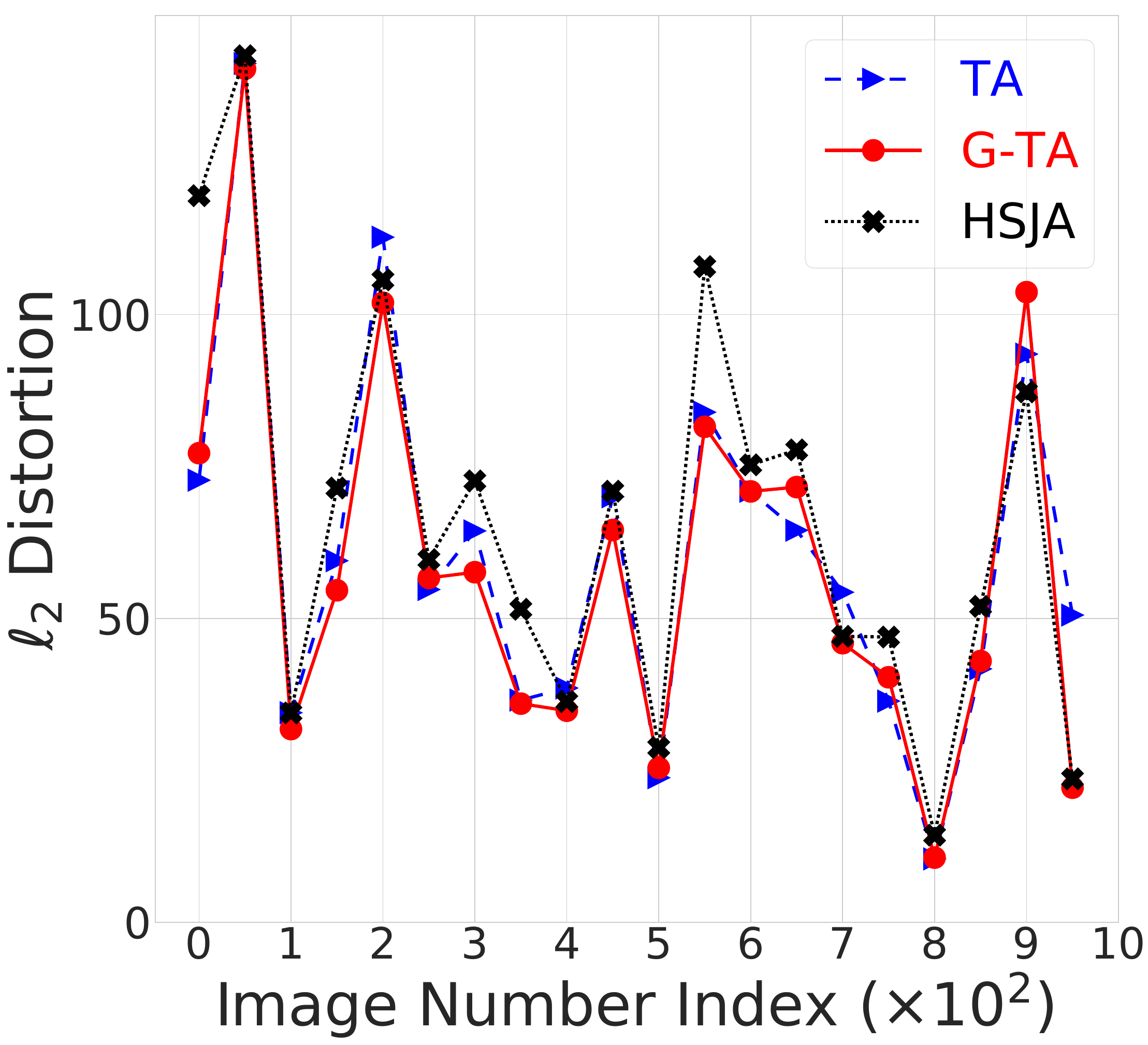}
		\subcaption{ResNet-101 (query: 1K)}
		\label{fig:resnet101_1K}
	\end{minipage}
	\begin{minipage}[b]{.3\textwidth}
		\includegraphics[width=\linewidth]{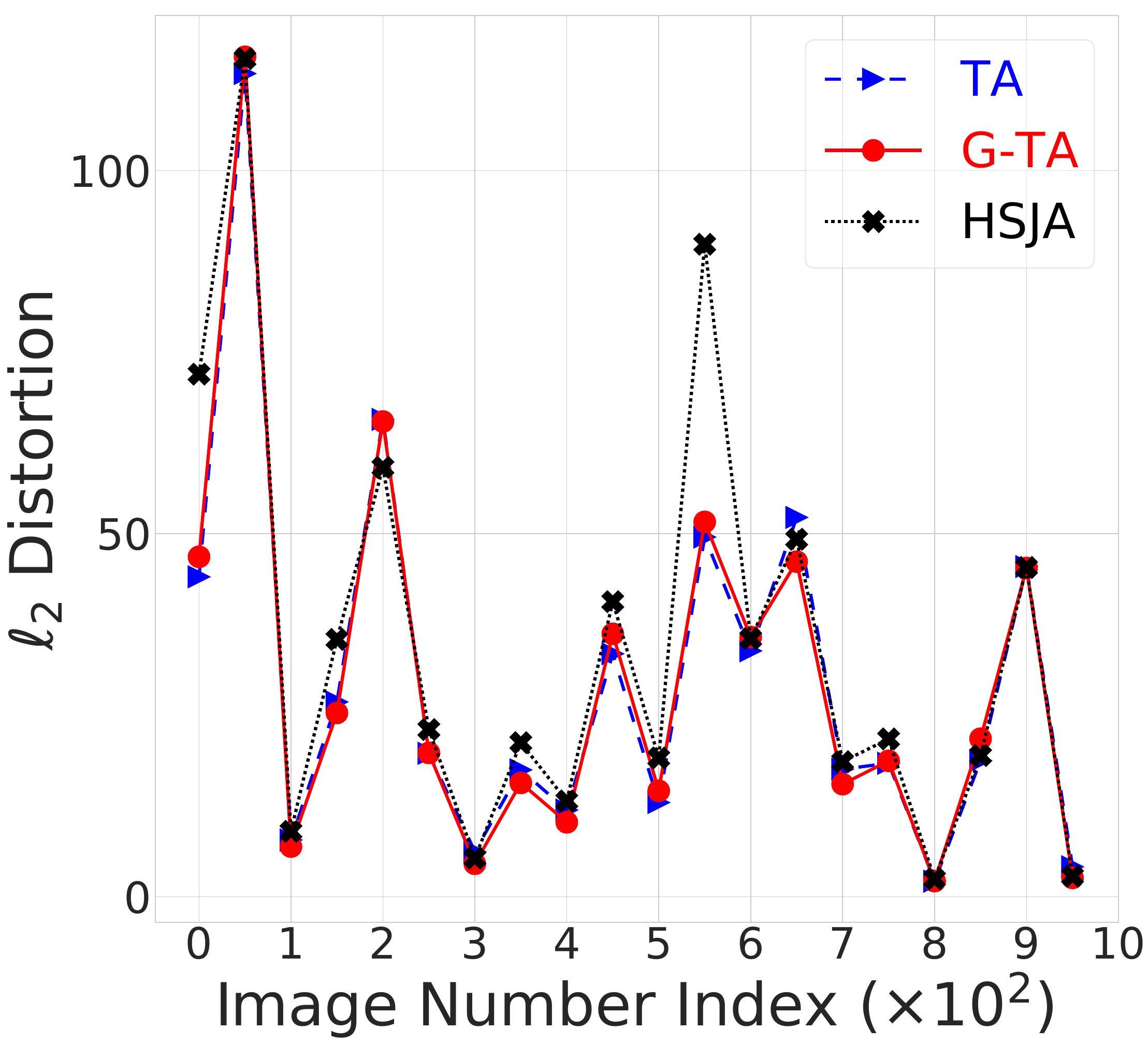}
		\subcaption{ResNet-101 (query: 5K)}
		\label{fig:resnet101_5K}
	\end{minipage}
	\begin{minipage}[b]{.3\textwidth}
		\includegraphics[width=\linewidth]{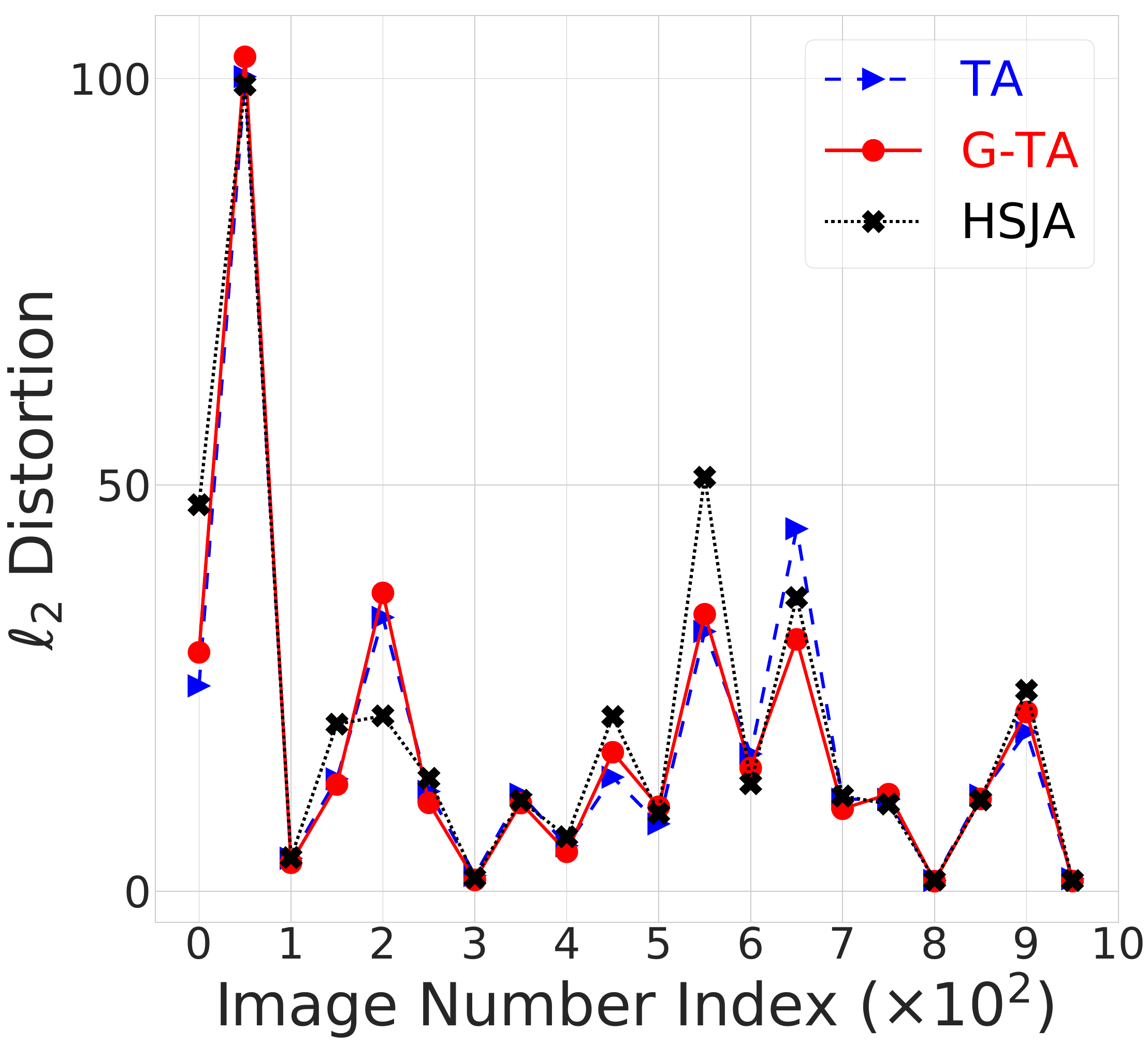}
		\subcaption{ResNet-101 (query: 10K)}
		\label{fig:resnet101_10K}
	\end{minipage}
	\caption{Comparisons of $\ell_2$ distortions across 20 adversarial examples in targeted attacks of the ImageNet dataset.}
	\label{fig:distribution_distortion}
\end{figure}

\subsection{Experimental Results of Median Distortions}
In this section, we report the median $\ell_2$ distortions of different query budgets on the CIFAR-10 and ImageNet datasets. Tables \ref{tab:ImageNet_normal_models_median_result} and \ref{tab:CIFAR-10_normal_models_median_result} show the experimental results. We can draw the following conclusions based on the results. 

\begin{table*}[h]
	\centering
	\small
	\tabcolsep=0.1cm
	\setlength{\belowcaptionskip}{0.5pt}%
	\caption{Median $\ell_2$ distortions of different query budgets on the ImageNet dataset. ``-'' denotes no adversarial example is found in this query budget.}
	\scalebox{0.8}{
		\begin{tabular}{cccccccc|cccccc}
			\toprule
			Target Model  & Method  & \multicolumn{6}{c|}{Targeted Attack} & \multicolumn{6}{c}{Untargeted Attack} \\
			& & @300&  @1K & @2K & @5K & @8K & @10K & @300&  @1K & @2K & @5K & @8K & @10K \\
			\midrule
			\multirow{6}{*}{Inception-v3}&  BA \cite{brendel2018decisionbased}& 105.513 & 101.877 & 101.056 & 97.481 & 81.269 & 73.524 & - & 109.507 & 103.637 & 96.340 & 79.027 & 58.924 \\
			& Sign-OPT \cite{cheng2019sign} & 96.905 & 83.215 & 66.601 & 43.350 & 31.036 & 25.380 & 115.140 & 73.319 & 38.327 & 12.761 & 8.277 & 6.808\\
			& SVM-OPT \cite{cheng2019sign} &  \B 93.649 & 77.838 & 63.631 & 42.897 & 31.838 & 26.322 & 114.879 & 59.343 & \B 30.627 & 12.085 & 8.352 & 7.025\\
			& HSJA \cite{chen2019hopskipjumpattack}& 106.341 & 92.114 & 79.225 & 47.469 & 30.624 & 23.838 & 105.702 & 53.880 & 32.684 & 12.360 & 7.829 & 6.227 \\
			& TA&   96.612 & 75.610 & 62.573 & \B 38.226 & \B 25.892 & \B 19.993 & \B 95.302 &  \B 50.878 & 31.833 & 11.921 & \B 7.464 & \B 6.030\\
			& G-TA &  97.449 & \B 75.499 & \B 62.484 & 38.886 & 26.004 & 20.091 & 96.410 & 50.985 &  31.176 & \B 11.861 &7.549 &  6.087\\
			\Xcline{1-14}{0.1pt}
			\multirow{6}{*}{Inception-v4} & BA \cite{brendel2018decisionbased}& 104.275 & 101.115 & 99.872 & 96.700 & 79.412 & 71.387 & - & 116.855 & 112.335 & 104.557 & 85.044 & 64.123 \\
			& Sign-OPT \cite{cheng2019sign} &  95.388 & 81.865 & 66.159 & 42.871 & 31.241 & 25.798 & 121.725 & 77.838 & 40.465 & 14.268 & 8.924 & 7.153\\
			& SVM-OPT \cite{cheng2019sign}& \B 92.640 & 77.616 & 62.949 & 42.552 & 31.142 & 26.238 & 120.407 & 64.600 & \B 33.960 & \B 13.586 & 9.035 & 7.535 \\
			& HSJA \cite{chen2019hopskipjumpattack}&  104.969 & 90.371 & 78.103 & 47.340 & 31.404 & 24.270 & 109.422 & 60.356 & 37.302 & 14.191 & 8.790 & 6.934\\
			& TA & 96.808 & \B 74.829 & \B 61.974 & \B 37.155 & \B 26.128 & 21.184 & \B 101.170 & \B 55.876 & 36.403 & 14.176 & \B 8.592 &  \B 6.814 \\
			& G-TA & 95.563 & 75.889 & 62.404 & 38.457 & 26.495 & \B 21.069 & 101.186 & 57.672 &  36.743 & 13.999 &8.694 & 6.856 \\
			\Xcline{1-14}{0.1pt}
			\multirow{6}{*}{SENet-154} & BA \cite{brendel2018decisionbased}& 75.653 & 72.327 & 71.420 & 68.293 & 52.332 & 44.391 & - & 75.355 & 70.498 & 65.186 & 51.950 & 38.164 \\
			& Sign-OPT \cite{cheng2019sign}&  70.500 & 59.556 & 45.566 & 27.062 & 18.218 & 14.400  & 65.524 & 42.690 & 23.688 & 9.054 & 5.331 & 4.194\\
			& SVM-OPT \cite{cheng2019sign}&  73.344 & 55.891 & 44.195 & 27.826 & 19.544 & 15.883 & 65.957 & 35.596 & \B 20.549 & 8.760 & 5.368 & 4.332 \\
			& HSJA \cite{chen2019hopskipjumpattack}&  72.589 & 60.361 & 49.487 & 25.718 & 14.929 & 12.197 & 70.043 & 34.697 & 21.811 & 8.098 & 4.482 & 3.707\\
			& TA &  66.285 & \B 51.012 & \B 40.475 & \B 21.590 & \B 13.293 & \B 10.782 & \B 64.784 & 34.034 & 22.269 & \B 7.636 & 4.273 & 3.555\\
			& G-TA & \B 66.077 & 51.852 & 41.065 & 21.946 & 13.461 & 10.899 & 65.122 & \B 33.841 & 21.823 & 7.772 &  \B 4.231 &  \B 3.489\\
			\Xcline{1-14}{0.1pt}
			\multirow{6}{*}{ResNet-101} & BA \cite{brendel2018decisionbased}& 76.772 & 72.674 & 71.761 & 68.231 & 54.847 & 47.785 & - & 63.568 & 59.384 & 55.402 & 42.777 & 29.097 \\
			& Sign-OPT \cite{cheng2019sign}& 72.361 & 62.383 & 48.664 & 30.089 & 20.752 & 16.478 & 53.757 & 35.070 & 19.035 & 8.442 & 5.929 & 4.999 \\
			& SVM-OPT \cite{cheng2019sign}&  73.758 & 58.716 & 47.496 & 30.443 & 21.502 & 17.535 & 52.471 & 29.225 & 16.469 & 8.245 & 6.043 & 5.259\\
			& HSJA \cite{chen2019hopskipjumpattack}& 73.422 & 60.175 & 49.443 & 26.504 & 16.035 & 12.661 & 54.869 & 24.971 & 15.161 & 6.084 & 3.787 & 3.237 \\
			& TA & 69.511 & \B 55.389 & 44.343 & 24.500 & \B 14.778 & \B 11.802 & \B 51.829 & 24.748 & 15.162 & 5.941 & \B 3.698 & 3.203 \\
			& G-TA &  \B 69.117 & 56.275 & \B 44.315 & \B 24.316 & 15.133 & 11.946 & 51.883 & \B 24.403 & \B 14.643 & \B 5.842 & 3.703 & \B 3.191 \\
			\bottomrule
	\end{tabular}	}
	\label{tab:ImageNet_normal_models_median_result}
\end{table*}
\begin{table*}[h]
	\centering
	\small
	\tabcolsep=0.1cm
	\setlength{\belowcaptionskip}{0.5pt}%
	\caption{Median $\ell_2$ distortions of different query budgets on the CIFAR-10 dataset.}
	\scalebox{0.8}{
		\begin{tabular}{cccccccc|cccccc}
			\toprule
			Target Model  & Method & \multicolumn{6}{c|}{Targeted Attack} & \multicolumn{6}{c}{Untargeted Attack} \\-
			& & 300 &  @1K & @2K & @5K & @8K & @10K & 300 & @1K & @2K & @5K & @8K & @10K \\
			\midrule
			\multirow{6}{*}{PyramidNet-272} & BA \cite{brendel2018decisionbased}& 8.240  & 7.711 & 7.697 & 6.013 & 3.938 & 3.068 & - & 5.133 & 4.268 & 4.060 & 2.471 & 1.460  \\
			& Sign-OPT \cite{cheng2019sign}& 7.900 & 6.050 & 3.796 & 1.441 & 0.762 & 0.549 &  3.821 & 1.952 & 0.980 & 0.345 & 0.232 & \B 0.196  \\
			& SVM-OPT \cite{cheng2019sign}& 8.870  & 6.432 & 4.199 & 1.651 & 0.894 & 0.655 & 3.777 & 1.956 & 0.877 & 0.363 & 0.235 & 0.202  \\
			& HSJA \cite{chen2019hopskipjumpattack}&  7.616  & 4.013 & 2.109 & \B 0.589 & 0.384 & 0.325 &  3.935 & \B 1.022 & \B 0.587 & 0.294 & 0.224 & 0.201  \\
			& TA &  7.650  & \B 3.874 & \B 2.071 &  0.599 & \B 0.380 & \B 0.318 & \B 3.758  & 1.028 & 0.589 & 0.289 & \B 0.223 &  0.197  \\
			& G-TA &  \B 7.452  & 3.980 & 2.110 & 0.602 & 0.387 & 0.324 &  3.938  & 1.033 & 0.590 & \B 0.288 & 0.224 & 0.198  \\
			\Xcline{1-14}{0.1pt}
			\multirow{6}{*}{GDAS} &  BA \cite{brendel2018decisionbased}& 8.098  & 7.568 & 7.554 & 5.774 & 3.301 & 2.396 &  -  & 2.626 & 2.409 & 2.286 & 1.541 & 1.015  \\
			& Sign-OPT \cite{cheng2019sign}& 7.947 & 6.418 & 4.166 & 1.514 & 0.669 & 0.457 & 2.067  & 1.331 & 0.766 & 0.298 & 0.209 & 0.176  \\
			& SVM-OPT \cite{cheng2019sign}& 9.138  & 7.242 & 5.090 & 2.103 & 1.043 & 0.673 &2.043 & 1.230 & 0.674 & 0.302 & 0.211 & 0.183  \\
			& HSJA \cite{chen2019hopskipjumpattack}&  7.687 & 3.061 & 1.383 & 0.435 & 0.298 & 0.254 & 1.905  & 0.674 & 0.429 & 0.232 & 0.185 & 0.168  \\
			& TA & \B 7.667 & \B 3.024 & \B 1.380 & 0.435 & \B 0.296 & \B 0.253 & 1.932 & 0.690 & \B 0.425 & 0.228 & 0.185 & 0.169  \\
			& G-TA &  7.728  & 3.104 & 1.385 & \B 0.430 & 0.298 & \B 0.253 & \B 1.883 & \B 0.665 & 0.428 & \B 0.226 & \B 0.182 & \B 0.167  \\
			\Xcline{1-14}{0.1pt}
			\multirow{6}{*}{WRN-28} & BA \cite{brendel2018decisionbased}& 8.317  &  7.789 & 7.764 & 5.493 & 2.199 & 1.293 &  - & 3.900 & 3.332 & 3.167 & 1.361 & 0.732  \\
			& Sign-OPT \cite{cheng2019sign}& 7.737 & 5.188 & 2.816 & 0.797 & 0.439 & 0.354 &  2.679 & 1.298 & 0.723 & 0.281 & \B 0.214 & \B 0.191  \\
			& SVM-OPT \cite{cheng2019sign}&  9.054 & 5.697 & 3.317 & 0.981 & 0.511 & 0.398 &  2.627  & 1.279 & 0.627 & 0.288 & 0.218 & 0.198  \\
			& HSJA \cite{chen2019hopskipjumpattack}& 6.446 & 2.064 & 1.005 & 0.443 & 0.339 & \B 0.306 & \B 2.497 & 0.697 & 0.442 & 0.264 & 0.224 & 0.208  \\
			& TA &  6.518 &  \B 2.018 & \B 0.988 & \B 0.428 & \B 0.337 & \B 0.306 &  2.606 & \B 0.682 & \B 0.431 & 0.262 & 0.225 & 0.210  \\
			& G-TA &  \B 6.444 & 2.060 & 1.000 & 0.439 & 0.341 & \B 0.306 &  2.538 & \B 0.682 & 0.444 & \B 0.261 &  0.223 & 0.209  \\
			\Xcline{1-14}{0.1pt}
			\multirow{6}{*}{WRN-40} & BA \cite{brendel2018decisionbased}&  8.181 & 7.760 & 7.722 & 5.482 & 2.193 & 1.363 & -  &3.773 & 3.187 & 3.045 & 1.321 & 0.726  \\
			& Sign-OPT \cite{cheng2019sign}&  7.782 & 5.285 & 2.895 & 0.845 & 0.455 & 0.369 &  2.510  & 1.213 & 0.679 & 0.265 & \B 0.194 & \B 0.173  \\
			& SVM-OPT \cite{cheng2019sign}&  9.042  & 5.835 & 3.400 & 1.030 & 0.549 & 0.420 &  2.500  & 1.251 & 0.611 & 0.272 & 0.198 & 0.179  \\
			& HSJA \cite{chen2019hopskipjumpattack}& 6.578 & 2.183 & 1.040 & 0.439 & 0.338 & \B 0.305 &2.470 & 0.702 & 0.453 & 0.256 & 0.214 &  0.198  \\
			& TA & 6.747  &  2.100 & \B 0.983 & \B 0.435 & \B 0.337 & \B 0.305 &  2.584  & \B 0.680 & \B 0.434 & \B 0.254 & 0.215 & 0.201  \\
			& G-TA & \B  6.514  & \B 2.069 & 1.014 & 0.438 & 0.339 & 0.306 & \B 2.453  & 0.695 & 0.441 & 0.255 &  0.213 & 0.199  \\
			\bottomrule
	\end{tabular}}
	\label{tab:CIFAR-10_normal_models_median_result}
\end{table*}

(1) TA and G-TA perform better in attacking high-resolution images, \textit{i.e.,} the images of the ImageNet dataset. The median $\ell_2$ distortions of Table \ref{tab:ImageNet_normal_models_median_result} are larger than that of Table \ref{tab:CIFAR-10_normal_models_median_result}, because the high-resolution images of the ImageNet dataset lead to larger $\ell_2$ distortions.

(2) TA is more effective in the targeted attacks. We speculate that it is because the adversarial region of the target class is narrower and more scattered in the targeted attack, resulting in a smoother decision boundary. Thus, TA is more suitable for targeted attacks.

\end{document}